\documentclass[11pt,oneside,letter]{article}
\usepackage[top=0.8 in, bottom=0.8 in, left=0.8 in, right=0.8 in]{geometry}
\usepackage{hyperref}

\usepackage{amsmath}
\usepackage[round]{natbib}
\usepackage[table]{xcolor}
\usepackage{xcolor}
\usepackage{amsthm} 
\newtheorem{theorem}{Theorem}
\newtheorem{lemma}{Lemma}

\newtheorem{definition}{Definition}

\newtheorem{assumption}{Assumption}
\newtheorem{remark}{Remark}
\usepackage{microtype}
\usepackage{amssymb}
\usepackage{tabularx}
\usepackage{subcaption}
\usepackage[noend,linesnumbered,ruled,vlined]{algorithm2e}
\usepackage{amssymb} 
\usepackage{graphicx}
\usepackage{multirow}
\usepackage{pifont}
\usepackage{booktabs} 
\usepackage[table]{xcolor}
\usepackage{enumitem}
\newcommand{\bc}{{B}^c}
\newcommand{\R}{\mathbb R}

\newcommand{\bP}{\mathbb{P}}

\newcommand{\cA}{\mathcal A}

\newcommand{\cE}{\mathcal E}
\newcommand{\cF}{\mathcal F}

\newcommand{\cJ}{\mathcal J}

\newcommand{\cO}{\mathcal O}

\newcommand{\cR}{\mathcal R}
\newcommand{\cS}{\mathcal S}
\newcommand{\cT}{\mathcal T}

\newcommand{\cV}{\mathcal V}

\newcommand{\bI}{\mathbf{I}}

\usepackage{hyperref} 
\hypersetup{
	colorlinks=true,
	linkcolor=black,
	filecolor=blue,
	citecolor = blue,      
	urlcolor=blue,
}

\title{Provably Efficient Model-Free Algorithms for Non-stationary CMDPs}
\date{}
\author{%
 Honghao Wei \\
  University of Michigan, Ann Arbor \\  \texttt{honghaow@umich.edu} \\ 
  \and
  Arnob Ghosh \\
 The Ohio State University\\ \texttt{ghosh.244@osu.edu} \\
  \and
  Ness Shroff \\
  The Ohio State University \\  \texttt{shroff.11@osu.edu}\\
  \and
  Lei Ying \\
   University of Michigan, Ann Arbor \\  \texttt{leiying@umich.edu}\\
  \and
 Xingyu Zhou \\
 Wayne State University\\
  \texttt{xingyu.zhou@wayne.edu}
       \\
}

\begin{document}

\maketitle

\begin{abstract}
	We study model-free reinforcement learning (RL) algorithms in episodic non-stationary constrained Markov Decision Processes (CMDPs), in which an agent aims to maximize the expected cumulative reward subject to a cumulative constraint on the expected utility (cost). In the non-stationary environment, reward, utility functions, and transition kernels can vary arbitrarily over time as long as the cumulative variations do not exceed certain variation budgets. We propose the first model-free, simulator-free RL algorithms with sublinear regret and zero constraint violation for non-stationary CMDPs in both tabular and linear function approximation settings with provable performance guarantees. Our results on regret bound and constraint violation for the tabular case match the corresponding best results for stationary CMDPs when the total budget is known. Additionally, we present a general framework for addressing the well-known challenges associated with analyzing non-stationary CMDPs, without requiring prior knowledge of the variation budget. We apply the approach for both tabular and linear approximation settings.
\end{abstract}

\section{INTRODUCTION}
Safe reinforcement learning (RL) studies how to apply RL algorithms in real-world applications \citep{AmoOlaSte_16,GarFer_15,BruGreHal_22} that can operate under safety-related constraints. A standard approach for modeling applications with safety constraints is Constrained Markov Decision Processes (CMDPs) \citep{Alt_99}, where an agent seeks to learn a policy that maximizes the expected total reward under safety constraints on the expected total utility. In classical safe-RL and CMDPs problems, an agent is assumed to interact with a stationary environment. However, stationary models cannot capture the time-varying real-world applications where safety is critical such that the transition functions and reward/utility functions are non-stationary. For example, in autonomous driving \citep{KirSobTal_21}, collisions must be avoided while modeling and tracking time-varying environments such as traffic conditions; in an automated medical system \citep{CorNaeDe_20}, it is essential to guarantee patient safety under varying patients' behavior.

Learning in a stationary CMDP is a long-standing topic and has been heavily studied recently, including using both model-based and model-free approaches \citep{BraDudLyk_20,EfrManPir_20,WeiLiuYin_22,WeiLiuYin_22-2,LiuZhoKal_21,BurHasKal_21,SinGupShr_20,DinWeiYan_20,CheJaiLuo_22}. RL in non-stationary CMDPs is more challenging since the rewards/utilities and dynamics are time-varying and probably unknown a priori. On the one hand, an agent has to handle the non-stationarity properly to guarantee a sublinear regret and a small or zero constraint violation. On the other hand, the agent also needs to forget the past data samples since they become less useful due to the dynamic of the system. The only existing work of which we are aware that studies non-stationary CMDPs is \cite{DingLav_22}, via a model-based approach assuming a priori knowledge of the total variation budgets, which is far less computationally efficient compared with model-free approaches and where knowing the variation budgets is less desirable in practice.

\begin{table}[ht]
	\caption{Dynamic regret and constraint violations comparisons for RL in non-stationary CMDPs. $S$ and $A$ are the number of states and actions, $H$ is the horizon of each episode, $K$ is the total number of episodes and $B$ is the variation budgets. $d$ is the dimension of the feature in linear CMDP. Algorithms \ref{alg:triple-Q},\ref{alg:triple-Q-double} are for tabular setting, \cite{DingLav_22} is for Linear kernel CMDP setting, and Algorithm \ref{algo:model_free} is for linear CMDP setting. ($^\ast:$ zero constraint holds holds when $K$ is large enough, $\dagger:$ we can further the regret order to $\tilde{\cO}(K^{3/4}),$ see section \ref{sec:linear_cmdp}. )}
	\label{ta:results}
	\begin{center}
	 	\footnotesize	
		\begin{tabular}{|c|c|c|c|c|}
			\toprule
			 {\bf Algorithm} & {\bf Model-Free?} & {\bf Regret} & {\bf Constraint Violation} & {\bf Known Budget?}\\
			\hline
				 \cite{DingLav_22} &    \ding{55}&	$\tilde{\mathcal{O}}(S^{\frac{2}{3}}A^{\frac{1}{3}}H^{\frac{5}{3}}K^{\frac{3}{4}} B^{\frac{1}{3}})$  & $\tilde{\mathcal{O}}(S^{\frac{2}{3}}A^{\frac{1}{3}}H^{\frac{5}{3}}K^{\frac{3}{4}} B^{\frac{1}{3}})$ & \ding{51} \\
			\toprule
			 \cellcolor{blue!25} Algorithm \ref{alg:triple-Q} & \ding{51}& $\tilde{\cal O}\left(H^4 S^{\frac{1}{2}}A^{\frac{1}{2}}K^{\frac{4}{5}}{B}^{\frac{1}{3}} \right)$ &$0^\ast$ & \ding{51} \\
			\hline
		 \cellcolor{blue!25} Algorithm \ref{alg:triple-Q-double} & \ding{51} &$\tilde{\cal O}\left(H^4S^\frac{1}{2}A^{\frac{1}{2}}K^{\frac{8}{9}}{B}^{\frac{1}{3}}\right)$ &$0^\ast$ & \ding{55} \\
			\hline
		 \cellcolor{blue!25} Algorithm \ref{algo:model_free} & \ding{51} &$\tilde{\mathcal{O}}\left(K^{3/4}H^{9/4}d^{5/4}B^{1/4}\right)$ &$0^\ast $ & \ding{51} \\
		 	\hline
		 \cellcolor{blue!25} Algorithm \ref{algo:model_free_unknown} & \ding{51} &$\tilde{\mathcal{O}}\left(K^{7/8}H^{9/4}d^{5/4}B^{1/4}\right)^\dagger $ &$0^\ast$ & \ding{55} \\
			\bottomrule
		\end{tabular}
	\end{center}
\end{table}
In this work, we manage to overcome these challenges and focus on designing model-free algorithms with sublinear regret and zero constraint violation guarantees for non-stationary CMDPs, especially for the scenario when the total variation budget is unknown. Our contributions are as follows:
\begin{itemize}[leftmargin=*]
    \item {Our work contributes to the theoretical understanding of non-stationary episodic CMDPs. We develop different type of model-free algorithms for non-stationary CMDP settings-- one is tailored for tabular CMDPs and has low memory and computational complexity, another one is computationally more intensive, however, can be applied to linear function approximation for large, possibly infinite, state and action spaces.}
    \item For the tabular setting, our algorithm adopts a periodic restart strategy and utilizes an extra optimism bonus term to counteract the non-stationarity of the CMDP that an over estimate of the combined objective is guaranteed during learning and exploration. For the case when the budget variation is known, our theoretical result $\tilde{O}(K^{4/5})$ matches the best existing result for stationary CMDPs in terms of the total number of episodes $K,$ and non-stationary MDPs in term of the variation budget $B.$ For linear CMDP,  we propose the first model-free, value-based algorithm which obtains $\tilde{\cO}(K^{3/4})$ regret and zero constraint violation using the same strategy. Our result in fact improves the dependency with respect to the budget variation and the episode length $H$ compared to \cite{DingLav_22}.
    \item We develop, for the first time, a general {\em double restart} method for non-stationary CMDPs based on the ``bandit over bandit'' idea. This method can be used for other non-stationary constrained learning problems which aims to achieve zero constraint violation. The method removes the need to have a priori knowledge of the variation budget, an open-problem raised in \cite{DingLav_22} for non-stationary CMDPs. While the ``bandit over bandit'' has been widely used and studied for unconstrained MDPs, adopting it for CMDPs is nontrivial due to multiple challenges that do not exist in unconstrained setting. For example, one  needs to account for the constraints. We overcome these difficulties by a new design of the bandit reward function for each arm. We show that the approach can be used in conjunction with the algorithms for the tabular and linear function approximation cases.
\end{itemize}
Our results are summarized in Table \ref{ta:results}.
\section{RELATED WORK}
{\bf Non-stationary MDP.}
Non-stationary unconstrained MDPs have been mostly studied recently \citep{AueJakOrt_08,CheSimZhu_20,DomMnaPir_21,FeiYanWan_20,OrtGajAUe_20,TouVin_20,WeiLuo_21,ZhoYanSce_21,ZhoCheVar_20,MaoZhaZhu_20}. \cite{AueJakOrt_08} consider a setting where the MDP is allowed to change for fixed number of times. When the variation budget is known a priori, \cite{FeiYanWan_20} propose a policy-based algorithm in the setting where they assume stationary transitions and adversarial full-information rewards. \cite{ZhoYanSce_21,MaoZhaZhu_20,TouVin_20,ZhoCheVar_20} consider a more general setting that both transitions and rewards are time-varying. A more recent work \cite{WeiLuo_21} introduce a procedure that can be used to convert any upper-confidence-bound-type stationary RL problem to a non-stationary RL algorithm to relax the assumption of having a priori knowledge on the variation budget.

{\bf CMDP.} Stationary CMDPs with provable guarantees have been heavily studied in recent years. In particular, \cite{BraDudLyk_20,EfrManPir_20,SinGupShr_20} propose model-based approaches for tabular CMDPS. \cite{GhoZhoShr_22,DinWeiYan_20} extend the results to the linear and linear kernel CMDPs. \cite{LiuZhoKal_21,BurHasKal_21} also provide efficient algorithms with a zero constraint violation guarantee. Besides using an estimated model, \cite{DinZhaBas_20,CheDonWan_21} leverage a simulator for policy evaluation to achieve provable regret guarantees. Moreover, \cite{WeiLiuYin_22,WeiLiuYin_22-2} propose the first model-free and simulator-free algorithms for CMDPs with sublinear regret and zero constraint violation. However, the studies on non-stationary CMDPs are limited. For non-stationary CMDPs, \cite{QiuWeiYan_20} consider CMDPs that assume that only the rewards vary over episodes. A concurrent work \citep{DingLav_22}, which is most related to ours, focuses on the same setting where the transitions and rewards/utilities vary over episodes under a linear kernel CMDP assumption. They also assume that the budget is known a priori. The method proposed is a model-based approach, but we instead consider a more challenge setting where the algorithm is model-free and the budget is not known. Fortunately, we answer the open-problem affirmatively raised in \cite{DingLav_22}.   

\section{PROBLEM FORMULATION} 
\label{sec:problem}
We consider an episodic CMDP where an agent interacts with a non-stationary system for $K$ episodes. The CMDP is denoted by $(\mathcal{S},\mathcal{A},H,\mathbb{P},r,g),$ where $\mathcal{S}$ is the state space with $\vert \mathcal{S}\vert=S,$ $\mathcal{A}$ is the action space with $\vert \mathcal{A}\vert=A,$ $H$ is the fixed length of each episode, $\mathbb{P}=\{\mathbb{P}_{k,h}\}_{k\in[K],h\in[H]}$ is a collection of transition kernels, and $r=\{r_{k,h}\}_{k\in[K],h\in[H]} (g=\{g_{k,h}\}_{k\in[K],h\in[H]}) $ is the set of reward (utility) functions. In Section~\ref{sec:linear_cmdp}, we extend our analysis to potentially infinite state-space.

At the beginning of an episode $k$, an initial state $x_{k,1}$ is sampled from the distribution $\mu_0.$ Then at step $h,$ the agent takes action $a_{k,h}\in\mathcal{A}$ after observing state $x_{k,h}\in\mathcal{S}$. Then the agent receives a reward $r_{k,h}(x_{k,h},a_{k,h})$ and incurs a utility $g_{k,h}(x_{k,h},a_{k,h}).$ The environment transitions to a new state $x_{k,h+1}$ following from the distribution $\mathbb{P}_{k,h}(\cdot\vert x_{k,h},a_{k,h}).$ It is worth emphasizing that the transition kernels, reward functions, and utility functions all depend on the episode index $k$ and time $h,$ and hence the system is non-stationary. For simplicity of notation, we assume that  $r_{k,h}(x,a)(g_{k,h}(x,a)):\mathcal{S}\times \mathcal{A}\rightarrow [0,1],$ are deterministic for convenience. Our results generalize to the setting where the reward and utility functions are random. Given a policy $\pi,$ which is a collection of $H$ functions $\pi: [H]\times \mathcal{S}\rightarrow\mathcal{A},$ where $[H]$ represents the set $\{1,2,\dots,H\}.$ Define the reward value function $V_{k,h}^\pi(x):\mathcal{S}\rightarrow\mathbb{R}^+ $ at episode $k$ and step $h$ to be the expected cumulative rewards from step $h$ to the end under the policy $\pi:$ 
\begin{align}
V_{k,h}^\pi(x)=\mathbb{E}\left[\left.\sum_{i=h}^H r_{k,i}(x_{k,i},\pi(x_{k,i}))\right\vert x_{k,h}=x \right].
\end{align}
The (reward) $Q$-function $Q_{k,h}^\pi(x, a):\mathcal{S}\times \mathcal{A}\rightarrow \mathbb{R}^+ $ is the expected cumulative reward when an agent starts from a state-action pair $(x, a)$ at episode $k$ and step $h$ following the policy $\pi:$ 
\begin{align}
&Q_{k,h}^\pi(x,a )=r_{k,h}(x,a) + \nonumber\\
&\mathbb{E}\left[\left.\sum_{i=h+1}^H r_{k,i}(x_{k,i},\pi(x_{k,i}))\right\vert x_{k,h}=x, a_{k,h}=a \right].
\end{align}
Similarly, we use 
$W_{k,h}^\pi(x):\mathcal{S}\rightarrow\mathbb{R}^+$ to denote the utility value function 
\begin{align}
W_{k,h}^\pi(x) =\mathbb{E}\left[\left.\sum_{i=h}^H g_{k,i}(x_{k,i},\pi(x_{k,i}))\right\vert x_{k,h}=x \right],
\end{align}
and we use

$C_{k,h}^\pi(x,a):\mathcal{S}\times\mathcal{A}\rightarrow\mathbb{R}^+$ to denote the utility $Q$-function at episode $k,$ step $h$:
\begin{align}
 & C_{k,h}^\pi(x,a) =g_{k,h}(x,a) + \nonumber \\
 & \mathbb{E}\left[\left.\sum_{i=h+1}^H g_{k,i}(x_{k,i},\pi(x_{k,i}))\right\vert  x_{k,h}=x, a_{k,h}=a
\right].   
\end{align}
For simplicity, we adopt the following notations: 
\begin{align}
\mathbb{P}_{k,h}V_{k,h+1}^\pi(x,a)=&\mathbb{E}_{x^\prime\sim\mathbb{P}_{k,h}(\cdot\vert x,a)}V^\pi_{k,h+1}(x^\prime),\\
\mathbb{P}_{k,h}W_{h+1}^\pi(x,a)=&\mathbb{E}_{x^\prime\sim\mathbb{P}_{k,h}(\cdot\vert x,a)}W^\pi_{k,h+1}(x^\prime).
\end{align}

We also denote the empirical counterparts as 
\begin{align}
    \hat{\mathbb{P}}_{k,h}V_{k,h+1}^\pi(x,a) = & V_{k,h+1}^\pi (x_{k+1,h}),\\
    \hat{\mathbb{P}}_{k,h}W_{k,h+1}^\pi(x,a) = &W_{k,h+1}^\pi(x_{k+1,h}),
\end{align}
and is only defined for $(x,a)=(x_{k,h},a_{k,h}).$ Given the model defined above, the objective of the episode $k$ is to find a policy that maximizes the expected cumulative reward subject to a constraint on the expected utility:
\begin{equation}
	\underset{\pi_k \in\Pi}{\text{max}}\  \mathbb E\left[V_{k,1}^{\pi_k}(x_1)\right] \quad \text{subject to:}\  \mathbb E\left[W_{k,1}^{\pi_k}(x_1)\right]\geq \rho, \label{eq:obj}
\end{equation}
where we assume $\rho\in[0,H]$ to avoid triviality, and the expectation is taken with respect to the initial distribution and the randomness of $\pi$. Let $\pi_k^*$ denote the optimal solution to the CMDP problem defined in \eqref{eq:obj} for episode $k$. We evaluate our model-free RL algorithms using dynamic regret $\mathcal{R}(K)$ and constraint violation $\mathcal{V}(K)$ defined below:
\begin{align}
	&\cR (K) =\mathbb E\left[\sum_{k=1}^K\left(V_{k,1}^{\pi_k^*}(x_{k,1})-V_{k,1}^{\pi_k}(x_{k,1})\right)\right],\label{def:regret}\\
	&\cV (K) =\mathbb E\left[ \sum_{k=1}^K \left(\rho-W_{k,1}^{\pi_k}(x_{k,1})\right)\right],\label{def:cv}
\end{align}
where $\pi_k$ is the policy used in episode $k$. 
 Note that here we use dynamic regret concept as the optimal policy may be different. We further make the following standard assumption \citep{EfrManPir_20,DinWeiYan_20,QiuWeiYan_20,WeiLiuYin_22}. 
\begin{assumption}
	\label{as:1}
 (Slater's Condition). Given initial distribution $\mu_0,$ for any episode $k\in[K],$ there exist $\delta>0$ and at least a policy $\pi$ such that $\mathbb E\left[W_{k,1}^\pi(x_{k,1})\right] -\rho \geq \delta.$ 
\end{assumption}

{\bf Variation:} The non-stationary of the CMDP is measured according to the variation budgets in the reward/utility functions and the transition kernels:
\begin{align*}
	{B}_r & := \sum_{k=1}^{K-1}\sum_{h=1}^H \max_{x,a} \vert r_{k,h}(x,a) - r_{k+1,h}(x,a)\vert \\
	{B}_g &:= \sum_{k=1}^{K-1}\sum_{h=1}^H \max_{x,a} \vert g_{k,h}(x,a) - g_{k+1,h}(x,a)\vert \\
	{B}_p &:= \sum_{k=1}^{K-1}\sum_{h=1}^H \max_{x,a} \Vert \mathbb{P}_{k,h}(\cdot\vert x,a) - \mathbb{P}_{k+1,h}(\cdot\vert x,a)\Vert_1.
\end{align*}

We further let ${B}= {B}_r + {B}_g + {B}_p$ to represent the total variation. To bound the regret, we consider the following offline optimization problem at  episode $k$ as our regret baseline:
\begin{align}
	&\underset{q_{k,h} }{\max} \sum_{h,x,a} q_{k,h}(x,a)r_{k,h}(x,a) \label{eq:lp_app}\\
	&\hbox{s.t.:} \sum_{h,x,a} q_{k,h}(x,a)g_{k,h}(x,a) \geq \rho \label{lp:cost_app}\\
	& \sum_a q_{k,h}(x,a) = \sum_{x^\prime,a^\prime} \mathbb P_{k,h-1}(x\vert x^\prime,a^\prime) q_{k,h-1} (x^\prime,a^\prime)\label{lp:gb_app}\\
	& \sum_{x,a}q_{k,h}(x,a) =1, \forall h\in[H] \label{lp:normalization_app}\\
	&\sum_{a} q_{k,1}(x,a)= \mu_0(x)\label{lp:ini_app} \\
	& q_{k,h}(x,a) \geq 0, \forall x\in\mathcal{S},\forall a\in\mathcal{A},\forall h\in[H]. \label{lp:p_app}
\end{align}	
To analyze the performance, we need to consider a tightened version of the LP, which is defined below: 
\begin{align}
	\underset{q_{k,h} }{\max} &\sum_{h,x,a} q_{k,h} (x,a)r_{k,h}(x,a) \label{eq:lp-epsilon_app}\\
	\hbox{s.t.:} &\sum_{h,x,a} q_{k,h} (x,a)g_{k,h}(x,a)  \geq \rho+\epsilon,\text{ and }\eqref{lp:gb_app}-\eqref{lp:p_app}\nonumber,
\end{align} where $\epsilon>0$ is called a tightness constant. When $\epsilon\leq \delta,$ this problem has a feasible solution due to Slater's condition. We use superscript ${ }^*$ to denote the optimal value/policy related to the original CMDP \eqref{eq:obj} or the solution to the corresponding LP \eqref{eq:lp_app} and superscript ${ }^{\epsilon,*}$ to denote the optimal value/policy related to the $\epsilon$-tightened version of CMDP. 
\section{ALGORITHM FOR TABULAR CMDPs}
Next we will start with presenting our algorithm Non-stationary Triple-Q in Algorithm \ref{alg:triple-Q} for the scenario when the variation budget is known. Our algorithm uses a restart strategy that divides the total episode $K$ into frames, which is commonly used in both non-stationary bandits and RL to address non-stationarity. We remark that in unconstrained RL, this restarting often results in a worse regret, for example, the regret bound is $\tilde{\mathcal{O}}(\sqrt{K})$ \citep{JinAllBub_18} in the stationary setting but becomes $\tilde{\mathcal{O}}(K^{\frac{2}{3}})$ \citep{MaoZhaZhu_20} when the system is non-stationary. However, the order of regret achieved by our Algorithm \ref{alg:triple-Q} matches the best existing result in stationary CMDPs obtained by the model-free algorithm Triple-Q \citep{WeiLiuYin_22} under the same setting. That is because Triple-Q itself is built on top of a two-time-scale scheme for balancing the estimation error and tracking the constraint violation, which shares the same insights as the restart strategy for dealing with non-stationarity. Therefore, by appropriately designing the frame size (restarting period), Algorithm \ref{alg:triple-Q} can achieve the same order as that in unconstrained CDMPs as well as the optimal order in terms of variation budget. 

We first divide the total $K$ episodes into frames, where each frame contains $K^\alpha / B^c$ episodes. Define $B_r^{(T)},B_g^{(T)},B_p^{(T)}$ to be the local variation budget of the reward functions, utility functions and transition kernels within the $T$th frame, let $\mathcal{N}_T$ denote the set of all the episodes in frame $T,$ then
\begin{align*}
	{B}_r^{(T)} & := \sum_{k\in\mathcal{N}_T}\sum_{h=1}^H \max_{x,a} \vert r_{k,h}(x,a) - r_{k+1,h}(x,a)\vert \\
	{B}_g^{(T)} &:= \sum_{k\in\mathcal{N}_T}\sum_{h=1}^H \max_{x,a} \vert g_{k,h}(x,a) - g_{k+1,h}(x,a)\vert \\
	{B}_p^{(T)} &:= \sum_{k\in\mathcal{N}_T}\sum_{h=1}^H \max_{x,a} \Vert \mathbb{P}_{k,h}(\cdot\vert s,a) - \mathbb{P}_{k+1,h}(\cdot\vert x,a)\Vert_1.
\end{align*}  
Let the total local variation budget $B^{(T)}= {B}_r^{(T)} + {B}_g^{(T)} + {B}_p^{(T)},$ then by definition we have $\sum_{T=1}^{K^{1-\alpha}B^c}B^{(T)} \leq B.$ Our algorithm uses two bonus terms $b_t$ and $\tilde{b}$ to update $Q$ values (Line $10-11$ in Algorithm \ref{alg:triple-Q}), where $b_t$ is the standard Hoeffding-based bonus in upper confidence bounds, and $\tilde{b}$ is the extra bonus to take into account the non-stationarity of the environment. We assume that $\tilde{b}$ is a uniform upper bound on the total local variation budget $B^T$ for any $T,$ and satisfies $K^{1-\alpha}B^c\tilde{b} \leq B $ which is an assumption commonly seen in the literature on non-stationary RL \citep{OrtGajAUe_20,MaoZhaZhu_20,ZhoCheVar_20}.

\begin{algorithm*}[!ht]
	\SetAlgoLined
	{\bf Input:} Total Budget $B$\;
	Choose $\alpha=0.6,\eta=K^{\frac{1}{5}}B^{\frac{1}{3}},\chi=K^{\frac{1}{5}},c=\frac{2}{3},\epsilon=\frac{8\sqrt{SAH^6\iota^3}B^{1/3}}{K^{0.2}},$ and $\iota = 128\log(\sqrt{2SAH}K)$ \;
	Initialize  $Q_{h}(x,a)=C_{h}(x,a)\leftarrow H$ and  $Z=\bar{C}=N_{h}(x,a)=V_{H+1}(x)={W}_{H+1}(x) \leftarrow 0$ for all $(x,a,h)\in{\cal S}\times {\cal A}\times [H]$\;
	\For{episode $k = 1,\dots,K $ }{
		Sample the initial state for episode $k:$ $x_{k,1} \sim \mu_0$\;
		\For{step $h=1,\dots,H$}{
			Take action $a_{h} \leftarrow \arg\max_a\left( {Q}_{h} (x_{k,h},a) + \frac{Z}{\eta} {C}_{h}(x_{k,h},a)\right)$\;
				Observe $r_{k,h}(x_{k,h},a_{k,h}), g_{k,h}(x_{k,h},a_{k,h}), $ and $x_{k,h+1}, N_{h}(x_{k,h},a_{k,h})\leftarrow N_{h}(x_{k,h},a_{k,h})+1$\;
				Set $t=N_{h}(x_{k,h},a_{k,h}), b_t = \frac{1}{4}\sqrt{\frac{H^2\iota \left(\chi +1 \right) }{\chi+t}}, \alpha_t=\frac{\chi+1}{\chi+t}$ \;
				 ${Q}_{h} (x_{k,h},a_{k,h})\leftarrow (1-\alpha_t)Q_{h} (x_{k,h},a_{k,h}) + \alpha_t\left(r_{k,h}(x_{k,h},a_{k,h})+ {V}_{h+1} (x_{k,h+1})+b_t + 2H\tilde{b} \right)$\;
				${C}_{h} (x_{k,h},a_{k,h})\leftarrow (1-\alpha_t)C_{h} (x_{k,h},a_{k,h}) + \alpha_t\left(g_{k,h}(x_{k,h},a_{k,h})+{W}_{h+1}(x_{k,h+1})+b_t + 2H\tilde{b} \right)$\;
		$a'=\arg\max_a\left({Q}_{h} (x_{k,h},a) + \frac{Z}{\eta} {C}_{h}(x_{k,h},a)\right)$,
		${V}_{h}(x_{k,h}) \leftarrow {Q}_h(x_{k,h},a')  \quad  {W}_{h}(x_{k,h}) \leftarrow {C}_h(x_{k,h},a') $ \;
			\If{$h=1$}{
				$\bar{C} \leftarrow \bar{C}+{C}_{1}(x_{k,1},a_{k,1})$
			}
		}
		\If{$k\mod(K^\alpha/ \bc )=0$  \tcp*{\small reset visit counts and Q-functions}} { 
			$N_{h}(x,a)\leftarrow 0,$ $Q_{h}(x,a) = C_h(x,a) = {Q}_{h}(x,a)={C}_{h}(x,a)\leftarrow H,$ 
			$Z\leftarrow \left(Z +\rho+\epsilon -  \frac{\bar{C}\cdot \bc }{K^\alpha}\right)^+,\bar{C}\leftarrow 0$
		}
	}
	\caption{{Non-stationary Triple-Q}}\label{alg:triple-Q}
\end{algorithm*}
\section{RESULTS OF TABULAR CMDPs}
We now present our main results of the Non-stationary Triple-Q.
\begin{theorem}\label{the:main}
Assume $K \geq \max\left\{\left(\frac{16\sqrt{SAH^6\iota^3}}{\delta} \right)^5,e^{\frac{1}{\delta}}\right\},$ where $\iota=128\log(\sqrt{2SAH}K).$ Algorithm \ref{alg:triple-Q} achieves the following regret and constraint violation bounds: 
	\begin{align*}
		\cR (K) = \tilde{\cal O}(H^4S^\frac{1}{2}A^{\frac{1}{2}}{B}^{\frac{1}{3}} K^{\frac{4}{5}})\quad\quad \cV(K) =  0
	\end{align*}
\end{theorem}
Due to the page limit, we only outline some of the key intuitions behind Theorem \ref{the:main}. The detailed proofs are deferred to Section \ref{ap:proof-the-tabular} in the supplementary materials. 
\subsection{Dynamic Regret}
As shown in Algorithm \ref{alg:triple-Q}, let $Q_{k,h}(x,a),C_{k,h}(x,a)$ denote the estimate $Q$ values at the beginning of the $k-$th episode. The dynamic regret can be decoupled as:
\begin{align}
&  \cR(K)
 = \mathbb E \left[\sum_{k=1}^{K} \left( \sum_a  \left\{{Q}_{k,1}^{*}{q}^{*}_{k,1} -{Q}_{k,1}^{\epsilon,*}{q}^{\epsilon,*}_{k,1}\right\}(x_{k,1},a)    \right)\right]  +\label{step:epsilon-dif} \\
&\mathbb E \left[\sum_{k=1}^{K}  \left(  \sum_a \left\{{Q}_{k,1}^{\epsilon,*}{q}^{\epsilon,*}_{k,1}\right\}(x_{k,1},a)-{Q}_{k,1}(x_{k,1}, a_{k,1}) \right)\right]+\label{step:drift}\\
&\mathbb E \left[\sum_{k=1}^{K}  \left\{{Q}_{k,1}-  Q_{k,1}^{\pi_k}\right\}(x_{k,1}, a_{k,1}) \right],\label{step:biase}
\end{align}
here we use the shorthand notation $\{f-g\}(x)=f(x)-g(x).$ Before bounding each term, we first show that for any triple $(x,a,h),$ the difference of two different reward/utility Q-value functions within the same frame are bounded by the local variation bound in that frame.
\begin{lemma}\label{le:qk-diff}
	Given any frame $T,$ for any $(x,a,h),$ and $(T-1)K^\alpha/\bc \leq k_1\leq k_2\leq TK^\alpha/ \bc,$ we have 
	\begin{align}
	&\vert Q_{k_1,h}^{\pi}(x,a) - Q_{k_2,h}^{\pi'}(x,a)\vert \leq  H\tilde{b}\\
	&\vert C_{k_1,h}^{\pi}(x,a) - C_{k_2,h}^{\pi'}(x,a)\vert \leq  H\tilde{b}  
	\end{align}
\end{lemma}
Then we show that in Lemma \ref{le:epsilon-dif} in supplementary materials the first term \eqref{step:epsilon-dif} can be bounded by comparing the original LP associated with the tightened LP such that $\eqref{step:epsilon-dif} \leq \frac{KH\epsilon}{\delta}.$ The term \eqref{step:biase} is the estimation error between $Q_{k,h}$ and the true $Q$ value under policy $\pi_k$ at episode $k.$ This estimation error can be bounded by our choice of the learning rate (Line $8$ in Algorithm \ref{alg:triple-Q}) and the added bonus. Then
$\eqref{step:biase} 
\leq H^2SA K^{1-\alpha}{B}^c  +\frac{2(H^3\sqrt{\iota}+2H^2\tilde{b})K}{\chi} +  \sqrt{H^4SA\iota K^{2-\alpha} (\chi+1){B}^c   } + 2\tilde{b}H^2K.$

For the remaining term \eqref{step:drift}, we need to add and subtract additional terms to construct an difference between the optimal combined $Q$ value  $\{Q_{k,h}^*+\frac{Z}{\eta}\}C_{k,h}^*(x,a)$ and the estimated counterpart $\{Q_{k,h} + \frac{Z}{\eta}C_{k,h}\}(x,a).$ We will show in Lemma \ref{le:qk-qpi-relation} that the estimation is always an overestimation for all $(x,a,h,k)$ due to the added bonus when the virtual ``queue'' $Z_T$ is fixed with high probability, which implies that the difference is negative with high probability. Then in Lemma \ref{le:drift} we leverage Lyapunov-drift method and consider Lyapunov function $L_T=\frac{1}{2}Z_T^2$ to show that the redundant term can also be bounded. Combining the bounds on the estimation and the redundant term we can obtain $\eqref{step:drift}\leq \frac{K(2H^4\iota +4H^2\tilde{b}^2 +\epsilon^2)}{\eta}+\frac{(\eta+K^{1-\alpha})H^2B^c}{\eta K}.$
Then combining inequalities \eqref{step:epsilon-dif},\eqref{step:drift},\eqref{step:biase} above we can obtain for $K \geq  \left(\frac{16\sqrt{SAH^6\iota^3}{B}^{1/3}}{\delta} \right)^5,$

applying the condition $K^{1-\alpha}B^c\tilde{b}\leq B,$ along with our choices of parameters (Line $2$ in Algorithm \ref{alg:triple-Q}) for balancing each terms, we conclude that $\cR (K)=\tilde{\cal O}(H^4S^\frac{1}{2}A^{\frac{1}{2}}{B}^{\frac{1}{3}} K^{\frac{4}{5}}).$

\subsection{Constraint Violation}\label{app:constraint}
According to the virtual-Queue update, we have 
\begin{align}
   Z_{T+1} =  \left(  Z_T   + \rho + \epsilon -\frac{\bar{C}_T\bc}{K^\alpha}\right)^+ 
\geq Z_T   + \rho + \epsilon -\frac{\bar{C}_T\bc}{K^\alpha}, 
\end{align}
which implies that for $ (T-1)K^\alpha/\bc\leq k\leq TK^\alpha/\bc,$
\begin{align*}
	\sum_{k} \left(-C_{k,1}^{\pi_k} (x_{k,1},a_{k,1}) +\rho \right) 
	\leq \frac{K^\alpha}{\bc}\left(Z_{T+1} - Z_T \right) 
	 + \sum_{k} \left(\left\{C_{k,1}- C_{k,1}^{\pi_k} \right\} (x_{k,1},a_{k,1}) - \epsilon \right).  
\end{align*}
Summing the inequality above over all frames and taking expectation on both sides, we obtain the following upper bound on the constraint violation: 
\begin{align}
    \mathbb{E} \left[\sum^{K}_{k=1}\rho-C_{k,1}^{\pi_k} (x_{k,1},a_{k,1})     \right]
	\leq  -K\epsilon + \frac{K^\alpha}{\bc} \mathbb{E}\left[ Z_{K^{1-\alpha}\bc+1} \right]  
 +\mathbb{E}\left[\sum_{k=1}^K\left\{C_{k,1}- C_{k,1}^{\pi_k}\right\} (x_{k,1},a_{k,1})\right],\label{eq:violation}
\end{align}
where the inequality is true due to the fact $Z_1=0.$ In Lemma \ref{le:qk-qpi-bound}, we will establish an upper bound on the estimation error of $\mathbb{E}\left[\sum_{k=1}^K\left\{C_{k,1}-C_{1}^{\pi_k}\right\} (x_{k,1},a_{k,1})\right].$

Next, we study the moment generating function of $Z_T,$ i.e. $\mathbb{E}\left[e^{rZ_T}\right]$ for some $r>0.$ In Lemma \ref{le:zk-bound}, based on a Lyapunov drift analysis of this moment generating function and Jensen's inequality, we will establish the following upper bound on $Z_T$ that holds for any $1\leq T\leq K^{1-\alpha}\bc,$  
\begin{align}
 \mathbb{E}[ Z_{T}]  \leq \frac{100(H^4\iota+\tilde{b}^2H^2)}{\delta}\log \left( \frac{16(H^2\sqrt{\iota}+\tilde{b}H)}{\delta} \right) +\frac{4H^2B^c}{K\delta}+\frac{4H^2B^c}{\eta K^\alpha\delta} + \frac{4\eta(\sqrt{H^2\iota}+2H^2\tilde{b}) }{\delta} . \label{eq:Z_T}
\end{align} 
Substituting the results from Lemma \ref{le:qk-qpi-bound}  and \eqref{eq:Z_T} into \eqref{eq:violation}, using the choice that $\epsilon = \frac{8\sqrt{SAH^6\iota^3}{B}^{1/3}}{ K^{0.2}},$ we can easily verify that when $K \geq  \max \left\{ \left(\frac{16\sqrt{SAH^6\iota^3}{B}^{1/3}}{\delta} \right)^5, e^{\frac{1}{\delta}}\right\} ,$ we have 
\begin{align}
  \cV(K)\leq &\frac{100(H^4\iota+\tilde{b}^2H^2)K^{0.6} }{\delta{B}^{2/3}}\log{\frac{16(H^2\sqrt{\iota}+H\tilde{b})}{\delta}} - \sqrt{SAH^6\iota^3}K^{0.8}{B}^{\frac{1}{3}} \leq 0.  
\end{align}

\section{UNKNOWN VARIATION BUDGETS}\label{sec:no-budget}
The design of the Algorithm \ref{alg:triple-Q} relies on the knowledge of the total variation budget $B$ to set the frame size to be ${K^\alpha}/{B^{c}}.$ When an upper bound on the total variation budget is not given, we propose the Algorithm \ref{alg:triple-Q-double} that adaptively learns the variation budget $B$ based on the ``Bandit over Bandit'' algorithm \citep{CheSimDav_22}. Algorithm \ref{alg:triple-Q-double} uses an outer loop ``bandit algorithm'' as a master to learn the true value $B,$ and use the inner loop Algorithm \ref{alg:triple-Q} to learn the optimal policy. We first need to divide total $K$ episodes into $\frac{K}{W}$ epochs, which contain $W=K^\zeta$ episodes.  Each epoch contains multiple frames. In each epoch, we run an instance of Algorithm~\ref{alg:triple-Q}. Given a candidate set $\mathcal{J}$ of the total budget $B,$  we choose ``arms'' (estimated budget) using the bandit adversarial bandit algorithm Exp3 \citep{AueBiaFre_03}. If the optimal ``arm'' from the candidate $\mathcal{J}$ can be learned efficiently, we expect that the cumulative reward and utility collected under that arm should be close to the performance of using the best-fixed candidate (closest to true Budget) from $\mathcal{J}$ in hindsight. 
\begin{algorithm*}[ht]
	\SetAlgoLined
	Choose $W=K^{5/9},\mathcal{J}$ defined in Eq. \eqref{eq:setj} $, \gamma_0=\min\left\{1,\sqrt{\frac{(K/W)\log{(K/W)} }{(e-1)KH} }   \right\} , \lambda=1/9$ \;
	Initialize weights of the bandit arms $s_1(j)=1,\forall j=0,1,\dots,J $ \;
	\For{epoch  $i = 1,\dots, \frac{K}{W} $ }{
	 Update $p_i(j)\leftarrow (1-\gamma_0)\frac{s_i(j)}{\sum_{j'=0}^Js_i(j')}+\frac{\gamma_0}{J+1} ,\forall j=0,1,\dots,J $ \;
	 Draw  an arm $A_i\in [J]$ randomly according to the probabilities $p_i(0),\dots,p_i(J)$ \;
	 Set  the estimated budget $B_i \leftarrow  \frac{K^{1/3}W^{\frac{A_i}{J}} }{\Delta^{3/2} W  }  $ \;
	 Run a new instance of Algorithm \ref{alg:triple-Q} for $W$ episodes with parameter value $B\leftarrow B_i,\tilde{b} = B_i^{1-c}K^{\alpha-1}$\;
	 Observe the cumulative reward $R_i$ and utility $G_i.$\; 
	\For{arm j=0,1,\dots,J}{
	    $\hat{R}_i(j) =  \left\{
\begin{aligned}
   (&  G_i/K^\lambda )  I_{\{j=A_i \}} / (WH(1+1/K^\lambda)p_i(j))  & \text{ if } G_i < W\rho   \\
    (& R_i + G_i/K^\lambda) I_{\{j=A_i \}} / (WH(1 +1/K^\lambda) p_i(j))  & \text{ if } G_i \geq W\rho 
\end{aligned}
\right. $ \tcp*{\small normalization } 
		$s_{i+1}\leftarrow s_i(j) \exp(\gamma_0 \hat{R}_i(j) /(J+1) ) $\;
		}
	}
	\caption{{Double Restart  Non-stationary Triple-Q}}\label{alg:triple-Q-double}
\end{algorithm*}
We remark that although the ``Bandit over Bandit'' approach is well studied in both unconstrained non-stationary bandit and RL, however, adopting it in CMDPs is nontrivial and new. We now describe the main challenge in adapting the idea to the constrained scenario and how we overcome the challenge. 

 In particular, given a choice of arm $B_i$ in the unconstrained version, one considers the cumulative reward $R_i(B_i)$ over the epoch $W$ to guide the EXP-3 algorithm towards selecting the optimal arm. The cumulative reward proves to be enough for the unconstrained case, as the optimal arm would correspond to close to the true budget. This can be reflected as the following regret decomposition,
\begin{align}
	\cR (K) 
	= & \mathbb E\left[\sum_{k=1}^K V_{k,1}^{\pi_k^*}(x_{k,1})- \sum_{i=1}^{K/W} R_i(\hat{B}) \right]  \label{eq:decouple-1} \\
	 +& \mathbb E\left[\sum_{i=1}^{K/W} R_i(\hat{B})-\sum_{i=1}^{K/W} R_i({B_i}) \right],\label{eq:decouple-2}
\end{align}
where $\hat{B}$ is the optimal candidate from $\mathcal{J}$  (i.e., the true budget). We can show that the term \eqref{eq:decouple-1} can be bounded since this corresponds to regret when the true budget is known (which we have already bounded). However, the problem becomes that how to bound the term \eqref{eq:decouple-2}. In the unconstrained case, one can employ the result of the EXP-3 algorithm to bound that. The main challenge in extending the above approach to the CMDP is that considering only the reward may lead to a larger violation, since we need to balance both the reward and utility. Thus, one needs to judiciously select the reward based on the total observed reward and utility corresponding to a drawn arm so that the EXP-3 algorithm can choose the arm closest to the optimal one. The natural idea is to set the reward to zero if the observed utility over the epoch does not satisfy the constraint, i.e., if $G_i(B_i)$ is the cumulative utility received after selecting the arm $B_i$, then one can set
\begin{align}
\left\{
   \begin{aligned}
   & \hat{R}_i(B_i)= 0  & \text{ if } G_i(B_i) < W\rho   \\
    & \hat{R}_i(B_i)= R_i(B_i)  & \text{ if } G_i(B_i) \geq W\rho.
\end{aligned}\right. 
\end{align}
Even though it is intuitive, it is not sufficient as it does not distinguish between small and large violation. Thus, we consider the following bandit reward function
\begin{align}
\left\{
   \begin{aligned}
   &  \hat{R}_i(B_i) =\frac{G_i(B_i)}{ K^{\lambda}} &\text{ if } G_i(B_i) < W\rho   \\
    & \hat{R}_i(B_i) = R_i(B_i) + \frac{G_i(B_i)}{K^\lambda} &\text{if } G_i(B_i) \geq W\rho.
\end{aligned}\right. \label{eq:bandit-re}
\end{align}
If $G_i(B_i)<W\rho$, then choosing the arm $B_i$ may lead to violating the constraint, hence, we penalize such arm. On the other hand, if $G_i(B_i)\geq W\rho$, the arm may lead to a feasible policy. We thus consider the reward as $R_i(B_i)+G_i(B_i)/K^{\lambda}$, i.e., the reward is dominated by the accumulated reward. However, the accumulated utility is also considered (albeit with a weight $1/K^{\lambda}$). Note that since $\lambda>0$, the weight factor is small as the main focus is to maximize the reward when the constraint is satisfied. Later, we show that how we select $\lambda$ to balance the regret and the violation. Hence, the weight factor is critical in obtaining sub-linear regret and zero violation.

Next we present a lemma to show the upper bound of the bandit algorithm using our designing of the bandit reward function \eqref{eq:bandit-re}. The proofs can be found in the supplementary materials (Section \ref{ap:prove-non-bandit}).
\begin{lemma}\label{le:non-bandit-c}
Let $R_i(B_i)(G_i(B_i))$ be the cumulative reward(utility) collected in epoch $i$ by any learning algorithm after running for $W$ episodes with the estimate value $B_i$ chosen using the Exp3 bandit algorithm. If we have $\mathbb{E}[ G_i(\hat{B})]\geq W\rho $ then we can obtain
	\begin{align*}
	\mathbb{E}\left[ \sum_{i=1}^{K/W}(R_i(\hat{B}) -R_i({B_i}))  \right] = &\tilde{\mathcal{O}}( H\sqrt{KW}+HK^{1-\lambda} )\\
 \mathbb{E}\left[ \sum_{i=1}^{K/W}  G_i(\hat{B}) - G_i({B_i})  \right] = & \tilde{\mathcal{O}}(HK^\lambda\sqrt{KW}).
	\end{align*}
\end{lemma}
Note that the above lemma bounds (\ref{eq:decouple-2}). Further, it also bounds the utilities for the choice of $\hat{B}$ and $B_i$ which will be useful to obtain violation.

Next, we will formally define the set $\cJ$. Subsequently, we will present the results of using ``bandit over bandit'' with our designing bandit reward function on the Algorithm \ref{alg:triple-Q} for the tabular setting. Then we will discuss how to apply it to the linear function approximation setting. We define set $\mathcal{J}$ as 
\begin{align}
    \mathcal{J} =  \left\{ \frac{K^{1/3}}{{\Delta^{3/2}}W }, \frac{K^{1/3}W^{\frac{1}{J}}}{{\Delta^{3/2}}W },\dots, \frac{K^{1/3}W}{{\Delta^{3/2}}W }  \right\}, \label{eq:setj}
\end{align}
as the candidate value for $B$ and we can see that $\vert \mathcal{J}\vert = \log(W)+1=J+1,$ where $\Delta = \left(\frac{40\sqrt{SAH^6\iota^3}}{\delta} \right)^2.$ After an estimated budget $B_i$ for each epoch $i$ is selected. Then we run a new instance of Algorithm \ref{alg:triple-Q} for consecutive $W=K^\zeta$ episodes. Each epoch contains $W B_i^{c}/{K^{\alpha\zeta}} $ frames. We remark here that when using the Algorithm \ref{alg:triple-Q} we need a local budget information, but under assumption $K^{1-\alpha}B^c\tilde{b} \leq B,$ we can simply choose $\tilde{b} = B_i^{1-c}K^{\alpha-1}$ with an estimated $B_i.$ The following Theorem states that the Algorithm \ref{alg:triple-Q-double} achieves a sublinear regret and zero constraint violation without the knowledge of the total variation budget $B.$ Detailed proofs are deferred to supplementary materials (Section \ref{ap:tirple-q-double}).
\begin{theorem} \label{the:double-triple-q}
Algorithm \ref{the:double-triple-q} achieves the following regret and constraint violation bounds with no prior knowledge of the total variation budget $B$ when $K = \Omega  \left(( \frac{40\sqrt{SAH^6\iota^3}{B}^{1/3}}{\delta})^9\right),$ and $K\geq e^{\frac{1}{\delta}}$:  
\begin{align*}
	\cR (K)  = \tilde{\cal O}(H^4S^\frac{1}{2}A^{\frac{1}{2}}{B}^{\frac{1}{3}} K^{8/9})\quad\quad 	\cV (K) =  0
\end{align*}
\end{theorem}

\section{LINEAR CMDPs}\label{sec:linear_cmdp}
In this section, we consider linear CMDP which can potentially model infinite state space. In particular, we consider reward, utility, and transition probability can be modeled as linear in known feature space \cite{GhoZhoShr_22}. The formal definition is given below
\begin{definition}\label{defn:linearmdp}
The CMDP is a linear MDP with feature map $\phi: \mathcal{S}\times\mathcal{A}\rightarrow\R^{d}$, if for any $h$ and $k$, there exists $d$ {\em unknown} signed measures $\mu_{k,h}=\{\mu^1_{k,h},\ldots,\mu^d_{k,h}\}$ over $\mathcal{S}$  such that  any $(x,a,x^{\prime})\in \mathcal{S}\times \mathcal{A}\times \mathcal{S}$, 
\begin{align}
\bP_{k,h}(x^{\prime}|x,a)=\langle\phi(x,a),\mu_{k,h}(x^{\prime})\rangle
\end{align}
and there exists (unknown) vectors $\theta_{k,r,h}$, $\theta_{k,g,h}\in \R^d$ such that for any $(x,a)\in \mathcal{S}\times\mathcal{A}$,
\begin{align*}
r_{k,h}(x,a)=&\langle\phi(x,a),\theta_{k,r,h}\rangle,\\ g_{k,h}(x,a)=&\langle\phi(x,a),\theta_{k,g,h}\rangle
\end{align*}
Without loss of generality, we assume $||\phi(x,a)||_2\leq 1$, $\max\{||\mu_{k,h}||_2,||\theta_{k,r,h}||_2,||\theta_{k,g,h}||_2\}\leq \sqrt{d}$.
\end{definition}
We adapt the stationary version of the linear CMDP in the non-stationary setup by considering time-varying $\mu_{k,h}$, and $\theta_{k,j,h}$. It extends the non-stationary unconstrained linear MDP  \cite{ZhoCheVar_20} to the constrained case. We remark that despite being linear, $\bP_{k,h}(\cdot|x,a)$ can still have infinite degrees of freedom since $\mu_{k,h}(\cdot)$ is unknown. Note that \cite{DinWeiYan_20,DingLav_22} studied another related concept known as linear kernel MDP.  In general, linear MDP and linear kernel MDPs are two different classes of MDP \cite{ZhoHeGu_21}.

Similar to  budget variations in the tabular case, we define the total (global) variations on $\mu_{k,h}$ and $\theta_{k,j,h}$ for $j=r,g$ and the total variations as
\begin{align}
    B_{j} = &\sum_{k=2}^K\sum_{h=1}^H||\theta_{k,j,h}-\theta_{k-1,j,h}||_2,\\
    B_{p} = &\sum_{k=2}^K\sum_{h=1}^H||\mu_{k,h}-\mu_{k-1,h}||_{F},
\end{align}
and $B=B_r+B_g+B_p$ is the global budget variation.

\textbf{Algorithm}: \cite{GhoZhoShr_22} proposed an algorithm for the stationary setup. It is a primal-dual adaptation of the unconstrained version \cite{DingLav_22} . However, there are some key differences with respect to the unconstrained case. For example, instead of greedy policy with respect to the combined state-action value function one needs the soft-max policy. We adapt the algorithm in the non-stationary case (Algorithm~\ref{algo:model_free} in the supplementary materials~\ref{proof:linear_cmdp}). In particular, we employ the restart strategy to adapt to the non-stationary environment.  We divide the total episodes $K$ in $K/D$ frames where each frame consists of $D$ episodes. We employ the algorithm proposed in \cite{GhoZhoShr_22} at each frame. Note that such type of restart strategy is already proposed for the unconstrained version as well \citep{ZhoCheVar_20}. However, the algorithm for the constrained linear MDP differs from the unconstrained version, thus, the analysis also differs.

\textbf{Tabular v.s. Linear Approximation}: We remark that although linear CMDPs include tabular CMDPs as a special case \citep{JinAllBub_18}. Directly applying the algorithm to a tabular CMDP will result in higher memory and computational complexity than Nonstationary Triple-Q. 

We now flesh out Algorithm~\ref{algo:model_free} for the tabular case which will clarify the memory and computational requirement. We can revert back to the tabular case by setting $\phi(s,a)=e_{s,a}$ where $e_{s,a}$ is a $d$-dimensional (here $d=|\mathcal{S}||\mathcal{A}|$) vector where $e_{s,a}=1$ for state-action pair $(s,a)$ and zero for other values of state and action. The $w_{r,h}$ vector update becomes as the following 
\begin{align*}
  w_{r,h}^k(x,a) 
  =\dfrac{1}{(n_h^k(x,a)+\lambda)}\sum_{\tau=1}^{n_h^k(x,a)}(r_h(x_h^{\tau},a_h^{\tau})+V_{r,h+1}^k(x_{h+1}^{\tau}))  
\end{align*}
where $n_h^k(x,a)$ is the number of times the state-action pair $(x,a)$ has been encountered at step $h$ till episode $k$. The $Q_{r,h}^k$ update will be
\begin{align*}
  Q_{r,h}^k(x,a) = 
  \min\{\langle w_{r,h}^k(x,a),\phi(x,a)\rangle +\beta \sqrt{ 1/(n_h^k(x,a)+\lambda)},H\}.  
\end{align*}
In a similar manner, we can update $Q_{g,h}^k$. Note that we need to update this table for every state-action pair at each step and use all the samples generated so far. Using this, one can update $V_{r,h}^k$, and $V_{g,h}^k$ using the soft-max policy.

{We further remark that if we maintain  $n_h^k(x,a,\tilde{x})$ to be the number of times the state-action-next state $(x,a, \tilde{x})$ has been encountered at step $h$ till episode $k$. Then 
\begin{align*}
   w_{r,h}^k(x,a)=\dfrac{1}{(n_h^k(x,a)+\lambda)}
   \cdot\left(n_h^k(x,a) r_h(x,a)+\sum_{\tilde{x}} n_h^k(x,a,\tilde{x})
V_{r,h+1}^k(\tilde{x})\right). 
\end{align*}
 In this case, we do not need to go through all samples at each iteration and do not even need to store the old samples. The memory  complexity of maintaining the counts $\{n_h(x,a,\tilde{x})\}$ is $O\left(H|{\cal S}|^2|{\cal A}|\right)$, which is higher than the memory complexity and computational complexity of non-stationary Triple-Q, which are $O\left(H|{\cal S}||{\cal A}|\right),$ but matches model-based algorithms for tabular settings.}

\subsection{Main Results}
\begin{theorem} \label{thm:linear_cmdp}
With $D=B^{-1/2}d^{1/2}K^{1/2}H^{-1/2},$ Algorithm \ref{algo:model_free} achieves the following regret and constraint violation bounds:
\begin{align*}
    & \cR (K)=\mathcal{O}(\frac{1+\delta}{\delta}K^{3/4}H^{9/4}d^{5/4}B^{1/4}\iota)\nonumber\\
    &\cV (K)=\dfrac{2(1+\xi)}{\xi}\mathcal{O}( K^{3/4}H^{9/4}d^{5/4}B^{1/4}\iota)
\end{align*}
where $\iota=\log(2\log(|\cA|)dT/p)$, and $\xi=2H/\delta$. 
\end{theorem}
Our algorithm provides a regret guarantee of $\tilde{\cO}(d^{5/4}K^{3/4}H^{9/4}B^{1/4})$ and the same order on violation. $\xi$ arises since we truncate the dual variable at $\xi$ in Algorithm~\ref{algo:model_free}. Note that regret and violation only scale with $d$ rather than the cardinality of the state space.

Compared to \cite{DingLav_22}, which also considers linear function approximation (however, it considers linear kernel CMDP rather linear CMDP), we improve their result by a factor of $H^{\frac{1}{4}}$. We also improve the dependence on $B$ and $d$. Further, we do not need to know the total variations in the optimal solution ($B_*$), unlike in \cite{DingLav_22}.
The algorithm proposed in \cite{DingLav_22} is a model-based policy-based algorithm; ours is a model-free value-based algorithm. Thus, our  algorithm enjoys an easy implementation and improved computation efficiency since it does not estimate the next step expected value function as in \cite{DingLav_22}, which requires an  integration oracle to compute a $d$-dimensional integration at every step.

\textbf{Zero Violation}: Similar to the tabular setup, we obtain zero violation by considering a tighter optimization problem. In particular, if we consider $\epsilon$-tighter constraint where $\epsilon=\min\{\dfrac{2(1+\xi)}{\xi}\tilde{O}(d^{5/4}B^{1/4}H^{9/4}K^{3/4})/K,\delta/2\}$,  the violation is $0$. Thus, if $K^{1/4} \geq \dfrac{4(1+\xi)}{\xi \delta}\tilde{O}(d^{5/4}B^{1/4}H^{9/4})$, we could obtain zero violation while maintaining the same order of regret with respect to $K$.

\begin{remark}
Our algorithm \ref{algo:model_free} doesn't require the information of the local budget. In the unconstrained version \cite{ZhoCheVar_20} achieves $\tilde{\mathcal{O}}(T^{2/3})$ regret if local budget variation is known. We can also achieve $\tilde{\mathcal{O}}(T^{2/3})$ regret and $\tilde{\mathcal{O}}(T^{2/3})$ violation if we assume local budget variation is known. 
\end{remark}

\subsection{Without knowing the variation budget}
Our idea of designing the ``bandit over bandit'' algorithm can still be applied to the linear CMDPs, We propose an algorithm (see Algorithm \ref{algo:model_free_unknown} in supplementary materials), which can achieve the following result. Details proofs can be found in supplementary materials (Section \ref{ap:proof-linear-non}).
\begin{theorem}\label{the:linear-non}
Let $D=B^{-1/2}d^{1/2}K^{1/2}H^{-1/2}, W = \sqrt{K}$, Algorithm \ref{algo:model_free_unknown} achieves the following regret and constraint violation bounds:
\begin{align*}
    & \cR(K)=\mathcal{O}(\frac{1+\delta}{\delta}K^{7/8}H^{9/4}d^{5/4}B^{1/4}\iota)\nonumber\\
    & \cV(K)=\dfrac{2(1+\xi)}{\xi}\mathcal{O}(\frac{1+\delta}{\delta}K^{7/8}H^{9/4}d^{5/4}B^{1/4}\iota)
\end{align*}
\end{theorem}
We can further achieve zero constraint violation by choosing $$\epsilon=\min\{ \dfrac{3(1+\xi)}{\xi}\tilde{O}((1+1/\delta)d^{5/4}\hat{B}^{1/4}H^{9/4}K^{1-\zeta/4})/K,\delta/2\},$$ when $K^{8} \geq \dfrac{6(1+\xi)}{\xi \delta}\tilde{O}(d^{5/4}B^{1/4}H^{9/4}).$ 

We also provide an approach based on convex optimization to further reduce the order from $\tilde{\cO}(K^{7/8})$ to $\tilde{\cO}(K^{3/4}),$ for both regret and violation see Section \ref{ap:method2-non-linear} in the supplementary materials for details.

\section{SIMULATION}
We compare Algorithm \ref{alg:triple-Q} with two baseline algorithms: an algorithm \citep{MaoZhaZhu_20} for {\em non-stationary} MDPs, and an algorithm \citep{WeiLiuYin_22} for {\em stationary} constrained MDPs for a grid-world environment. From the simulation results, we observe that our Algorithm \ref{alg:triple-Q} can quickly learn a well-performed policy while satisfying the safety constraint even when the MDP varies, while other methods all fail to satisfy the constraints. All the details can be found in supplementary materials (Section \ref{sec:sim}).

\section{CONCLUSION}
We have studied model-free reinforcement learning algorithms in non-stationary episodic CMDPs. In particular, we consider two settings -- one is computationally less intensive for the tabular setting, and another one is computationally more intensive but can be applied to a more general linear approximation setup. We have further presented a general framework for applying any algorithms with zero constraint violation to a more practical scenario where the total variation budget is unknown. Whether we can tighten the bounds for model-free algorithms remains an important future research direction. Whether we can design an approach for using any learning algorithms for CMDPs in a non-stationary environment without the knowledge of the budget also constitutes a future research direction.

\subsubsection*{Acknowledgements}
We thank the anonymous paper reviewers for their insightful comments. The work of Honghao Wei and Lei Ying is supported in part by NSF under
grants 2001687, 2112471, 2134081, and 2228974.  This work of Ness Shroff and Arnob Ghosh has been partly supported by NSF grants NSF AI Institute (AI-EDGE) 2112471, CNS-2106933, 2007231, CNS-1955535, and CNS-1901057, and in part by Army Research Office under Grant W911NF-21-1-0244. The work of Xingyu Zhou is supported in part by NSF under grants NSF CNS-2153220.

\bibliographystyle{apalike}

\onecolumn
\newpage
\appendix
\section{NOTATION TABLE}
The notations used throughout this paper are summarized in Table \ref{ta:notations}.
\begin{table}[ht]
	\caption{Notation Table}
	\label{ta:notations}
	\vskip 0.10in
	\begin{center}
			\begin{tabular}{c|l}
				\toprule
				Notation & Definition  \\
				\midrule
				$ K $    &  total number of episodes\\
				\hline
				$ S$    &   number of states\\
				\hline
				$ A$    &   number of actions\\
				\hline
				$ H$    &   length of each episode\\
				\hline
				$B $    &   total variation budget\\
				\hline
				$ W$    &   number of episodes in one epoch.\\
				\hline
				$D$ & number of episodes in one frame.\\
				\hline
				$ B_i$    &  arm selected by the bandit algorithm.\\
				\hline
				$ \alpha_t$    &  learning rate\\
				\hline
				$R_i(B_i) (G_i(B_i)) $    &   reward/utility collected at the epoch $i$ under selected estimate value $B_i$     \\
				\hline
				$Q_{k,h}(x,a) (C_{k,h}(x,a)) $ & estimated reward (utility) Q-function at step $h$ in episode $k$ \\
				\hline
				$Q_{k,h}^\pi (x,a) (C_{k,h}^\pi (x,a))  $ &  reward (utility) Q-function at step $h$ in episode $k$ under policy $\pi.$\\
				\hline
				$V_{k,h}(x) (W_{k,h}(x)) $ & estimated reward (utility) value-function at step $h$ in episode $k$ \\
				\hline
				$V_{k,h}^\pi (x) (W_{k,h}^\pi (x) )$ & reward (utility) value-function at step $h$ in episode $k$ under policy $\pi$\\
				\hline
				$F_{k,h}(x,a)$ & 
				$F_{k,h} (x,a)= Q_{k,h}(x,a) + \frac{Z_k}{\eta} C_{k,h}(x,a).$  \\
				\hline
				$U_{k,h}(x)$ &
				$U_{k,h} (x)=V_{k,h}(x) + \frac{Z_k}{\eta} W_{k,h}(x).$  \\
				\hline
				$r_{k,h}(x,a) (g_{k,h}(x,a))$ & reward (utility) of (state, action) pair $(x,a)$ at step $h$ in episode $k$ \\
				\hline
				$N_{k,h}(x,a)$ & number of visits to $(x,a)$ when at step $h$ in episode $k$ (not including $k$) \\ 
				\hline
				$Z_k$ & dual estimation (virtual queue) in episode $k.$\\
				\hline
				$q_{k,h}^*$ & The optimal solution to the LP \eqref{eq:lp_app} in episode $k$ \\
				\hline
				${q}_{k,h}^{\epsilon,*}$ & optimal solution to the tightened LP \eqref{eq:lp-epsilon_app}  in episode $k$ \\
				\hline
				${\pi}_{k}^{*}$ & optimal policy in episode $k$ \\
				\hline
				$\delta$ & Slater's constant.\\
				\hline
				$d$ & dimension of the feature vector.\\
				\hline
				$b_t$ & the UCB bonus for given $t$\\
				\hline
				$\mathbb{I}(\cdot)$ & indicator function\\
				\hline
				$\bP_{k,h} $ & transition kernel at step $h$ in episode $k$\\
				\hline
				$\hat{\bP}_{k,h} $ & empirical transition kernel at step $h$ in episode $k$\\
				\hline
				$B_r,B_g,B_p$ & variation budget for reward, utility, and transition\\
				\hline
				$B_r^{(T)},B_g^{(T)},B_p^{(T)}$ & variation budget for reward, utility, and transition in frame $T$\\
				\hline
				$\phi(x,a)$ & feature map for the linear MDP\\
				\hline
				$\theta_{k,r,h},\theta_{k,g,h},\mu_{k,h}$ & underlying parameters for the linear MDP \\
				\bottomrule
			\end{tabular}
		\end{center}
	\end{table}
	
	\section{AUXILIARY LEMMAS}
	In this section, we state several lemmas that used in our analysis.
	The first lemma establishes some key properties of the learning rates used in Non-stationary Triple-Q.  The proof closely follows the proof of Lemma 4.1 in \cite{JinAllBub_18}. 
	\begin{lemma}\label{le:lr}
		Recall that the learning rate used in Triple-Q is $\alpha_t = \frac{\chi+1}{\chi+t},$ and \begin{equation}
			\alpha_t^0=\prod_{j=1}^t(1-\alpha_j)\quad \hbox{and}\quad \alpha_t^i=\alpha_i\prod_{j=i+1}^t(1-\alpha_j). \label{le:lr-def}
		\end{equation} 
		The following properties hold for $\alpha_t^i:$ 
		\begin{enumerate}[label=(\alph*)]
			\item $\alpha_t^0=0$ for $t \geq 1, \alpha_t^0=1$ for $t=0.$ \label{le:lr-a}
			\item $\sum_{i=1}^t\alpha_t^i=1$ for $t\geq 1,$ $\sum_{i=1}^{t}\alpha_t^i=0$ for $t=0.$\label{le:lr-b}
			\item $\frac{1}{\sqrt{\chi+t}}\leq \sum_{i=1}^t \frac{\alpha_t^i}{\sqrt{\chi + i}}\leq \frac{2}{\sqrt{\chi+t}}.$ \label{le:lr-c}
			\item $\sum_{t=i}^\infty\alpha_t^i=1+\frac{1}{\chi}$ for every $i\geq 1.$\label{le:lr-d}
			\item $ \sum_{i=1}^t (\alpha_t^i)^2\leq \frac{\chi+1}{\chi+t}$ for every $t\geq 1.$ \label{le:lr-e}
		\end{enumerate}\hfill{$\square$}
	\end{lemma}
	\begin{proof} The proof of \ref{le:lr-a} and \ref{le:lr-b} are straightforward by using the definition of $\alpha_t^i$. The proof of \ref{le:lr-d} is the same as that in \cite{JinAllBub_18}.
		
		\ref{le:lr-c}: We next prove \ref{le:lr-c} by induction. 
		
		For $t=1,$ we have $\sum_{i=1}^t\frac{\alpha_t^i}{\sqrt{\chi+i}}=\frac{\alpha_1^1}{\sqrt{\chi+1}}=\frac{1}{\sqrt{\chi+1}},$ so \ref{le:lr-c} holds for $t=1$.
		
		Now suppose that \ref{le:lr-c} holds for $t-1$ for $t\geq 2,$ i.e. $$\frac{1}{\sqrt{\chi+t-1}}\leq \sum_{i=1}^{t-1} \frac{\alpha_t^i}{\sqrt{\chi + i-1}}\leq \frac{2}{\sqrt{\chi+t-1}}.$$
		From the relationship $\alpha_t^i = (1-\alpha_t)\alpha_{t-1}^i$ for $i=1,2,\dots,t-1,$ we have $$\sum_{i=1}^t\frac{\alpha_t^i}{\sqrt{\chi + i}} =\frac{\alpha_t}{\sqrt{\chi+t}}+(1-\alpha_t)\sum_{i=1}^{t-1}\frac{\alpha_{t-1}^i}{\sqrt{\chi+i}}.$$
		
		Now we apply the induction assumption. To prove the lower bound in \ref{le:lr-c}, we have
		$$\frac{\alpha_t}{\sqrt{\chi+t}}+(1-\alpha_t)\sum_{i=1}^{t-1}\frac{\alpha_{t-1}^i}{\sqrt{\chi+i}}\geq \frac{\alpha_t}{\sqrt{\chi+t}} + \frac{1-\alpha_t}{ \sqrt{\chi +t- 1}}\geq \frac{\alpha_t}{\sqrt{\chi+t}} + \frac{1-\alpha_t}{ \sqrt{\chi+t}}\geq \frac{1}{\sqrt{\chi+t}}.$$
		To prove the upper bound in \ref{le:lr-c}, we have
		\begin{align}
			\frac{\alpha_t}{\sqrt{\chi+t}}+(1-\alpha_t)\sum_{i=1}^{t-1}\frac{\alpha_{t-1}^i}{\sqrt{\chi+i}} \leq & \frac{\alpha_t}{\sqrt{\chi+t}} + \frac{2(1-\alpha_t)}{\sqrt{\chi+t-1}} = \frac{\chi+1}{(\chi+t)\sqrt{\chi+t}} + \frac{2(t-1)}{(\chi+t)\sqrt{\chi+t-1}},\nonumber\\
			=& \frac{1-\chi-2t}{(\chi+t)\sqrt{\chi+t}}+ \frac{2(t-1)}{(\chi+t)\sqrt{\chi+t-1}} +\frac{2}{\sqrt{\chi+t}} \nonumber\\
			\leq & \frac{-\chi-1}{(\chi+t)\sqrt{\chi+t-1}}+\frac{2}{\sqrt{\chi+t}} \leq \frac{2}{\sqrt{\chi+t}}.
		\end{align}
		\ref{le:lr-e} According to its definition, we have 
		\begin{align}
			\alpha_t^i  =& \frac{\chi+1}{i+\chi}\cdot \left(\frac{i}{i+1+\chi}\frac{i+1}{i+2+\chi}\cdots \frac{t-1}{t+\chi} \right)\nonumber\\
			= & \frac{\chi+1}{t+\chi}\cdot \left(\frac{i}{i+\chi}\frac{i+1}{i+1+\chi}\cdots \frac{t-1}{t-1+\chi} \right) \leq \frac{\chi+1}{\chi+t}.
		\end{align}
		Therefore, we have $$\sum_{i=1}^t (\alpha_t^i)^2 \leq [\max_{i\in[t]}\alpha_t^i]\cdot \sum_{i=1}^t\alpha_t^i\leq \frac{\chi+1}{\chi+t},$$ because $\sum_{i=1}^t\alpha_t^i=1.$
	\end{proof}
	\begin{lemma}\label{le:q1-bound}
		For any $(x,a,h,k)\in\mathcal{S}\times\mathcal{A}\times[H]\times[K],$ we have the following bounds on $Q_{k,h}(x,a)$ and $C_{k,h}(x,a):$
		\begin{align*}
			0\leq Q_{k,h}(x,a)\leq H^2(\sqrt{\iota} + 2\tilde{b})\\
			0\leq C_{k,h}(x,a)\leq H^2(\sqrt{\iota} + 2\tilde{b}).
		\end{align*}
	\end{lemma}
	
	\begin{proof}
		We first consider the last step of an episode, i.e. $h=H.$ 
		Recall that $V_{k, H+1}(x)=0$ for any $k$ and $x$ by its definition and $Q_{0,H}=H\leq H(\sqrt{\iota}+2\tilde{b}).$ Suppose $Q_{k',H}(x,a)\leq H(\sqrt{\iota} + 2\tilde{b})$ for any $k'\leq k-1$ and any $(x,a).$ Then,
		\begin{align}
			{Q}_{k,H}(x,a) & = (1-\alpha_t)Q_{k_t,H}(x,a) + \alpha_t\left(r_{k,H}(x,a)+b_t + 2H\tilde{b}\right) \\
			& \leq  \max\left\{H\sqrt{\iota} + 2\tilde{b}H, 1+\frac{H\sqrt{\iota}}{4} + 2H\tilde{b}\right\}\leq H\sqrt{\iota} + 2\tilde{b}H,
		\end{align}
		where $t=N_{k,H}(x,a)$ is the number of visits to state-action pair $(x,a)$ when in step $H$ by episode $k$ (but not include episode $k$) and $k_t$ is the index of the episode of the most recent visit.  Therefore, the upper bound holds for $h=H.$ 
		Note that $Q_{0,h}=H\leq H(H-h+1)(\sqrt{\iota} + 2\tilde{b}).$ 
		Now suppose the upper bound holds for $h+1,$ and also holds for $k'\leq k-1.$ Consider step $h$ in episode $k:$  
		\begin{align*}
			{Q}_{k,h}(x,a)= &(1-\alpha_t)Q_{k_t,  h}(x,a) + \alpha_t\left(r_{k,h}(x,a)+V_{k_t, {h}+1}(x_{k_t, {h}+1})+b_t + 2\tilde{b}H\right),
		\end{align*} where $t=N_{k,{h}}(x,a)$ is the number of visits to state-action pair $(x,a)$ when in step ${h}$ by episode $k$ (but not include episode $k$) and $k_t$ is the index of the episode of the most recent visit.  We also note that $V_{k,h+1}(x)\leq \max_a Q_{k,h+1}(x,a)\leq H(H-h)(\sqrt{\iota}+ 2\tilde{b}).$
		Therefore, we obtain
		\begin{align*}
			{Q}_{k,h}(x,a)\leq& \max \left\{H(H-h+1)(\sqrt{\iota}+ 2\tilde{b}), 1+H(H-h)(\sqrt{\iota}+2\tilde{b})+\frac{H\sqrt{\iota}}{4} + 2\tilde{b}H\right\} \\
			\leq & H(H-h+1)(\sqrt{\iota} + 2\tilde{b}).
		\end{align*} Therefore, we can conclude that $Q_{k,h}(x,a)\leq H^2(\sqrt{\iota}+2\tilde{b})$ for any $k,$ $h$ and $(x,a).$  The proof for $C_{k,h}(x,a)$ is identical. 
	\end{proof}
	
	\begin{lemma}\label{le:u-hoeffding}
		Consider any frame $T,$ any episode $k'.$ Let t=$N_{k,h}(x,a)$ be the number of visits to $(x,a)$ at step $h$ before episode $k$ in the current frame and let $k_1,\dots,k_t < k$ be the indices of these episodes. Under any policy $\pi,$  with probability at least $1-\frac{1}{K^3},$ the following inequalities hold simultaneously for all $(x,a,h,k)\in\mathcal{S}\times\mathcal{A}\times[H]\times [K],$ 
		\begin{align*}
			\left\vert \sum_{i=1}^t\alpha_t^i\left\{(\hat{\mathbb{P}}_{k_i,h}-\mathbb{P}_{k_i,h}) V_{k,h+1}^{\pi}\right\}(x,a)\right\vert \leq  &\frac{1}{4} \sqrt{\frac{H^2\iota(\chi+1)}{(\chi+t)}}, \\
			\left\vert \sum_{i=1}^t\alpha_t^i\left\{(\hat{\mathbb{P}}_{k_i,h}-\mathbb{P}_{k_i,h}) W_{k,h+1}^{\pi}\right\}(x,a)\right\vert \leq  &\frac{1}{4} \sqrt{\frac{H^2\iota(\chi+1)}{(\chi+t)}}.
		\end{align*}
	\end{lemma}
	\begin{proof}
		Without loss of generality, we consider $T=1.$ Fix any $(x,a,h) \in \mathcal{S}\times \mathcal{A}\times \mathcal{H},$ a fixed episode $k,$ and any $n\in [K^\alpha/B^c],$  define
		$$X(n)=\sum_{i=1}^n\alpha_\tau^i\cdot \mathbb{I}_{\{k_i\leq K\}}\left\{(\hat{\mathbb{P}}_{k_i,h}-\mathbb{P}_{k_i,h} )V_{k,h+1}^\pi\right\}(x,a).$$ Let $\mathcal{F}_i$ be the $\sigma-$algebra generated by all the random variables until step $h$ in episode $k_i.$ Then 
		$$\mathbb{E}[X(n+1)\vert \mathcal{F}_n]= X(n) + \mathbb{E}\left[\alpha_\tau^{n+1}\mathbb{I}_{\{k_{n+1}\leq K\}}\left\{(\hat{\mathbb{P}}_{k_{n+1},h}-\mathbb{P}_{k_{n+1},h} )V_{k,h+1}^\pi\right\}(x,a) \vert \mathcal{F}_n\right]=X(n),$$
		which shows that $X(n)$ is a martingale. We also have for $1\leq m \leq n,$
		\begin{align*}
			\vert X(m)-X(m-1)\vert \leq  \alpha_\tau^m \left\vert \left\{(\hat{\mathbb{P}}_{k_{m},h}-\mathbb{P}_{k_m,h} )V_{k,h+1}^\pi\right\}(x,a)\right\vert  \leq  \alpha_\tau^m H
		\end{align*}
		Let $k_i=K+1$ if it is taken for fewer than $i$ times, and let $\sigma = \sqrt{8\log\left(\sqrt{2SAH}K\right)\sum_{i=1}^\tau(\alpha_\tau^iH)^2}.$ Then by applying the Azuma-Hoeffding inequality, we have with probability at least $1-2\exp\left(-\frac{\sigma^2}{2\sum_{i=1}^\tau(\alpha^i_\tau H )^2 }\right)\geq 1-\frac{1}{2S^2A^2H^2K^4},$
		$$ \vert X(\tau)\vert \leq \sqrt{8\log\left(\sqrt{2SAH}K\right)\sum_{i=1}^\tau(\alpha_\tau^i H)^2}\leq  \sqrt{\frac{\iota}{16} H^2\sum_{i=1}^\tau(\alpha_\tau^i)^2}\leq \frac{1}{4}\sqrt{\frac{H^2\iota(\chi+1)}{\chi+\tau}},$$
		Because this inequality holds for any $\tau\in[K]$, it also holds for $\tau=t=N_{k,h}(x,a)\leq K.$ 
		Applying the union bound, we obtain that with probability at least $1-\frac{1}{2SAHK^3}$ the following inequality holds simultaneously for all $(x,a,h,k)\in\mathcal{S}\times\mathcal{A}\times[H]\times[K],$:
		$$\left\vert \sum_{i=1}^t\alpha_t^i\left\{(\hat{\mathbb{P}}_{k_i,h}-\mathbb{P}_{k_i,h}) V_{k,h+1}^{\pi}\right\}(x,a)\right\vert \leq \frac{1}{4} \sqrt{\frac{H^2\iota(\chi+1)}{(\chi+t)}}.$$
		Following a similar analysis, we also have that with probability at least $1-\frac{1}{2SAHK^3}$ the following inequality holds simultaneously for all $(x,a,h,k)\in\mathcal{S}\times\mathcal{A}\times[H]\times[K],$: $$\left\vert \sum_{i=1}^t\alpha_t^i\left\{(\hat{\mathbb{P}}_{k_i,h}-\mathbb{P}_{k_i,h}) W_{k,h+1}^{\pi}\right\}(x,a)\right\vert \leq \frac{1}{4} \sqrt{\frac{H^2\iota(\chi+1)}{(\chi+t)}}.$$
		
		Therefore applying a union bound on the two events we finish proving the lemma.

	\end{proof}

	\section{PROOFS OF THE TECHNICAL LEMMAS}

	\begin{lemma}\label{le:q-diff}
		For any frame $T,$ any $x,a,h$ and any $(T-1)K^\alpha/\bc \leq k_1\leq k_2\leq TK^\alpha/ \bc,$ we have 
		$$\vert Q_{k_1,h}^{\pi}(x,a) - Q_{k_2,h}^{\pi'}(x,a)\vert \leq  H\tilde{b}$$ 
		$$\vert C_{k_1,h}^{\pi}(x,a) - C_{k_2,h}^{\pi'}(x,a)\vert \leq  H\tilde{b}$$ 
	\end{lemma}
	\begin{proof}
		First define $B_{h}^r, B_{h}^g, B_{h}^p $ to be the variation of reward, utility functions and transitions at step $h$ within frame $T.$ 
		\begin{align}
			B_{h}^r & = \sum_{k=(T-1)K^\alpha/\bc +1 }^{TK^\alpha/\bc } \sup_{x,a} \vert r_{k,h}(x,a) - r_{k+1,h}(x,a)\vert \\
			B_{h}^g & = \sum_{k=(T-1)K^\alpha/\bc +1 }^{TK^\alpha/\bc } \sup_{x,a} \vert g_{k,h}(x,a) - g_{k+1,h}(x,a)\vert \\
			B_{h}^p & = \sum_{k=(T-1)K^\alpha/\bc +1 }^{TK^\alpha/\bc } \sup_{x,a} \Vert \mathbb{P}_{k,h}(\cdot\vert x,a) - \mathbb{P}_{k+1,h}(\cdot\vert x,a)\Vert_1 
		\end{align}
		We will prove the following statement by induction.
		$$\vert Q_{k_1,h}^{\pi}(x,a) - Q_{k_2,h}^{\pi'}(x,a)\vert \leq  \sum_{h'=h}^H B_{h'}^r + H\sum_{h'=h}^H B_{h'}^p  $$ 
		For step $H,$ the statement holds because for any $(x,a),$
		\begin{align*}
			\vert Q_{k_1,H}^{\pi}(x,a) - Q_{k_2,H}^{\pi'}(x,a)\vert  =  & \vert r_{k_1,H}(x,a) - r_{k_2,H}(x,a)\vert \\
			\leq  & \sum_{k=k_1}^{k_2-1} \vert r_{k,H}(x,a) - r_{k+1,H}(x,a)\vert \leq B_{H}^r
		\end{align*}
		Now suppose the statement holds for $h+1,$ then
		\begin{align*}
			& Q_{k_1,h}^{\pi}(x,a) - Q_{k_2,h}^{\pi'}(x,a)\\
			= & \mathbb{P}_{k_1,h}V_{k_1,h+1}^\pi (x,a) - \mathbb{P}_{k_2,h}V_{k_2,h+1}^{\pi'} (x,a) + r_{k_1,h}(x,a) - r_{k_2,h}(x,a)\\
			\leq &  \mathbb{P}_{k_1,h}V_{k_1,h+1}^\pi (x,a)-\mathbb{P}_{k_2,h}V_{k_2,h+1}^{\pi'} (x,a)  + B_{h}^r \\
			= & \sum_{x'}\mathbb{P}_{k_1,h}(x'\vert x,a )V_{k_1,h+1}^\pi(x') - \sum_{x'}\mathbb{P}_{k_2,h}(x'\vert x,a )V_{k_2,h+1}^{\pi'}(x') +  B_{h}^r \\
			= & \sum_{x'}\mathbb{P}_{k_1,h}(x'\vert x,a )Q_{k_1,h+1}^\pi(x',\pi_{h+1}(x') ) - \sum_{x'}\mathbb{P}_{k_2,h}(x'\vert x,a )Q_{k_2,h+1}^{\pi'}(x',{\pi'}_{h+1}(x')) +  B_{h}^r 
		\end{align*}
		According to the hypothesis on $h+1,$ we have 
		\begin{align}
			Q_{k_1,h+1}^\pi(x',\pi_{h+1}(x') ) \leq Q_{k_2,h+1}^{\pi'}(x',{\pi'}_{h+1}(x') ) + \sum_{h'=h+1}^H B_{h'}^r + H\sum_{h'=h+1}^H B_{h'}^p
		\end{align}	 
		Therefore
		\begin{align*}
			& Q_{k_1,h}^{\pi}(x,a) - Q_{k_2,h}^{\pi'}(x,a)\\
			\leq & \sum_{x'}\left(\mathbb{P}_{k_1,h}(x'\vert x,a ) - \mathbb{P}_{k_2,h}(x'\vert x,a ) \right)Q_{k_2,h+1}^{\pi'}(x',\pi_{h+1}(x') ) + \sum_{h'=h}^H B_{h'}^r + H\sum_{h'=h+1}^H B_{h'}^p \\
			\leq & \Vert \mathbb{P}_{k_1,h}(\cdot \vert x,a ) - \mathbb{P}_{k_2,h}(\cdot \vert x,a ) \Vert_1 \cdot H + \sum_{h'=h}^H B_{h'}^r + H\sum_{h'=h+1}^H B_{h'}^p  \\
			\leq & B_{h}^p H + \sum_{h'=h}^H B_{h'}^r + H\sum_{h'=h+1}^H B_{h'}^p \\
			\leq &  \sum_{h'=h}^H B_{h'}^r + H\sum_{h'=h}^H B_{h'}^p 
		\end{align*}
		where the last inequality comes from the assumption on $\tilde{b}.$ The same analysis can be applied to $\vert C_{k_1,h}^{\pi}(x,a) - C_{k_2,h}^\pi(x,a)\vert.$ We finish the proof by using the fact that $ \sum_{h'=h}^H B_{h'}^r + H\sum_{h'=h}^H B_{h'}^p  \leq  H\tilde{b}.$
	\end{proof}

	\begin{lemma}\label{le:qk-qpi-relation}
		With probability at least $1-\frac{1}{K^3},$  the following inequality holds simultaneously for all $(x,a,h,k)\in\mathcal{S}\times\mathcal{A}\times[H]\times[K]:$
		\begin{equation}
			\left\{ {F}_{k,h}-F_{k,h}^{\pi}\right\}(x,a)\geq 0,\label{eq:qk-qpi-relation-c}
		\end{equation} 
		Let $\pi$ be a joint policy such that $\pi$ is the optimal policy for the $\epsilon$-tight problem at episode $k,$ whose reward (utility) $Q$ value functions at step $h$ are denoted by $Q_{k,h}^{\epsilon,*}(C_{k,h}^{\epsilon,*}).$ Then we can further obtain 
		\begin{equation}
			\mathbb{E}\left[ \sum_{k=1}^{K} \sum_a \left\{\left({F}^{\epsilon,*}_{k,1} -F_{k,1}\right) {q}^{\epsilon,*}_{k,1}\right\}(x_{k,1},a) \right] \leq \frac{(\eta+K^{1-\alpha})H^2B^c }{\eta K}.
		\end{equation} The function $F$ will be defined in Eq.\eqref{eq:f}.
	\end{lemma}

	\begin{proof}
		Consider frame $T$ and episodes in frame $T.$ Define $Z=Z_{(T-1)K^\alpha/\bc +1}$ because the value of the virtual queue does not change during each frame. We further define/recall the following notations:
		\begin{align}
			\begin{aligned}
				{F}_{k,h}(x,a)=   {Q}_{k,h}(x,a)+ \frac{Z}{\eta} {C}_{k,h}(x,a),\quad & {U}_{k,h}(x) =  {V}_{k,h}(x)+ \frac{Z}{\eta} {W}_{k,h}(x) \\
				F_{k,h}^{\pi}(x,a)=   Q_{k,h}^{\pi}(x,a)+ \frac{Z}{\eta} C_{k,h}^{\pi}(x,a),\quad & U_{k,h}^\pi(x) =  V_{k,h}^{\pi}(x)+ \frac{Z}{\eta}W_{k,h}^{\pi}(x).		
			\end{aligned} \label{eq:f}
		\end{align}
		From the updating rule of $Q$ functions, we first know that 
		\begin{align}
			\{Q_{k,h}-Q_{k,h}^
			\pi \}(x,a)= & \alpha_t^0\{Q_{(T-1)K^\alpha/B^c+1,h} - Q_{k,h}^\pi \}(x,a) \nonumber \\
			&+ \sum_{i=1}^t\alpha_t^i\left( \{V_{k_i,h+1}-V_{k,h+1}^\pi 
			\}(x_{k_i,h+1})+ \{(\hat{\mathbb{P}}_{k,h}^{k_i}-\mathbb{P}_{k,h} )V_{k,h+1}^\pi  \}(x,a) + b_i + 2H\tilde{b} \right)    
		\end{align}
		Then we have with probability at least $1-\frac{1}{k^3}$
		\begin{align}
			&\{F_{k,h} - F_{k,h}^{\pi}\}(x,a) \nonumber\\
			= &\alpha_t^0\left\{F_{(T-1)K^\alpha/\bc +1, h}-F_{k,h}^{\pi}\right\}(x,a)
			\nonumber\\
			&+\sum_{i=1}^t\alpha_t^i\left(\left\{{U}_{k_i,h+1}-U_{k,h+1}^{\pi}  \right\}(x_{k_i,h+1}) + \{(\hat{\mathbb{P}}_{k,h}^{k_i}-\mathbb{P}_{k,h}) U_{k,h+1}^{{\pi}}\}(x,a) + \left(1+\frac{Z}{\eta}\right)(b_i + 2H\tilde{b})   \right)\nonumber\\
			= &\alpha_t^0\left\{F_{(T-1)K^\alpha/\bc +1, h}-F_{k,h}^{\pi}\right\}(x,a) + \sum_{i=1}^t\alpha_t^i\left(\{(\hat{\mathbb{P}}_{k,h}^{k_i}-\mathbb{P}_{k_i,h}) U_{k,h+1}^{{\pi}} \} \right)
			\nonumber\\
			&+\sum_{i=1}^t\alpha_t^i\left(\left\{{U}_{k_i,h+1}-U_{k,h+1}^{\pi}  \right\}(x_{k_i,h+1}) + \{({\mathbb{P}}_{k_i,h}-\mathbb{P}_{k,h}) U_{k,h+1}^{{\pi}}\}(x,a) + \left(1+\frac{Z}{\eta}\right)(b_i + 2H\tilde{b})   \right)\nonumber\\
			\geq&_{(a)}  \alpha_t^0\left\{{F}_{(T-1)K^\alpha/\bc+1, h}-F_{k,h}^\pi\right\}(x,a) 
			\nonumber \\
			&+\sum_{i=1}^t\alpha_t^i\left( \left\{{U}_{k_i,h+1}-U_{k,h+1}^\pi  \right\}(x_{k_i,h+1}) +  \{({\mathbb{P}}_{k_i,h}-\mathbb{P}_{k,h}) U_{k,h+1}^{{\pi}}\}(x,a) + \left(1+\frac{Z}{\eta}\right)(b_i+H\tilde{b}  )\right)\nonumber\\
			\geq&_{(b)}  \alpha_t^0\left\{{F}_{(T-1)K^\alpha/\bc+1, h}-F_{k,h}^\pi\right\}(x,a) 
			+\sum_{i=1}^t\alpha_t^i\left( \left\{{U}_{k_i,h+1}-U_{k,h+1}^\pi  \right\}(x_{k_i,h+1}) + \left(1+\frac{Z}{\eta}\right)H\tilde{b}  \right)\nonumber\\
			= & \alpha_t^0\left\{{F}_{(T-1)K^\alpha/\bc+1, h}-F_{k,h}^\pi\right\}(x,a)
			+\sum_{i=1}^t\alpha_t^i\left\{{U}_{k_i,h+1}-U_{k_i,h+1}^\pi  \right\}(x_{k_i,h+1}) \nonumber \\
			&+\sum_{i=1}^t\alpha_t^i\left\{{U}_{k_i,h+1}^\pi -U_{k,h+1}^\pi  \right\}(x_{k_i,h+1}) + \left(1+\frac{Z}{\eta}\right)H\tilde{b} \nonumber\\
			= &_{(c)} \alpha_t^0\left\{F_{(T-1)K^\alpha/\bc +1, h}-F_{k,h}^\pi\right\}(x,a)
			+\sum_{i=1}^t\alpha_t^i\left(\max_a {F}_{k_i,h+1} (x_{k_i,h+1},a) -F_{k_i,h+1}^\pi(x_{k_i,h+1},\pi(x_{k_i,h+1})) \right)\nonumber\\
			&+\sum_{i=1}^t\alpha_t^i \left\{{U}_{k_i,h+1}^\pi -U_{k,h+1}^\pi  \right\}(x_{k_i,h+1}) + \left(1+\frac{Z}{\eta}\right)H\tilde{b} \nonumber\\
			\geq &_{(d)} \alpha_t^0\left\{F_{(T-1)K^\alpha/\bc +1, h}-F_{k,h}^\pi\right\}(x,a)
			+\sum_{i=1}^t\alpha_t^i\left(\max_a {F}_{k_i,h+1} (x_{k_i,h+1},a) -F_{k_i,h+1}^\pi(x_{k_i,h+1},\pi(x_{k_i,h+1})) \right)\nonumber\\
			& -\sum_{i=1}^t\alpha_t^i \vert (1+\frac{Z}{\eta})H\tilde{b}\vert  + (1+\frac{Z}{\eta})H\tilde{b} \nonumber \\
			\geq &\alpha_t^0\left\{F_{(T-1)K^\alpha/\bc +1, h}-F_{k_i,h}^\pi\right\}(x,a)+ \sum_{i=1}^t\alpha_t^i\left\{{F}_{k_i,h+1} -F_{k_i,h+1}^\pi\right\}(x_{k_i,h+1},\pi(x_{k_i,h+1})), \label{eq:induction-F}
		\end{align}
		where inequality $(a)$ holds because that 
		\begin{align*}
			&\left\vert \sum_{i=1}^t\alpha_t^i\left\{({\mathbb{P}}_{k_i,h}-{\mathbb{P}}_{k,h}) V_{k,h+1}^{\pi}\right\}(x,a)\right\vert  
			=  \left\vert \sum_{i=1}^t\sum_{j=k_i}^{k-1} \alpha_t^i\left\{({\mathbb{P}}_{j,h}-{\mathbb{P}}_{j+1,h}) V_{k,h+1}^{\pi}\right\}(x,a)\right\vert  
			\leq \tilde{b}H,
		\end{align*}
		and the same analysis can be applied to $\left\vert \sum_{i=1}^t\alpha_t^i\left\{({\mathbb{P}}_{k_i,h}-\mathbb{P}_{k,h}) W_{k,h+1}^{\pi}\right\}(x,a)\right\vert.$ The inequality $(b)$ is true due to the concentration result in Lemma \ref{le:u-hoeffding} and $$\sum_{i=1}^t\alpha_t^i(1 +\frac{Z}{\eta})b_i  = \frac{1}{4}\sum_{i=1}^t\alpha_t^i(1 +\frac{Z}{\eta})\sqrt{\frac{H^2\iota(\chi+1)}{\chi + t}} \geq\frac{\eta +Z}{4\eta}\sqrt{\frac{H^2\iota(\chi+1)}{\chi + t}}.$$ Equality $(c)$ holds because our algorithm selects the action that maximizes ${F}_{k_i,h+1} (x_{k_i,h+1},a)$ so ${U}_{k_i,h+1}(x_{k_i,h+1})=\max_a {F}_{k_i,h+1} (x_{k_i,h+1},a),$ and inequality $(c)$ is obtained by using Lemma \ref{le:q-diff} and the property \ref{le:lr-d} of the learning rate.
		
		The inequality above suggests that we can prove $\{{F}_{k,h} - F_{k,h}^\pi\}(x,a)$ for any $(x,a)$ if (i) $$\left\{F_{(T-1)K^\alpha/\bc+1, h}-F_{k,h}^\pi\right\}(x,a)\geq 0,$$ i.e. the result holds at the beginning of the frame and (ii) $$\left\{{F}_{k',h+1} -F_{k',h+1}^\pi\right\}(x,a)\geq 0\quad\hbox{ for any }\quad k'\leq k$$ and $(x,a),$ i.e. the result holds for step $h+1$ in all the episodes in the {\em same} frame.  
		
		It is straightforward to see that (i) holds because all reward and cost Q-functions are set to $H$ at the beginning of each frame.
		
		We now prove condition (ii) using induction, and consider the first frame, i.e. $T=1$. The proof is identical for other frames.
		
		Consider $h=H$ i.e. the last step. In this case, inequality \eqref{eq:induction-F} becomes
		\begin{align}
			\{F_{k,H} - F_{k,H}^\pi\}(x,a) 
			\geq \alpha_t^0\left\{H+\frac{Z_1}{\eta}H-F_{k,H}^\pi\right\}(x,a)\geq 0,
		\end{align} i.e. condition (ii) holds for any $k$ in the first frame and $h=H.$ By applying induction on $h$, we conclude that \begin{align}
			\{F_{k,h} - F_{k,h}^\pi\}(x,a) \geq  0. 
		\end{align} holds for any  $k,$ $h,$ and $(x,a),$ which completes the proof of \eqref{eq:qk-qpi-relation-c}. 
		Since  Eq. \eqref{eq:qk-qpi-relation-c} can only be applied to a single policy, in order to have a bound on $ \sum_{k=1}^{K} \sum_a \left\{\left({F}^{\epsilon,*}_{k,1} -F_{k,1}\right) {q}^{\epsilon,*}_{k,1}\right\}(x_{k,1},a),$ we first need to substitute $F_{k,1}^\pi$ with $F_{k,1}^{\epsilon,*}$ in Eq. \eqref{eq:qk-qpi-relation-c}, and use a union bound over all the episodes, which means with probability at least $1-\frac{1}{K^2}$ that $F_{k,1} - F_{k,1}^{\epsilon,*}\geq 0.$ Let $\cal E$ denote such event that $F_{k,h} - F_{k,h}^{\epsilon,*} \geq 0$ holds for all $k,$ $h$ and $(x,a).$ Then we conclude that
		\begin{align}
			&\mathbb{E}\left[ \sum_{k=1}^{K} \sum_a \left\{\left({F}^{\epsilon,*}_{k,1} -F_{k,1}\right) {q}^{\epsilon,*}_{k,1}\right\}(x_{k,1},a) \right]\nonumber\\
			=&   \mathbb{E}\left[\left. \sum_{k=1}^{K} \sum_a \left\{\left({F}^{\epsilon,*}_{k,1} -F_{k,1}\right) {q}^{\epsilon,*}_{k,1}\right\}(x_{k,1},a) \right| {\cal E}\right]\Pr({\cal E})+ \mathbb{E}\left[\left. \sum_{k=1}^{K} \sum_a \left\{\left({F}^{\epsilon,*}_{k,1} -F_{k,1}\right) {q}^{\epsilon,*}_{k,1}\right\}(x_{k,1},a) \right| {\cal E}^c\right]\Pr({\cal E}^c)\nonumber\\
			\leq& K H\left(1+\frac{K^{1-\alpha}\bc H}{\eta}\right)\frac{1}{K^2}\leq 
			\frac{(\eta+K^{1-\alpha})H^2B^c }{\eta K}.\label{eq:F-bound}
		\end{align}
	\end{proof}

	\begin{lemma} \label{le:qk-qpi-bound}
		Under our algorithm, we have for any $T\in[K^{1-\alpha}\cdot {B}^c],$ 
		\begin{align*}
			&\mathbb{E}\left[ \sum_{k=(T-1)K^\alpha /{B}^c +1}^{TK^\alpha/ {B}^c }  \left\{ {Q}_{k,1} - Q_{k,1}^{\pi_k}\right\}(x_{k,1},a_{k,1}) \right] \\
			\leq & H^2SA  +\frac{2(H^3\sqrt{\iota}+2H^3\tilde{b})  K^\alpha}{ {B}^c\chi}  +  \sqrt{\frac{H^4SA\iota K^{\alpha}(\chi+1)}{{B}^c}} +  \frac{2K^\alpha H^2\tilde{b}}{{B}^c}\\
			&\mathbb{E}\left[  \sum_{k=(T-1)K^\alpha/{B}^c +1}^{TK^\alpha /{B}^c }  \left\{ {C}_{k,1} - C_{k,1}^{\pi_k}\right\}(x_{k,1},a_{k,1}) \right] \\
			\leq & H^2SA  +\frac{2(H^3\sqrt{\iota}+2H^3\tilde{b})  K^\alpha}{ {B}^c\chi} +   \sqrt{\frac{H^4SA\iota K^{\alpha}(\chi+1)}{{B}^c}} +  \frac{2K^\alpha H^2\tilde{b}}{{B}^c} .
		\end{align*}
	\end{lemma}

	\begin{proof}
		We prove this lemma for the first frame such that $1\leq k \leq k^\alpha /{B}^c.$ By using the update rule recursively, we have
		\begin{align}
			{Q}_{k,h}(x, a)  \leq & \alpha_t^0H + \sum_{i=1}^t\alpha_t^i \left(r_{k_i,h}(x,a)+ {V}_{k_i,h+1}(x_{k_i,h+1})+ b_i + 2H\tilde{b}\right),\end{align} where 
		$\alpha_t^0=\prod_{j=1}^t(1-\alpha_j)$ and  $\alpha_t^i=\alpha_i\prod_{j=i+1}^t(1-\alpha_j).$ From the inequality above, we further obtain 	
		\begin{align}
			\sum_{k=1}^{K^\alpha/{B}^c } {Q}_{k,h} (x,a) \leq  \sum_{k=1}^{K^\alpha/{B}^c }  \alpha_t^0 H+  \sum_{k=1}^{K^\alpha/{B}^c }   \sum_{i=1}^{N_{k,h}(x,a)}\alpha_{N_{k,h}}^i \left(r_{k_i,h}(x,a)+{V}_{k_i,h+1}(x_{k_i,h+1})+ b_i  + 2H\tilde{b} \right).\label{eq:Q-V}
		\end{align} 	
		We simplify our notation in this proof and use the following notations: 
		\begin{align*}
			N_{k,h}=N_{k,h}(x_{k,h},a_{k,h}),\quad k^{(k,h)}_i=k_i(x_{k,h},a_{k,h}),
		\end{align*} where $k^{(k,h)}_i$ is the index of the episode in which the agent visits state-action pair $(x_{k,h}, a_{k,h})$ for the $i$th time. Since in a given sample path, $(k,h)$ can uniquely determine  $(x_{k,h},a_{k,h}),$  this notation introduces no ambiguity. We note that
		\begin{equation}
			\sum_{k=1}^{K^\alpha / {B}^c} \sum_{i=1}^{N_{k,h}} \alpha_{N_{k,h}}^i  {V}_{k_i^{(k,h)}, h+1}\left(x_{k_i^{(k,h)},h+1}\right)\leq \sum_{k=1}^{K^\alpha /{B}^c } {V}_{k, h+1} (x_{k,h+1})\sum_{t=N_{k,h}}^\infty \alpha_t^{N_{k,h}} \leq \left(1+\frac{1}{\chi}\right)\sum_k {V}_{k, h+1} (x_{k,h+1}),\label{eq:Vbound}
		\end{equation}
		Then we obtain
		\begin{align*}
			&\sum_{k=1}^{K^\alpha/{B}^c} Q_{k,h} (x_{k,h},a_{k,h}) \\
			\leq &\sum_{k=1}^{K^\alpha/{B}^c} \alpha_t^0H + (1+ \frac{1}{\chi})	\sum_{k=1}^{K^\alpha/{B}^c} \left(r_{k,h}(x_{k,h},a_{k,h}) + {V}_{k,h+1}(x_{k,h+1})\right) +\sum_{k=1}^{K^\alpha/{B}^c} \sum_{i=1}^{N_{k,h}}\alpha_{N_{k,h}}^i  b_i + K^\alpha\tilde{b}/{B}^c\\
			\leq & \sum_{k=1}^{K^\alpha/{B}^c}\left(r_{k,h}(x_{k,h},a_{k,h})+ {V}_{k,h+1}(x_{k,h+1})\right) +HSA
			+\frac{2(H^2\sqrt{\iota}+2H^2\tilde{b})K^\alpha }{\bc \chi}  \\
			&+  \frac{1}{2}\sqrt{H^2SA\iota K^{\alpha}(\chi+1)/ {B}^c} + 2K^\alpha H\tilde{b}/{B}^c,
		\end{align*}
		where the last inequality holds because (i) we have
		\begin{align*}
			\sum_{k=1}^{K^\alpha/{B}^c}  \alpha_{N_{k,h}}^0 H =\sum_k H\mathbb{I}_{\{N_{k,h}=0\}}\leq HSA,
		\end{align*} (ii) $ {V}_{k,h+1}(x_{k,h+1}) \leq (H^2\sqrt{\iota}+\tilde{b}), r_{k,h}(x_{k,h},a_{k,h})\leq 1,$ and (iii) we know that
		\begin{align*}
			& \sum_{k=1}^{K^\alpha/{B}^c} \sum_{i=1}^{N_{k,h}} \alpha_{N_{k,h}}^i  b_i = \frac{1}{4} \sum_{k=1}^{K^\alpha/{B}^c} \sum_{i=1}^{N_{k,h}} \alpha_{N_{k,h}}^i  \sqrt{\frac{H^2\iota(\chi+1)}{\chi+i}} \leq \frac{1}{2} \sum_{k=1}^{K^\alpha/{B}^c} \sqrt{\frac{H^2\iota(\chi+1)}{\chi+N_{k,h}}}\\
			= &\frac{1}{2}\sum_{x,a}\sum_{n=1}^{N_{K^\alpha/\bc,h}(x,a)}\sqrt{\frac{H^2\iota(\chi+1)}{\chi+n}} \leq \frac{1}{2}\sum_{x,a}\sum_{n=1}^{N_{K^\alpha/\bc,h}(x,a)}\sqrt{\frac{H^2\iota(\chi+1)}{n}} \overset{(1)}{\leq} \sqrt{H^2SA\iota K^{\alpha}(\chi+1)/{B}^c }, 
		\end{align*}
		where the last inequality  above holds because the left hand side of $(1)$ is the summation of $K^\alpha/\bc$ terms and $\sqrt{\frac{H^2\iota(\chi+1)}{\chi+n}}$ is a decreasing function of $n.$ 
		
		Therefore, it is maximized when $N_{K^\alpha/\bc,h} = K^\alpha / {B}^cSA $ for all $x,a.$ Thus we can obtain
		\begin{align*}
			&\sum_{k=1}^{K^\alpha/{B}^c}  Q_{k,h} (x_{k,h},a_{k,h}) -\sum_k Q_{k,h}^{\pi_k}(x_{k,h},a_{k,h})\\
			\leq &\sum_{k=1}^{K^\alpha/{B}^c} \left({V}_{k,h+1}(x_{k,h+1})-\mathbb P_{k,h} V_{k,h+1}^{\pi_k}(x_{k,h},a_{k,h})\right) +HSA  +\frac{2(H^2\sqrt{\iota}+2H^2\tilde{b})  K^\alpha}{ {B}^c\chi } \\
			&+  \sqrt{H^2SA\iota K^{\alpha}(\chi+1)/ {B}^c} +\frac{ 2K^\alpha H\tilde{b}}{{B}^c}\\
			\leq &\sum_{k=1}^{K^\alpha/{B}^c} \left({V}_{k,h+1}(x_{k,h+1})-\mathbb P_{k,h} V_{h+1}^{\pi_k}(x_{k,h},a_{k,h})+V_{k,h+1}^{\pi_k}(x_{k,h+1})-V_{k,h+1}^{\pi_k}(x_{k,h+1})\right)\\
			&+HSA  +\frac{2(H^2\sqrt{\iota}+2H^2\tilde{b})  K^\alpha}{{B}^c \chi} +  \sqrt{H^2SA\iota K^{\alpha}(\chi+1) / {B}^c} + 2K^\alpha H\tilde{b}/{B}^c\\
			=&\sum_{k=1}^{K^\alpha/{B}^c}  \left({V}_{k,h+1}(x_{k,h+1}))-V_{k,h+1}^{\pi_k}(x_{k,h+1})-\mathbb P_{k,h} V_{k,h+1}^{\pi_k}(x_{k,h},a_{k,h})+\hat{\mathbb P}_{k,h} V^{\pi_k}_{k,h+1}(x_{k,h}, a_{k,h})\right)\\
			&+HSA  +\frac{2(H^2\sqrt{\iota}+2H^2\tilde{b}) K^\alpha }{ {B}^c \chi} +  \sqrt{H^2SA\iota K^{\alpha}(\chi+1)  / {B}^c} + 2K^\alpha H\tilde{b}/{B}^c\\
			= & \sum_{k=1}^{K^\alpha/{B}^c} \left({Q}_{k,h+1}(x_{k,h+1},a_{k,h+1})-Q_{k,h+1}^{\pi_k}(x_{k,h+1},a_{k,h+1})-\mathbb P_h V_{k,h+1}^{\pi_k}(x_{k,h},a_{k,h})+\hat{\mathbb P}_{k,h} V^{\pi_k}_{k,h+1}(x_{k,h}, a_{k,h}\right)\\
			&+HSA  +\frac{2(H^2\sqrt{\iota}+2H^2\tilde{b}) K^\alpha }{{B}^c \chi}  \\
			&+  \sqrt{H^2SA\iota K^{\alpha}(\chi+1) / {B}^c} + 2K^\alpha H\tilde{b}/{B}^c.
		\end{align*} Taking the expectation on both sides yields
		\begin{align*}
			&\mathbb E\left[\sum_{k=1}^{K^\alpha/{B}^c}  Q_{k,h} (x_{k,h},a_{k,h}) -\sum_k Q_{k,h}^{\pi_k}(x_{k,h},a_{k,h})\right]\\
			\leq & \mathbb E\left[\sum_{k=1}^{K^\alpha/{B}^c}\left(Q_{k,h+1}(x_{k,h+1}, a_{k,h+1})-Q_{k,h+1}^{\pi_k}(x_{k,h+1}, a_{k,h+1})\right)\right] \\
			&+HSA  +\frac{2(H^2\sqrt{\iota}+2H^2\tilde{b}) K^\alpha}{ {B}^c\chi} + \sqrt{H^2SA\iota K^{\alpha}(\chi+1)/{B}^c} + 2K^\alpha H\tilde{b}/{B}^c .
		\end{align*} 
		Then by using the inequality repeatably, we obtain for any $h\in[H],$
		\begin{align*}
			&	\mathbb E\left[\sum_{k=1}^{K^\alpha/{B}^c} \left( Q_{k,h} (x_{k,h},a_{k,h}) - Q_{k,h}^{\pi_k}(x_{k,h},a_{k,h})\right)\right]	\\
			\leq & H^2SA  +\frac{2(H^3\sqrt{\iota}+2H^3\tilde{b})  K^\alpha}{ {B}^c\chi} +  \sqrt{H^4SA\iota K^{\alpha}(\chi+1)/{B}^c} +  2K^\alpha H^2\tilde{b}/ {B}^c.
		\end{align*} We finish the proof.

	\end{proof}

	\begin{lemma}\label{le:epsilon-dif}
		Given $\epsilon\leq \delta$,  we have
		$$\mathbb{E}\left[ \sum_a \left\{Q_{k,1}^{*}q_{k,1}^*-  {Q}_{k,1}^{\epsilon,*}{q}_{k,1}^{\epsilon,*}\right\}(x_{k,1},a)\right] \leq \frac{H \epsilon }{\delta}.$$ \hfill{$\square$}
	\end{lemma}
	\begin{proof}
		Given ${q}_{k,h}^*(x,a)$ is the optimal solution for episode $k$, we have $$\sum_{h,x,a} {q}_{k,h}^*(x,a)g_{k,h}(x,a) \geq \rho.$$ Under Assumption \ref{as:1}, we know that there exists a feasible solution $\{q^{\xi_1}_{k,h}(x,a)\}_{h=1}^H$ such that
		$$\sum_{h,x,a} q_{k,h}^{\xi_1}(x,a)g_{k,h}(x,a) \geq \rho+\delta.$$ We construct $q_{k,h}^{\xi_2}(x,a) = (1-\frac{\epsilon }{\delta})q^*_{k,h}(x,a) + \frac{\epsilon}{\delta} q^{\xi_1}_{k,h}(x,a),$ which satisfies that 
		\begin{align*}
			\sum_{h,x,a} q_{k,h}^{\xi_2}(x,a)g_{k,h}(x,a) &  = \sum_{h,x,a} \left( (1-\frac{\epsilon }{\delta}) q^*_{k,h}(x,a) +\frac{\epsilon }{\delta} q^{\xi_1}_{k,h}(x,a)  \right)g_{k,h}(x,a)\geq \rho + \epsilon ,\\
			\sum_{h,x,a} q_{k,h}^{\xi_2}(x,a)  &= \sum_{x^\prime,a^\prime} \mathbb{P}_{k,h-1} (x\vert x^\prime,a^\prime) q_{k,h-1}^{\xi_2}(x^\prime,a^\prime),\\
			\sum_{h,x,a} q_{k,h}^{\xi_2}(x,a)  &= 1.
		\end{align*}
		Also we have $q_{k,h}^{\xi_2}(x,a)\geq 0 $ for all $(h,x,a).$ Thus $\{q^{\xi_2}_{k,h}(x,a)\}_{h=1}^H$ is a feasible solution to the $\epsilon$-tightened optimization problem \eqref{eq:lp-epsilon_app}. Then given $\{{q}^{\epsilon,*}_{k,h}(x,a)\}_{h=1}^H$ is the optimal solution to the $\epsilon$-tightened optimization problem, we have
		\begin{align*}
			& \sum_{h,x,a} \left( {q}_{k,h}^{*}(x,a) -  {q}_{k,h}^{\epsilon,*}(x,a)\right)r_{k,h}(x,a)  \\
			\leq & \sum_{h,x,a} \left( q_{k,h}^{*}(x,a) -  q_{k,h}^{\xi_2}(x,a)\right)r_{k,h}(x,a) \\
			\leq &\sum_{h,x,a} \left( q_{k,h}^{*}(x,a) -  \left(1-\frac{\epsilon }{\delta} \right)q_{k,h}^*(x,a) - \frac{\epsilon }{\delta}  q_{k,h}^{\xi_1}(x,a)\right)r_{k,h}(x,a) \\
			\leq & \sum_{h,x,a} \left( q_{k,h}^{*}(x,a) -  \left(1-\frac{\epsilon }{\delta} \right)q_{k,h}^*(x,a) \right)r_{k,h}(x,a) \\
			\leq &\frac{\epsilon }{\delta}  \sum_{h,x,a}  q_{k,h}^{*}(x,a) r_{k,h}(x,a) \leq \frac{H\epsilon }{\delta},
		\end{align*}
		where the last inequality holds because $0\leq r_{k,h}(x,a)\leq 1$ under our assumption. Therefore the result follows because 
		\begin{align*}
			\sum_a Q_{k,1}^{*}(x_{k,1},a)q_{k,1}^*(x_{k,1},a)  =&\sum_{h,x,a} {q}_{k,h}^{*}(x,a)r_{k,h}(x,a) \\
			\sum_a {Q}_{k,1}^{\epsilon,*}(x_{k,1},a) {q}_{k,1}^{\epsilon,*}(x_{k,1},a)  =&\sum_{h,x,a} {q}_{k,h}^{\epsilon,*}(x,a)r_{k,h}(x,a).
		\end{align*}
	\end{proof}
	
	
	\begin{lemma}\label{le:drift}
		Assume $\epsilon \leq \delta.$ The expected Lyapunov drift satisfies
		\begin{align}
			& \mathbb{E}\left[L_{T+1}-L_T\vert Z_T=z \right] \nonumber\\
			\leq & \frac{{B}^c}{K^\alpha}\sum_{k=(T-1)K^\alpha+1}^{TK^\alpha}\left(-\eta \mathbb{E} \left[ \left.\sum_a  \left\{\hat{Q}_{k,1}{q}^{\epsilon,*}_1\right\}(x_{k,1},a)-\hat{Q}_{k,1} (x_{k,1},a_{k,1}) \right\vert Z_T= z \right]\right.\nonumber\\
			&\left.+ z \mathbb{E}\left[\left. \sum_a  \left\{\left({C}^{\epsilon,*}_{k,1}-C_{k,1}\right){q}^{\epsilon,*}_{k,1}\right\}(x_{k,1},a)\right\vert Z_T=z  \right]\right) + 2H^4\iota + 4H^4\tilde{b} +\epsilon^2. 
		\end{align}
	\end{lemma}
	\begin{proof}
		Assume $\epsilon \leq \delta.$ The expected Lyapunov drift satisfies
		\begin{align}
			& \mathbb{E}\left[L_{T+1}-L_T\vert Z_T=z \right] \nonumber\\
			\leq & \frac{{B}^c}{K^\alpha}\sum_{k=(T-1)K^\alpha+1}^{TK^\alpha}\left(-\eta \mathbb{E} \left[ \left.\sum_a  \left\{\hat{Q}_{k,1}{q}^{\epsilon,*}_1\right\}(x_{k,1},a)-\hat{Q}_{k,1} (x_{k,1},a_{k,1}) \right\vert Z_T= z \right]\right.\nonumber\\
			&\left.+ z \mathbb{E}\left[\left. \sum_a  \left\{\left({C}^{\epsilon,*}_{k,1}-C_{k,1}\right){q}^{\epsilon,*}_{k,1}\right\}(x_{k,1},a)\right\vert Z_T=z  \right]\right) + 2H^4\iota + 4H^4\tilde{b} +\epsilon^2. \label{eq:drift-inq}
		\end{align}
		Based on the definition of $L_T=\frac{1}{2}Z_T^2,$ the Lyapunov drift is
		\begin{align*}
			L_{T+1}-L_T   \leq & Z_T\left(\rho+ \epsilon  - \frac{\bar{C}_T\bc}{K^\alpha} \right) + \frac{\left( \frac{\bar{C}_T\bc }{K^\alpha} +\epsilon  -\rho\right)^2}{2} \\
			\leq & Z_T\left(\rho +\epsilon  -\frac{\bar{C}_T\bc }{K^\alpha} \right)+ 2H^4\iota + 4H^4\tilde{b} +\epsilon^2\\
			\leq & \frac{Z_T\bc }{K^\alpha}\sum_{k=TK^\alpha/\bc+1}^{(T+1)K^\alpha/\bc} \left(\rho +\epsilon  - \hat{C}_{k,1}(x_{k,1},a_{k,1}) \right)+2H^4\iota + 4H^4\tilde{b}+\epsilon^2
		\end{align*}
		where the first inequality is because the upper bound on $\vert \hat{C}_{k,1}  (x_{k,1},a_{k,1})\vert$ is $H^2(\sqrt{\iota}+2\tilde{b})$ from Lemma \ref{le:q1-bound}. Let $\{q^{\epsilon}_{k,h}\}_{h=1}^H$ be a feasible solution to the tightened LP \eqref{eq:lp-epsilon_app} at episode $k.$  Then the expected Lyapunov drift conditioned on $Z_T=z$ is
		\begin{align}
			& \mathbb{E}\left[L_{T+1}-L_T\vert Z_T=z \right] \nonumber\\
			\leq &  \frac{\bc}{K^\alpha}\sum_{k=(T-1)K^\alpha +1}^{TK^\alpha/\bc}  \left(\mathbb{E}\left[\left.z\left( \rho+ \epsilon  - \hat{C}_{k,1} (x_{k,1},a_{k,1})  \right) -\eta\hat{Q}_{k,1}(x_{k,1},a_{k,1}) \right\vert Z_T=z \right] + \eta\mathbb{E}\left[\left. \hat{Q}_{k,1} (x_{k,1},a_{k,1}) \right\vert Z_T=z  \right]\right)\nonumber\\
			&+ 2H^4\iota + 4H^4\tilde{b} +\epsilon^2. \label{eq:drift-1}
		\end{align}
		Now we focus on the term inside the summation and obtain that 
		\begin{align*}
			&\left(\mathbb{E}\left[\left.z\left( \rho+ \epsilon  - \hat{C}_{k,1} (x_{k,1},a_{k,1})  \right) -\eta \hat{Q}_{k,1}(x_{k,1},a_{k,1}) \right\vert Z_T=z \right] + \eta\mathbb{E}\left[\left. \hat{Q}_{k,1} (x_{k,1},a_{k,1}) \right\vert Z_T=z  \right]\right) \\
			\leq  &_{(a)} z(\rho+\epsilon)- \mathbb{E} \left[\left.  \eta\left(\sum_a\left\{\frac{z}{\eta} \hat{C}_{k,1}q^{\epsilon}_{k,1} + \hat{Q}_{k,1}q^{\epsilon}_{k,1}\right\}(x_{k,1},a)\right) \right\vert  Z_T=z  \right] + \eta\mathbb{E}\left[\left. \hat{Q}_{k,1} (x_{k,1},a_{k,1}) \right\vert Z_T=z  \right]\nonumber\\
			= &  \mathbb{E} \left[ \left. z\left(  \rho  +\epsilon  -\sum_a\hat{C}_{k,1} (x_{k,1},a)q^{\epsilon}_{k,1}(x_{k,1},a) \right) \right\vert Z_T=z \right]\\
			&-\mathbb{E}\left[\left. \eta \sum_a  \hat{Q}_{k,1} (x_{k,1},a)q^{\epsilon}_{k,1}(x_{k,1},a)	- \eta \hat{Q}_{k,1} (x_{k,1},a_{k,1})\right\vert Z_T=z\right] \nonumber\\
			=& \mathbb{E} \left[ z\left(\left.\rho +\epsilon  - \sum_a {C}^{\epsilon}_{k,1} (x_{k,1},a)q^{\epsilon}_{k,1}(x_{k,1},a) \right) \right\vert Z_T=z \right]\\ &-\mathbb{E}\left[\left. \eta \sum_a  \hat{Q}_{k,1} (x_{k,1},a)q^{\epsilon}_{k,1}(x_{k,1},a)- \eta \hat{Q}_{k,1} (x_{k,1},a_{k,1})\right\vert Z_T=z\right] + \mathbb{E} \left[\left. z \sum_a \left\{(C^{\epsilon}_{k,1}-\hat{C}_{k,1})q^{\epsilon}_{k,1}\right\}(x_{k,1},a)   \right\vert Z_T=z  \right]  \nonumber\\
			\leq & -\eta \mathbb{E} \left[\left. \sum_a  \hat{Q}_{k,1} (x_{k,1},a)q^{\epsilon}_{k,1}(x_{k,1},a)	- \hat{Q}_{k,1} (x_{k,1},a_{k,1}) \right\vert  Z_T=z \right]  + \mathbb{E}\left[\left. z \sum_a \left\{ (C^\epsilon_{k,1}-\hat{C}_{k,1})q^{\epsilon}_{k,1}\right\}(x_{k,1},a)\right\vert  Z_T=z   \right],
		\end{align*}
		where inequality $(a)$ holds because $a_{k,h}$ is chosen to maximize $ \hat{Q}_{k,h} (x_{k,h},a) + \frac{Z_T}{\eta} \hat{C}_{k,h} (x_{k,h},a).$  and the last equality holds due to that $\{q^{\epsilon}_{k,h}(x,a)\}_{h=1}^H$ is a feasible solution to the optimization problem \eqref{eq:lp-epsilon_app}, so 
		\begin{align*}
			\left(\rho +\epsilon  -\sum_a {C}_{k,1}^{\epsilon}(x_{k,1},a)q^{\epsilon}_{k,1}(x_{k,1},a) \right)&=\left(\rho +\epsilon  -\sum_{h,x,a}g_{k,h}(x,a){q}^{\epsilon}_{k,h}(x,a)  \right)\leq 0.
		\end{align*}
		Therefore, we can conclude the lemma by substituting ${q}^{\epsilon}_{k,h}(x,a)$ with the optimal solution ${q}^{\epsilon,*}_{k,h}(x,a)$.
	\end{proof}
	\begin{lemma}\label{le:zk-bound}
		Assuming $\epsilon \leq \frac{\delta}{2},$
		we have for any $1\leq T\leq K^{1-\alpha}\cdot\bc$
		\begin{align}
			\mathbb{E}[ Z_{T}]  \leq &\frac{100(H^4\iota+\tilde{b}^2H^2)}{\delta}\log \left( \frac{16(H^2\sqrt{\iota}+\tilde{b}H^2)}{\delta} \right) +\frac{4H^2B^c}{K\delta} + \frac{4H^2B^c}{\eta\delta K^\alpha}+ \frac{4\eta(\sqrt{H^2\iota}+2H^2\tilde{b}) }{\delta}. \label{eq:zk-bound}
		\end{align} 
	\end{lemma}
	
	The proof will also use the following lemma from \citep{Nee_16}.
	\begin{lemma}\label{le:drift-bond-cond}
		Let $S_t$ be the state of a Markov chain, $L_t$ be a Lyapunov function with $L_0=l_0,$  and its drift $\Delta_t= L_{t+1}-L_t.$ Given the constant $\delta$ and $v$ with $0<\delta\leq v,$ suppose that the expected drift $\mathbb{E}[\Delta_t\vert S_t=s]$ satisfies the following conditions:
		\begin{itemize}
			\item[(1)] There exists constant $\gamma>0$ and $\theta_t>0$ such that $\mathbb{E}[\Delta_t\vert S_t=s]\leq -\gamma$ when $L_t\geq \theta_t.$
			\item[(2)] $\vert L_{t+1}-L_t\vert \leq v$ holds with probability one.
		\end{itemize}
		Then we have $$\mathbb{E}[e^{rL_t}]\leq e^{rl_0}+\frac{2e^{r(v+\theta_t)}}{r\gamma},$$ where $r=\frac{\gamma}{v^2+v\gamma/3}.$\hfill{$\square$}
	\end{lemma}
	
	\begin{proof}[Proof of Lemma \ref{le:zk-bound}]
		We apply Lemma \ref{le:drift-bond-cond} to a new Lyapunov function: $$\bar{L}_T= Z_T.$$ 
		To verify condition (1) in Lemma \ref{le:drift-bond-cond}, consider $$\bar{L}_T =Z_T\geq \theta_T =\frac{4( 	\frac{(\eta+K^{1-\alpha})H^2B^c }{\eta K}+ \eta (\sqrt{H^2\iota}+2H^2\tilde{b}) + H^4\iota+\epsilon^2 +2H^4\tilde{b}^2 )}{\delta}$$ and $2\epsilon\leq \delta.$ The conditional expected drift of 
		\begin{align*}
			&\mathbb{E} \left[ Z_{T+1} - Z_T \vert Z_T=z\right]\\
			=& \mathbb{E}\left[ \left. \sqrt{Z_{T+1}^2}  - \sqrt{z^2} \right\vert Z_T=z \right] \\
			&	\leq  \frac{1}{2z} \mathbb{E}  \left[ \left. Z_{T+1}^2  - z^2 \right\vert Z_T=z\right]\\
			&	\leq_{(a)} -\frac{\delta}{2}  + \frac{4(	\frac{(\eta+K^{1-\alpha})H^2B^c }{\eta K}+ \eta (\sqrt{H^2\iota}+2H^2\tilde{b}) + H^4\iota+\epsilon^2 +2H^4\tilde{b}^2)}{z} \\
			&	\leq  -\frac{\delta}{2}  + \frac{4(	\frac{(\eta+K^{1-\alpha})H^2B^c }{\eta K}+ \eta (\sqrt{H^2\iota}+2H^2\tilde{b}) + H^4\iota+\epsilon^2 +2H^4\tilde{b}^2 }{\theta_T}\\
			&	=  -\frac{\delta}{4},
		\end{align*}
		where  inequality ($a$) is obtained according to Lemma \ref{le:drift_epi_neg}; and the last inequality holds given $z \geq\theta_T.$ 
		
		To verify condition (2) in Lemma \ref{le:drift-bond-cond}, we have 
		$$Z_{T+1}-Z_T \leq \vert Z_{T+1} - Z_{T}\vert \leq \left|\rho+\epsilon -\bar{C}_T\right|\leq (H+H^2\sqrt{\iota}+2\tilde{b}H^2)+\epsilon\leq 2(H^2\sqrt{\iota}+\tilde{b}H^2) ,$$
		where the last inequality holds because $2\epsilon\leq \delta\leq 1.$
		
		Now choose $\gamma  = \frac{\delta}{4}$ and $v=2(\sqrt{H^4\iota}+\tilde{b}H^2) .$ From Lemma \ref{le:drift-bond-cond}, we obtain
		\begin{equation}
			\mathbb{E}\left[e^{rZ_T}\right]\leq e^{rZ_1} + \frac{2e^{r(v+\theta_T)}}{r\gamma},\quad\hbox{where}\quad r=\frac{\gamma}{v^2+v\gamma/3}. \label{eq:rzk-bound} 
		\end{equation}
		By Jensen's inequality, we have $$e^{r\mathbb{E}\left[ Z_T \right]  }\leq  \mathbb{E}\left[e^{r Z_T }\right],$$ which implies that
		\begin{align}
			&	\mathbb{E}[ Z_T]  \leq \frac{1}{r}\log{\left(1 + \frac{2e^{r(v+\theta_T)}}{r\gamma} \right)} \nonumber\\
			= &  \frac{1}{r}\log{\left( 1 + \frac{6v^2+2v\gamma}{3\gamma^2} e^{r(v+\theta_T)}  \right)}\nonumber\\
			\leq 	&  \frac{1}{r}\log{\left(1+ \frac{8v^2}{3\gamma^2} e^{r(v+\theta_T)}  \right)}\nonumber\\
			\leq 	&  \frac{1}{r}\log{\left(\frac{11v^2}{3\gamma^2} e^{r(v+\theta_T)}  \right)}\nonumber\\
			\leq  	&  \frac{4v^2}{3\gamma}\log{\left(\frac{11v^2}{3\gamma^2} e^{r(v+\theta_T)}  \right)} \nonumber\\
			\leq & \frac{3v^2}{\gamma} \log \left(\frac{2v}{\gamma} \right) + v + \theta_T \nonumber\\
			\leq&\frac{3v^2}{\gamma} \log \left(\frac{2v}{\gamma} \right)+ v  \nonumber\\
			&+ \frac{4( 	\frac{(\eta+K^{1-\alpha})H^2B^c }{\eta K}+ \eta (\sqrt{H^2\iota}+2H^2\tilde{b}) + H^4\iota+\epsilon^2 +2H^4\tilde{b}^2 )}{\delta} \nonumber\\
			= & \frac{96(H^4\iota+\tilde{b}^2H^2)}{\delta}\log \left( \frac{16(H^2\sqrt{\iota}+\tilde{b}H^2)}{\delta} \right) + 2(H^2\sqrt{\iota}+\tilde{b}H^2) \nonumber\\
			& + \frac{4( 	\frac{(\eta+K^{1-\alpha})H^2B^c }{\eta K}+ \eta (\sqrt{H^2\iota}+2H^2\tilde{b}) + H^4\iota+\epsilon^2 +2H^4\tilde{b}^2 )}{\delta} \nonumber\\
			\leq & \frac{100(H^4\iota+\tilde{b}^2H^2)}{\delta}\log \left( \frac{16(H^2\sqrt{\iota}+\tilde{b}H^2)}{\delta} \right) +\frac{4H^2B^c}{K\delta} + \frac{4H^2B^c}{\eta\delta K^\alpha}+ \frac{4\eta(\sqrt{H^2\iota}+2H^2\tilde{b}) }{\delta} ,\label{eq:zk_lemma}
		\end{align} which completes the proof of Lemma \ref{le:zk-bound}. 
	\end{proof}
	
	\begin{lemma}\label{le:drift_epi_neg}
		Given $\delta\geq 2\epsilon,$ under our algorithmsl, the conditional  expected drift is
		\begin{align}
			\mathbb{E}\left[L_{T+1}-L_T\vert Z_T=z \right]\leq  
			-\frac{\delta}{2}z+ 	\frac{(\eta+K^{1-\alpha})H^2B^c }{\eta K}+ \eta (\sqrt{H^2\iota}+2H^2\tilde{b}) + H^4\iota+\epsilon^2 +2H^4\tilde{b}^2
		\end{align}\label{le:drift_epi}
	\end{lemma}
	
	\begin{proof}
		Recall that $L_T = \frac{1}{2}  Z_T^2,$ and the virtual queue is updated by using 
		$$	Z_{T+1} = \left(  Z_T   + \rho + \epsilon -\frac{\bar{C}_T\bc }{K^\alpha}\right)^+.$$ From inequality \eqref{eq:drift-1}, we have
		\begin{align*}
			& \mathbb{E}\left[L_{T+1}-L_T\vert Z_T=z \right] \\
			\leq &   \frac{\bc}{K^\alpha}\sum_{k=(T-1)K^\alpha/\bc+1}^{TK^\alpha/\bc}\mathbb E\left[ Z_T\left(\rho+\epsilon-C_{k,1}(x_{k,1},a_{k,1})  \right)  -\eta  Q_{k,1} (x_{k,1},a_{k,1}) \right. \\
			&\left.+ \eta  Q_{k,1} (x_{k,1},a_{k,1}) \vert Z_T=z \right]+H^4\iota +2H^4\tilde{b}^2+ \epsilon^2\\
			\leq &_{(a)}  \frac{\bc}{K^\alpha} \sum_{k=(T-1)K^\alpha/\bc+1}^{TK^\alpha/\bc}\mathbb E\left[Z_T\left( \rho+\epsilon-\sum_a \left\{C_{k,1}q_{k,1}^{\pi}\right\}(x_{k,1},a)\right) \right. \\
			&\left. -\eta \sum_a \{Q_{k,1}q_{k,1}^{\pi}\} (x_{k,1},a)+ \eta  Q_{k,1} (x_{k,1},a_{k,1})\vert Z_T=z \right]\\
			&+\epsilon^2 +H^4\iota + 2H^4\tilde{b}^2 \\
			\leq &  \frac{\bc}{K^\alpha}\sum_{k=(T-1)/\bc K^\alpha+1}^{TK^\alpha/\bc} \mathbb E\left[  Z_T\left(\rho+\epsilon-\sum_a \left\{C_{k,1}^{\pi}q_{k,1}^{\pi}\right\}(x_{k,1},a) \right) \right.\\
			&\left.-\eta \sum_a \{Q_{k,1}q_{k,1}^{\pi}\} (x_{k,1},a)+ \eta  Q_{k,1} (x_{k,1},a_{k,1})\vert Z_T=z\right]  \\
			&+\frac{\bc}{K^\alpha}\sum_{k=(T-1)K^\alpha/\bc+1}^{TK^\alpha/\bc} \mathbb E\left[Z_T\sum_a  \left\{C_{k,1}^{\pi}q_{k,1}^{\pi}\right\}(x_{k,1},a) - Z_T\sum_a \left\{C_{k,1}q_{k,1}^{\pi}\right\}(x_{k,1},a)\vert Z_T=z\right] \\
			&+\frac{\bc}{K^\alpha}\sum_{k=(T-1)K^\alpha/\bc+1}^{TK^\alpha/\bc}\mathbb E\left[\eta \sum_a \left\{Q_{k,1}^{\pi}q_{k,1}^{\pi}\right\}(x_{k,1},a)  - \eta \sum_a \left\{Q_{k,1}^{\pi}q_{k,1}^{\pi}\right\}(x_{k,1},a) \vert Z_T=z\right]+H^4\iota+\epsilon^2 +2H^4\tilde{b}^2 \\ 
			\leq  &_{(b)}  - \frac{\delta}{2}z+ \frac{\bc}{K^\alpha}\sum_{k=(T-1)K^\alpha/\bc+1}^{TK^\alpha/\bc}\mathbb E\left[\left.\eta \sum_a\left\{(F_{k,1}^{\pi}-F_{k,1})q_{k,1}^{\pi}\right\}(x_{k,1},a)+ \eta  Q_{k,1} (x_{k,1},a_{k,1})\right\vert Z_T=z\right] \\
			&+H^4\iota+\epsilon^2 +2H^4\tilde{b}^2\\
			\leq &_{(c)} - \frac{\delta}{2}z+	\frac{(\eta+K^{1-\alpha})H^2B^c }{\eta K}+ \eta (\sqrt{H^2\iota}+2H^2\tilde{b}) + H^4\iota+\epsilon^2 +2H^4\tilde{b}^2. 
		\end{align*}
		Inequality $(a)$ holds because of our algorithm. Inequality $(b)$ holds because $\sum_a \left\{Q_{k,1}^{\pi}q_{k,1}^{\pi}\right\}(x_{k,1},a)$ is non-negative, and under Slater's condition, we can find policy $\pi$ such that $$\epsilon+\rho-\mathbb E\left[\sum_a {C}^{\pi}_{k,1}(x_{k,1},a){q}^{\pi}_{k,1}(x_{k,1},a) \right]=\rho+\epsilon-\mathbb E\left[\sum_{h,x,a} {q}^{\pi}_{k,h}(x,a)g_{k,h}(x,a)\right]\leq -\delta+\epsilon  \leq -\frac{\delta}{2}.$$ Finally, inequality $(c)$ is obtained due to the fact that $Q_{k,1} (x_{k,1},a_{k,1})$ is bounded by using Lemma \ref{le:q1-bound}, and the fact that 
		\begin{align*}
			\mathbb E\left[\left.\sum_{k=(T-1)K^\alpha/\bc+1}^{TK^\alpha/\bc}\sum_a\left\{(F_{k,1}^{\pi}-F_{k,1})q_{k,1}^{\pi}\right\}(x_{k,1},a)\right\vert Z_T=z\right]
		\end{align*}
		can be bounded as \eqref{eq:F-bound} (note that the overestimation result and the concentration result in frame $T$ hold regardless of the value of $Z_T$).
	\end{proof}
	
	\section{Proof of Lemma \ref{le:non-bandit-c}}\label{ap:prove-non-bandit}
	\begin{lemma}
		Let \begin{align}
			f_g(G_i) =& \left\{
			\begin{aligned}
				&   G_i/K^\lambda & \quad \text{ if } G_i < W\rho   \\
				&  G_i/K^\lambda   & \quad \text{ if } G_i \geq W\rho 
			\end{aligned}\right.\\
			f_r(R_i) = &\left\{
			\begin{aligned}
				&   0  & \quad \text{ if } G_i < W\rho   \\
				&  R_i   & \quad \text{ if } G_i \geq W\rho 
			\end{aligned}\right.
		\end{align} 
		Let $R_i(B_i)(G_i(B_i))$ be the cumulative reward(utility) collected in epoch $i$ by the given algorithm with the estimate value $B_i$ chosen using Exp3 Algorithm. Let $\hat{B}$ be the optimal candidate from $\mathcal{J}$ that leads to the lowest regret while achieving zero constraint violation. Then we have 
		\begin{align*}
			\mathbb{E}\left[ \sum_{i=1}^{K/W}(R_i(\hat{B}) -R_i({B_i}))  \right] = &\tilde{\mathcal{O}}( H\sqrt{KW}+HK^{1-\lambda} )\\
			\mathbb{E}\left[ \sum_{i=1}^{K/W}  G_i(\hat{B}) - G_i({B_i})  \right] = & \tilde{\mathcal{O}}(HK^\lambda\sqrt{KW})
		\end{align*}
		\begin{proof}
			Apply the regret bound of the Exp3 algorithm, we have 
			\begin{align}
				&\mathbb{E}\left[ \sum_{i=1}^{K/W}( f_r(R_i(\hat{B})) + f_g(G_i( \hat{B})) - \sum_{i=1}^{K/W}( f_r(R_i({B_i})) + f_g(G_i(B_i))  \right] \\
				\leq & 2\sqrt{e-1}WH(1+ 1/K^\lambda)\sqrt{(K/W)(J+1)\ln(J+1)} = \tilde{\mathcal{O}}(H \sqrt{KW} ) ,
			\end{align}
			Recall that $\mathbb{E}[W\rho - G_i(\hat{B})  ] \leq 0 .$ Then it is easy to obtain
			\begin{align}
				&\mathbb{E}\left[ \sum_{i=1}^{K/W}(R_i(\hat{B}) -R_i({B_i}))  \right] \leq \mathbb{E}\left[ \sum_{i=1}^{K/W}( f_r(R_i(\hat{B})) -f_r(R_i({B_i})) ) \right]  \\
				\leq & 2\sqrt{e-1}WH(1+1/K^\lambda) \sqrt{(K/W)(J+1)\ln(J+1)}  + \mathbb{E}\left[ \sum_{i=1}^{K/W}  (f_g(G_i({B_i})) -f_g(G_i(\hat{B})))   \right]\\
				\leq & 2\sqrt{e-1}WH(1+ 1/K^\lambda) \sqrt{(K/W)(J+1)\ln(J+1)}  +  \frac{WH}{K^\lambda}\cdot \frac{K}{W} \\
				= &\tilde{\mathcal{O}}(H  \sqrt{KW} + HK^{1-\lambda} ),
			\end{align}
			where the last inequality due to the fact that the term $ \mathbb{E}\left[ \sum_{i=1}^{K/W}  (-f_g(G_i({\hat{B}})))   \right] $ is always non-positive. Furthermore, we have
			\begin{align}
				&\mathbb{E}\left[ \sum_{i=1}^{K/W}  G_i(\hat{B}) - G_i({B_i})  \right] =K^\lambda \mathbb{E}\left[ \sum_{i=1}^{K/W}  \frac{G_i(\hat{B}) - G_i({B_i})}{K^\lambda}  \right]\\
				= & K^\lambda \mathbb{E}\left[ \sum_{i=1}^{K/W}  f_g(G_i(\hat{B})) - f_g(G_i({B_i}))  \right]\\
				\leq & K^\lambda\left( 2\sqrt{e-1}WH(1+1/K^\lambda) \sqrt{(K/W)(J+1)\ln(J+1)} + \mathbb{E}\left[ \sum_{i=1}^{K/W}  (f_r(R_i({B_i})) -f_r(R_i(\hat{B})))   \right]    \right)\\
				\leq &K^\lambda\left( 2\sqrt{e-1}WH(1+1/K^\lambda) \sqrt{(K/W)(J+1)\ln(J+1)}  \right)\\
				= &  \tilde{\mathcal{O}}(HK^\lambda \sqrt{KW}),
			\end{align}
			where the last inequality is true because the second term is always non-positive. The reason is that when $\mathbb{E}[G_i(B_i)]\geq W\rho,$ $\mathbb{E}[f_r(R_i({B_i}))]\leq \mathbb{E}[f_r(R_i(\hat{B}))]$ because $\mathbb{E}[f_r(R_i(\hat{B}))] =  \mathbb{E}[R_i(\hat{B})]$ is the largest return, and when $\mathbb{E}[G_i(B_i)]< W\rho,$ we have $\mathbb{E}[f_r(R_i(B_i))]=0.$ 
		\end{proof}
	\end{lemma}
	
	\section{DETAILS PROOF OF THEOREM \ref{the:main}} \label{ap:proof-the-tabular}
	\subsection{Dynamic Regret}
	Recall that the regret can be decoupled as
	\begin{align}
		&  \hbox{Regret}(K)\nonumber\\
		= & \mathbb E \left[\sum_{k=1}^{K} \left( \sum_a  \left\{{Q}_{k,1}^{*}{q}^{*}_{k,1} -{Q}_{k,1}^{\epsilon,*}{q}^{\epsilon,*}_{k,1}\right\}(x_{k,1},a)    \right)\right]  +\label{step:epsilon-dif-c} \\
		&\mathbb E \left[\sum_{k=1}^{K}  \left(  \sum_a \left\{{Q}_{1}^{\epsilon,*}{q}^{\epsilon,*}_1\right\}(x_{k,1},a)-{Q}_{k,1}(x_{k,1}, a_{k,1}) \right)\right]+\label{step(i)-c}\\
		&\mathbb E \left[\sum_{k=1}^{K}  \left\{{Q}_{k,1}-  Q_{k,1}^{\pi_k}\right\}(x_{k,1}, a_{k,1}) \right].\label{step:biase-c}
	\end{align}
	
	Firstly, in lemma \ref{le:q-diff} we show that the first term can be bounded by comparing the original LP associated with the tightened LP such that 
	\begin{align}
		\eqref{step:epsilon-dif-c} \leq \frac{KH\epsilon}{\delta}. \label{ap:lp-lp-e}
	\end{align}
	
	By using Lemma \ref{le:qk-qpi-bound}, we can show that:
	\begin{align*}
		\eqref{step:biase-c} 
		\leq H^2SA K^{1-\alpha}{B}^c  +\frac{2(H^3\sqrt{\iota}+2H^4\tilde{b})K}{\chi} +  \sqrt{H^4SA\iota K^{2-\alpha} (\chi+1){B}^c   } + 2\tilde{b}H^2K
	\end{align*}
	For the  last term \ref{step(i)-c}, we first add and subtract additional terms to obtain
	
	\begin{align}
		&\mathbb E \left[\sum_{k=1}^{K}  \left(  \sum_a \left\{{Q}_{k,1}^{\epsilon,*}{q}^{\epsilon,*}_{k,1}\right\}(x_{k,1},a)- {Q}_{k,1}(x_{k,1}, a_{k,1}) \right)\right]\nonumber\\
		=&
		\mathbb{E}\left[ \sum_{k} \sum_a \left(\left\{{Q}_{k,1}^{\epsilon,*}{q}^{\epsilon,*}_{k,1}+\frac{Z_k}{\eta} {C}_{k,1}^{\epsilon,*}{q}^{\epsilon,*}_{k,1}\right\}(x_{k,1},a) -  \left\{{Q}_{k,1}{q}^{\epsilon,*}_{k,1} +\frac{Z_k}{\eta} {C}_{k,1}{q}^{\epsilon,*}_{k,1}\right\}(x_{k,1},a)\right)\right] \label{F-new_app}\\
		&+ \mathbb{E}\left[\sum_{k} \left(\sum_a  \left\{{Q}_{k,1} {q}^{\epsilon,*}_{k,1}\right\}(x_{k,1},a)	- {Q}_{k,1} (x_{k,1},a_{k,1})\right)\right] +\mathbb{E}\left[ \sum_{k} \frac{Z_k}{\eta} \sum_a \left\{\left({C}_{k,1} - {C}^{\epsilon,*}_{k,1} \right){q}^{\epsilon,*}_{k,1}\right\}(x_{k,1},a) \right].\label{eq:(i)expanded-new}
	\end{align}

	We can see \eqref{F-new_app} is the difference of two combined $Q$ functions. In Lemma \ref{le:qk-qpi-relation} we show that $\left\{Q_{k,h} +\frac{Z_k}{\eta}C_{k,h}\right\}(x,a)$ is an overestimate of $\left\{{Q}_{k,h}^{\epsilon,*}+\frac{Z_k}{\eta} C_{k,h}^{\epsilon,*}\right\}(x,a)$ (i.e. $\eqref{F-new_app}\leq 0$) with high probability. To bound \eqref{eq:(i)expanded-new}, we use the Lyapunov-drift method and consider Lyapunov function $L_T=\frac{1}{2} Z_T^2,$ where $T$ is the frame index and $Z_T$ is the value of the virtual queue at the beginning of the $T$th frame. We show that in Lemma \ref{le:drift} that the Lyapunov-drift satisfies
	
	\begin{align}
		\mathbb{E}[L_{T+1}-L_T] \leq  \hbox{a negative drift}+2H^4\iota +4H^4\tilde{b}^2 +\epsilon^2- \frac{\eta\bc}{K^\alpha}\sum_{k=TK^\alpha/\bc+1}^{(T+1)K^\alpha/\bc} \Phi_k ,
		\label{outline:drift_app}
	\end{align}
	where
	\begin{align*}
		\Phi_k=\mathbb{E}\left[ \left(\sum_a  \left\{  {Q}_{k,1} {q}^{\epsilon,*}_{k,1}\right\}(x_{k,1},a)	- {Q}_{k,1} (x_{k,1},a_{k,1})\right)\right] +\mathbb{E}\left[\frac{Z_k}{\eta} \sum_a \left\{\left({C}_{k,1} - {C}^{\epsilon,*}_{k,1}  \right){q}^{\epsilon,*}_{k,1}\right\}(x_{k,1},a) \right],
	\end{align*}
	
	So we can bound \eqref{eq:(i)expanded-new} by applying the telescoping sum over the $K^{1-\alpha}$ frames on the inequality above: 
	\begin{align}
		\eqref{eq:(i)expanded-new}=\sum_k\Phi_k \leq \frac{K^\alpha\bc \mathbb{E}\left[L_1-L_{K^{1-\alpha}+1}\right]}{\eta}+\frac{K(2H^4\iota +4H^4\tilde{b}^2 +\epsilon^2)}{\eta}\leq \frac{K(2H^4\iota +4H^4\tilde{b}^2 +\epsilon^2)}{\eta}, \label{eq:ldrift}
	\end{align}
	where the last inequality holds because $L_1=0$ and $L_T\geq 0$ for all $T.$ 
	Now combining Lemma \ref{le:qk-qpi-relation} and inequality \eqref{eq:ldrift}, we conclude that 
	\begin{align*}
		\eqref{step(i)-c}\leq \frac{K(2H^4\iota +4H^4\tilde{b}^2 +\epsilon^2)}{\eta}+	\frac{(\eta+K^{1-\alpha})H^2B^c }{\eta K}.
	\end{align*}
	Further combining inequality above we can obtain for $K \geq  \left(\frac{16\sqrt{SAH^6\iota^3}{B}^{1/3}}{\delta} \right)^5,$
	\begin{align*}
		\text{Regret} (K)&  \leq \frac{KH\epsilon}{\delta} +  H^2SA K^{1-\alpha}{B}^c  +\frac{2(H^3\sqrt{\iota}+2H^4\tilde{b})K}{\chi} +  \sqrt{H^4SA\iota K^{2-\alpha} (\chi+1){B}^c   } + 2\tilde{b}H^2K \\
		&+\frac{K(2H^4\iota +4H^4\tilde{b}^2 +\epsilon^2)}{\eta}+	\frac{(\eta+K^{1-\alpha})H^2B^c }{\eta K}.
	\end{align*}
	We conclude that under our choices of $\iota=128\log(\sqrt{2SAH}K), \epsilon= \frac{8\sqrt{SAH^6\iota^3}B^{1/3}}{K^{0.2}} $ and $\alpha=0.6, \eta=K^{\frac{1}{5}}{B}^{\frac{1}{3}},\chi=K^{\frac{1}{5}}, c=\frac{2}{3},$ and $K^{1-\alpha}B^c\tilde{b}\leq B,$ 
	$$\text{Regret} (K)=\tilde{\cal O}(H^4S^\frac{1}{2}A^{\frac{1}{2}}{B}^{\frac{1}{3}} K^{\frac{4}{5}}).$$
	\subsection{Constraint Violation}
	Again, we use $Z_T$ to denote the value of virtual-Queue in frame $T.$ According to the virtual-Queue update, we have 
	\begin{align*}
		Z_{T+1} = \left(  Z_T   + \rho + \epsilon -\frac{\bar{C}_T\bc}{K^\alpha}\right)^+ 
		\geq   Z_T   + \rho + \epsilon -\frac{\bar{C}_T\bc}{K^\alpha},
	\end{align*}
	which implies that 
	\begin{align*}
		\sum_{k=(T-1)K^\alpha/\bc +1}^{TK^\alpha/\bc}\left(-C_{k,1}^{\pi_k} (x_{k,1},a_{k,1}) +\rho \right) 
		\leq \frac{K^\alpha}{\bc}\left(Z_{T+1} - Z_T \right)+ \sum_{k=(T-1)K^\alpha/\bc+1}^{TK^\alpha/\bc} \left(\left\{C_{k,1}- C_{k,1}^{\pi_k} \right\} (x_{k,1},a_{k,1}) - \epsilon \right).  
	\end{align*}
	Summing the inequality above over all frames and taking expectation on both sides, we obtain the following upper bound on the constraint violation: 
	\begin{align}
		\mathbb{E} \left[\sum^{K}_{k=1}\rho-C_{k,1}^{\pi_k} (x_{k,1},a_{k,1})     \right]
		\leq  -K\epsilon + \frac{K^\alpha}{\bc} \mathbb{E}\left[ Z_{K^{1-\alpha}\bc+1} \right]+\mathbb{E}\left[\sum_{k=1}^K\left\{C_{k,1}- C_{k,1}^{\pi_k}\right\} (x_{k,1},a_{k,1})\right],\label{ap:violation}
	\end{align} 
	where we used the fact $Z_1=0.$ 
	
	In Lemma \ref{le:qk-qpi-bound}, we established an upper bound on the estimation error of $C_{k,1}:$
	\begin{align}
		&\mathbb{E}\left[\sum_{k=1}^K\left\{C_{k,1}-C_{1}^{\pi_k}\right\} (x_{k,1},a_{k,1})\right] \nonumber\\
		\leq& H^2SA K^{1-\alpha}{B}^c  +\frac{2(H^3\sqrt{\iota}+2H^4\tilde{b})K}{\chi} +  \sqrt{H^4SA\iota K^{2-\alpha} (\chi+1){B}^c   } + 2\tilde{b}H^2K . \label{eq:C-error_app}
	\end{align} 
	In Lemma \ref{le:zk-bound}, based on a Lyapunov drift analysis of this moment generating function and Jensen's inequality, we establish the following upper bound on $Z_T$ that holds for any $1\leq T\leq K^{1-\alpha}\bc+1$  
	\begin{align}
		\mathbb{E}[ Z_{T}]  \leq & \frac{100(H^4\iota+\tilde{b}^2H^2)}{\delta}\log \left( \frac{16(H^2\sqrt{\iota}+\tilde{b}H^2)}{\delta} \right) +\frac{4H^2B^c}{K\delta} + \frac{4H^2B^c}{\eta\delta K^\alpha}+ \frac{4\eta(\sqrt{H^2\iota}+2H^2\tilde{b}) }{\delta}. \label{ap:Z_T}
	\end{align} 
	
	Substituting the results from Lemmas \ref{le:qk-qpi-bound}  and \eqref{ap:Z_T} into \eqref{ap:violation}, under assumption $K \geq  \left(\frac{16\sqrt{SAH^6\iota^3}{B}^{1/3}}{\delta} \right)^5,$ which guarantees $\epsilon\leq \frac{\delta}{2}.$ Then by using the choice that $\epsilon = \frac{8\sqrt{SAH^6\iota^3}{B}^{1/3}}{ K^{0.2}},$ we can easily verify that
	\begin{align*}
		&\hbox{Violation}(K)  \leq \frac{100(H^4\iota+\tilde{b}^2H^2)K^{0.6} }{\delta{B}^{2/3}}\log{\frac{16(H^2\sqrt{\iota}+\tilde{b}H^2)}{\delta}} +\frac{4(H^2\sqrt{\iota}+2H^2\tilde{b})}{\delta{B}^{1/3} }K^{0.8}- 5\sqrt{SAH^6\iota^3}K^{0.8}{B}^{\frac{1}{3}}. 
	\end{align*}
	If further we have $K\geq e^{\frac{1}{\delta}},$ we can obtain 
	$$\hbox{Violation}(K) \leq \frac{100(H^4\iota+\tilde{b}^2H^2)K^{0.6} }{\delta{B}^{2/3}}\log{\frac{16(H^2\sqrt{\iota}+H^2\tilde{b})}{\delta}} - \sqrt{SAH^6\iota^3}K^{0.8}{B}^{\frac{1}{3}} = 0.$$
	
	
	
	\section{PROOF OF THEOREM \ref{the:double-triple-q}}\label{ap:tirple-q-double}
	Let $\hat{B}$ be the optimal candidate value in $\mathcal{J}$ that leads to the lowest regret while achieving zero constraint violation. Let $R_i(B_i)$ be the expected cumulative reward received in epoch $i$ with the estimated budget $B_i.$ Then the regret can be decomposed into:
	\begin{align*}
		\text{Regret} (K) = & \mathbb E\left[\sum_{k=1}^K\left(V_{k,1}^{\pi_k^*}(x_{k,1})-V_{k,1}^{\pi_k}(x_{k,1})\right)\right] \\
		= & \mathbb E\left[\sum_{k=1}^K V_{k,1}^{\pi_k^*}(x_{k,1})- \sum_{i=1}^{K/W} R_i(\hat{B}) \right]  +   \mathbb E\left[\sum_{i=1}^{K/W} R_i(\hat{B})-\sum_{i=1}^{K/W} R_i({B_i}) \right].
	\end{align*}
	The first term is the regret of using the optimal candidate $\hat{B}$ from $\mathcal{J};$ the second term is the difference between using $\hat{B}$ and $B_i$ which is selected by Exp3 algorithm. Applying the analysis of the Exp3 algorithm, we know that by using Lemma \ref{le:non-bandit-c} for any choice of $\hat{B},$ the second term is upper bounded:
	\begin{align*}
		\mathbb E\left[\left(\sum_{i=1}^{K/W} R_i(\hat{B})-\sum_{i=1}^{K/W} R_i({B_i}) \right)\right] \leq \tilde{\mathcal{O}}( H\sqrt{KW}+HK^{1-\lambda} ).
	\end{align*}
	For the first term, according to the regret bound analysis of Algorithm \ref{alg:triple-Q}, we have that 
	\begin{align}
		E\left[\sum_{k=1}^K\left(V_{k,1}^{\pi_k^*}(x_{k,1})- \sum_{i=1}^{K/W} R_i(\hat{B}) \right)\right]  \leq \tilde{\mathcal{O}}\left(H^4 S^{\frac{1}{2}}A^{\frac{1}{2}} K^{1-0.2\zeta} \left( \hat{B} \right)^{\frac{1}{3}}  \right)\label{eq:unknown-regret-1}.
	\end{align}
	We need to consider whether $B$ is covered in the range of $\mathcal{J}$ to further obtain the bound of \eqref{eq:unknown-regret-1}. First we assume that  $K = \Omega  \left( \left( \frac{40\sqrt{SAH^6\iota^3}{B}^{1/3}}{\delta} \right)^9\right) ,$ which implies $B \leq \frac{K^{1/3}W}{\Delta^{3/2}W}.$ Then we need to consider the following two cases:
	\begin{itemize}
		\item The first case is that $B$ is covered in the range of $\mathcal{J}.$ Note that two consecutive values in $\mathcal{J}$ only differ from each other by a factor of $W^{\frac{1}{J}},$ then there exists a value $\hat{B}\in\mathcal{J} $ such that $B\leq \hat{B}\leq W^{1/J}B.$ Therefore we can bound the RHS of \eqref{eq:unknown-regret-1} by 
		\begin{align*}
			\tilde{\mathcal{O}}\left(H^4 S^{\frac{1}{2}}A^{\frac{1}{2}} K^{1-0.2\zeta} \left( \hat{B} \right)^{\frac{1}{3}} \right) 
			\leq & \tilde{\mathcal{O}}\left(H^4 S^{\frac{1}{2}}A^{\frac{1}{2}} K^{1-0.2\zeta} \left({B}W^{{1}/{J}} \right)^{\frac{1}{3}}   \right) \\
			\leq & \tilde{\cal O}\left(H^4S^\frac{1}{2}A^{\frac{1}{2}}{B}^{\frac{1}{3}} K^{1-0.2\zeta}\right),
		\end{align*}
		where the last step comes from the fact $W^{1/J}=W^{1/(\ln W+1)}\leq e.$
		\item The second case is that $B$ is not covered in the range of $\mathcal{J},$ i.e., ${B} < \frac{K^{1/3}}{\Delta^{3/2}W}.$ The optimal candidate in $\mathcal{J}$ is the smallest such that one $\hat{B} = \frac{K^{1/3}}{\Delta^{3/2}W},$ then we can bound the RHS of \eqref{eq:unknown-regret-1} by 
		\begin{align*}
			\tilde{\mathcal{O}}\left(H^4 S^{\frac{1}{2}}A^{\frac{1}{2}} K^{1-0.2\zeta} \left( \hat{B} \right)^{\frac{1}{3}} \right) \leq &  \tilde{\mathcal{O}}\left(H^4 S^{\frac{1}{2}}A^{\frac{1}{2}} K^{1-0.2\zeta} \left(\frac{K^{1/3}}{\Delta^{3/2}W}  \right)^{\frac{1}{3}}\right) \\
			\leq &  \tilde{\mathcal{O}} \left(HK^{10/9-0.2\zeta}\frac{1}{K^{\zeta/3}} \right).
		\end{align*}
	\end{itemize} 
	
	For the constraint violation, according to Lemma \ref{le:non-bandit-c} we have 
	\begin{align*}
		& \mathbb{E} \left[\sum^{K}_{k=1}\rho-C_{k,1}^{\pi_k} (x_{k,1},a_{k,1}) \right] =  \mathbb{E} \left[\sum^{K/W}_{i=1} \left(W\rho - G_i(B_i)  \right) \right] \\
		=& \mathbb{E} \left[\sum^{K/W}_{i=1} \left(W\rho  - G_i(\hat{B}) \right) \right] + \mathbb{E}\left[\sum^{K/W}_{i=1} \left( G_i(\hat{B}) - G_i(B_i)  \right) \right] 
	\end{align*}
	For the first term, according to Theorem \ref{the:main}, by selecting $\epsilon$ as $\epsilon = \frac{20\sqrt{SAH^6\iota^3}{\hat{B}}^{1/3}}{ K^{0.2\zeta}}.$  we have 
	\begin{align}
		\mathbb{E} \left[\sum^{K/W}_{i=1} \left(W\rho  - G_i(\hat{B}) \right) \right] \leq \frac{100(H^4\iota+\tilde{b}^2H^2)K^{0.6\zeta} }{\delta{\hat{B}}^{2/3}}\log{\frac{16(H^2\sqrt{\iota}+H^2\tilde{b})}{\delta}} -13 \sqrt{SAH^6\iota^3}K^{1-0.2\zeta}{\hat{B}}^{\frac{1}{3}}.
	\end{align}
	For the second term, we are able to obtain an upper bound by using Lemma \ref{le:non-bandit-c} 
	\begin{align}
		&  \mathbb{E} \left[\sum_{i=1}^{K/W} (G_i(\hat{B}) - G_i(B_i))   \right]  \leq 12 K^\lambda H \sqrt{K^{1+\zeta}(J+1)\ln(J+1)}  
	\end{align}
	By balancing the terms $\tilde{O}(K^{1-0.2\zeta}), \tilde{O}(K^{\lambda + (1+\zeta)/2} ) $ and $K^{1-\lambda},$ the best selection are $\zeta=5/9$ and $\lambda=1/9.$ Therefore we further obtain when $K\geq e^{\frac{1}{\delta}},$
	\begin{align}
		\hbox{Violation}(K) \leq \frac{100(H^4\iota+\tilde{b}^2H^2)K^{1/3} }{\delta{\hat{B}}^{2/3}}\log{\frac{16(H^2\sqrt{\iota}+H^2\tilde{b})}{\delta}} - \sqrt{SAH^6\iota^3}K^{8/9}{\hat{B}}^{\frac{1}{3}} \leq 0.
	\end{align}
	We finish the proof of Theorem \ref{the:double-triple-q}.

	\begin{algorithm}[ht]
			\SetAlgoLined
			\textbf{Initialization: } $Y_1=0$, $w_{j,h}=0$, $\alpha=\dfrac{\log(|\cA|)K}{2(1+\xi+H)}$, $\eta=\xi/\sqrt{KH^2}$, $\beta=dH\sqrt{\log(2\log|\cA|dT/p)}$, $D=B^{-1/2}H^{-1/2}d^{1/2}K^{1/2}$.
			
			\For {frames $\cE=1,\ldots,K/D$}{
				\For {episodes $k=1,\ldots, D$}{
					Receive the initial state $x_1^k$.
					
					\For {step $h=H,H-1,\ldots, 1$}{
						$\Lambda^k_{h}\leftarrow \sum_{\tau=1}^{k-1}\phi(x_h^{\tau},a_h^{\tau})\phi(x_h^{\tau},a_h^{\tau})^T+\lambda \bI$\;
						$w^k_{r,h}\leftarrow (\Lambda^k_h)^{-1}[\sum_{\tau=1}^{k-1}\phi(x_h^{\tau},a_h^{\tau})[r_h(x_h^{\tau},a_h^{\tau})+V^k_{r,h+1}(x_{h+1}^{\tau})]]$ \;
						$w^k_{g,h}\leftarrow (\Lambda^k_{h})^{-1}[\sum_{\tau=1}^{k-1}\phi(x_h^{\tau},a_h^{\tau})[g_h(x_h^{\tau},a_h^{\tau})+V^k_{g,h+1}(x^{\tau}_{h+1})]]$ \;
						$Q^k_{r,h}(\cdot,\cdot)\leftarrow \min\{\langle w_{r,h}^k,\phi(\cdot,\cdot)\rangle+\beta(\phi(\cdot,\cdot)^T(\Lambda^k_{h})^{-1}\phi(\cdot,\cdot))^{1/2},H\}$ \;
						$Q^k_{g,h}(\cdot,\cdot)\leftarrow \min\{\langle w_{g,h}^k,\phi(\cdot,\cdot)\rangle+\beta(\phi(\cdot,\cdot)^T(\Lambda^k_{h})^{-1}\phi(\cdot,\cdot))^{1/2},H\}$ \;
						$\pi_{h,k}(a|\cdot)=\dfrac{\exp(\alpha(Q^k_{r,h}(\cdot,a)+Y_kQ^k_{g,h}(\cdot,a)))}{\sum_{a}\exp(\alpha(Q^k_{r,h}(\cdot,a)+Y_kQ^k_{g,h}(\cdot,a)))}$ \;
						$V^k_{r,h}(\cdot)=\sum_{a}\pi_{h,k}(a|\cdot)Q^k_{r,h}(\cdot,a)$  \;
						$V^k_{g,h}(\cdot)=\sum_{a}\pi_{h,k}(a|\cdot)Q^k_{g,h}(\cdot,a)$  \;
						
					}
					\For {step $h=1,\ldots,H$}{
						Compute $Q_{r,h}^k(x_h^k,a)$, $Q_{g,h}^k(x_h^k,a)$, $\pi(a|x_{h}^k)$ for all $a$ \;
						Take action $a_h^k\sim \pi_{h,k}(\cdot|x_h^k)$ and observe $x_{h+1}^k$ \;
					}
					$Y_{k+1}=\max\{\min\{Y_k+\eta(\rho-V^k_{g,1}(x_1)),\xi\},0\}$
				}
			}
		\caption{Model Free Primal-Dual Algorithm for Linear Function Approximation for Non-stationary Setting} 	\label{algo:model_free}
	\end{algorithm}

	
	\section{DETAILS PROOF OF THEOREM \ref{thm:linear_cmdp}}\label{proof:linear_cmdp}
	\textbf{Notations}: We describe the specific notations we have used in this section. With slight abuse of notations, in this section, we denote $V^{\pi}_{k,r,h}$ as the value function at step $h$ for policy $\pi$ at episode $k$. We denote $V^{\pi}_{k,g,h}$ as the utility value function at step $h$ of episode $k$. We denote $Q^{\pi}_{k,j,h}$, $j=r,g$ as the state-action value function at step $j$ for policy $\pi$.
	
	Throughout this section, we denote $Q_{r,h}^k, Q_{g,h}^k,w_{r,h}^k, w_{g,h}^k, \Lambda_h^k$ as the $Q$-value and the parameter values estimated at the episode $k$. $V_{j,h}^k(\cdot)=\langle \pi_{h,k}(\cdot|\cdot),Q_{j,h}^{k}(\cdot,\cdot)\rangle_{\cA}$. $\pi_{h,k}(\cdot|x)$ is the soft-max policy based on the composite $Q$-function at the $k$-th episode as $Q_{r,h}^k+Y_kQ_{g,h}^k$. To simplify the presentation, we denote $\phi_h^k=\phi(x_h^k,a_h^k)$.

	\subsection{Outline of Proof of Theorem~\ref{thm:linear_cmdp}}
	\textbf{Step 1:} The key to prove both the dynamic regret and violation is to show the following
	\begin{lemma}\label{lem:dual_variable}
		For any $Y\in [0,\xi]$, 
		\begin{align}
			&\sum_{k=1}^K (V_{k,r,1}^{\pi^*_k}(x_1) - V_{k,r,1}^{\pi_k}(x_1)) + Y \sum_{k=1}^K (\rho-V_{k,g,1}^{\pi_k}(x_1)) 
			\le \frac{1}{2\eta} Y^2 + \frac{\eta}{2}H^2K + \nonumber\\& \underbrace{\sum_{k=1}^K \left(V_{k,r,1}^{\pi^*_k}(x_1) + Y_kV_{k,g,1}^{\pi^*_k}(x_1) \right) - \left(V_{r,1}^k(x_1) + Y_k V_{g,1}^k(x_1)\right)}_{\mathcal{T}_1}
			+\nonumber\\ &  \underbrace{\sum_{k=1}^K \left(V_{r,1}^k(x_1) - V_{k,r,1}^{\pi_k}(x_1)\right) + Y \sum_{k=1}^K \left(V_{g,1}^k(x_1) - V_{k,g,1}^{\pi_k}(x_1)\right)}_{\mathcal{T}_2}
		\end{align}
	\end{lemma}
	Note that when $Y=0$, we recover the dynamic regret. The proof is in Appendix~\ref{proof:lem_dualvariable}.
	
	\textbf{Step-2}: In order to bound $\cT_1$, and $\cT_2$, we use the following result
	\begin{lemma}\label{lem:t_1andt_2}
		With probability $1-2p$,
		\begin{align}
			\cT_1\leq H^3(1+2/\delta)BD^{3/2}\sqrt{d}+\dfrac{KH\log(|\cA|)}{\alpha}\nonumber\\
			\cT_2\leq(1+Y)(\mathcal{O}(\sqrt{H^4d^3K^2\iota^2/D})+\sqrt{d}D^{3/2}BH^2) 
		\end{align}
	\end{lemma}
	The proof is in Appendix~\ref{proof:lem_t1t2}.
	
	\textbf{Step-3}: The final result is obtained by combining all the pieces. 
	
	\textit{Proof of Theorem~\ref{thm:linear_cmdp}}: 
	
	Note from Lemma~\ref{lem:dual_variable} we have
	\begin{align}
		\sum_{k=1}^K(V_{k,r,1}^{\pi_k^*}(x_1)-V_{k,r,1}^{\pi_k}(x_1))+Y(\rho-V_{k,g,1}^{\pi_k}(x_1))\leq \dfrac{Y^2}{2\eta}+\dfrac{\eta KH^2}{2}+\cT_1+\cT_2\nonumber
	\end{align}
	From Lemma~\ref{lem:t_1andt_2}, we obtain
	\begin{align}
		& \sum_{k=1}^K(V_{k,r,1}^{\pi_k^*}(x_1)-V_{k,r,1}^{\pi_k}(x_1))+Y(\rho-V_{k,g,1}^{\pi_k}(x_1))\leq  \dfrac{Y^2}{2\eta}+\dfrac{\eta KH^2}{2}+\nonumber\\ & \dfrac{HK\log(|\cA|)}{\alpha}+H^3(1+2/\delta)BD^{3/2}\sqrt{d}+(1+Y)(\mathcal{O}(\sqrt{H^4d^3K^2\iota^2/D})+\sqrt{d}D^{3/2}BH^2) 
	\end{align}
	Since $\eta=\dfrac{\xi}{\sqrt{KH^2}}$,$\alpha=\dfrac{\log(|\cA|)K}{2(1+\xi+H)}$, $D=B^{-1/2}H^{-1/2}d^{1/2}K^{1/2}$, we obtain
	\begin{align}\label{eq:composite_expression}
		& \sum_{k=1}^K(V_{k,r,1}^{\pi_k^*}(x_1)-V_{k,r,1}^{\pi_k}(x_1))+Y(b-V_{k,g,1}^{\pi_k}(x_1))\leq  \xi\sqrt{KH^2}\nonumber\\
		& +H2(1+\xi+H)+H^{9/4}(1+2/\delta)B^{1/4}K^{3/4}d^{5/4}+(Y+1)(\mathcal{O}(H^{9/4}d^{5/4}K^{3/4}B^{1/4}\iota^2)+H^{5/4}d^{5/4}K^{3/4}) 
	\end{align}
	Since the above expression is true for any $Y\in [0,\xi]$, thus, plugging $Y=0$, we obtain
	\begin{align}
		\mathrm{Regret}(K)\leq \mathcal{O}(H^{9/4}d^{5/4}K^{3/4}B^{1/4}\iota^2)+\mathcal{O}((1+1/\delta)H^{9/4}d^{5/4}K^{3/4}B^{1/4})\nonumber
	\end{align}
	For the constraint violation bound, we use Lemma~\ref{lem:copy}.
	Note that $\xi\geq 2\max_k\mu^{k,*}$. Thus, we replace $Y=\xi$ in (\ref{eq:composite_expression}).  Thus, from (\ref{eq:composite_expression}) and Lemma~\ref{lem:copy}, we obtain
	\begin{align}
		\sum_{k=1}^K(\rho-V_{k,g,1}^{\pi}(x_1))\leq \dfrac{2(1+\xi)}{\xi}(\mathcal{O}(H^{9/4}d^{5/4}K^{3/4}B^{1/4}\iota^2)+\mathcal{O}(H^{5/4}d^{5/4}K^{3/4}B^{1/4}))
	\end{align}
	Hence, the result follows.\qed
	
	\subsection{Proof of Lemma~\ref{lem:dual_variable}}\label{proof:lem_dualvariable}
	We first state and prove the following result which is similar to the one proved in \cite{GhoZhoShr_22}. 
	\begin{lemma}\label{lem:violation}
		For $Y\in [0,\xi]$,
		\begin{align}
			\sum_{k=1}^{K}(Y-Y_k)(\rho-V_{g,1}^k(x_1))\leq \dfrac{Y^2}{2\eta}+\dfrac{\eta H^2K}{2}
		\end{align}
	\end{lemma}
	\begin{proof}
		\begin{align}
			& |Y_{k+1}-Y|^2=|Proj_{[0,\xi]}(Y_k+\eta(\rho-V_{g,1}^k(x_1)))-Proj_{[0,\xi]}(Y)|^2\nonumber\\
			& \leq (Y_k+\eta(\rho-V_{g,1}^k(x_1)))-Y)^2\nonumber\\
			& \leq (Y_k-Y)^2+\eta^2H^2+2\eta Y_k(\rho-V_{g,1}^k(x_1))
		\end{align}
		Summing over $k$, we obtain
		\begin{align}
			& 0\leq |Y_{K+1}-Y|^2\leq
			|Y_1-Y|^2+2\eta\sum_{k=1}^{K}(\rho-V_{g,1}^{k}(x_1))(Y_k-Y)+\eta^2H^2K\nonumber\\
			& \sum_{k=1}^{K}(Y-Y_k)(\rho-V_{g,1}^{k}(x_1))\leq
			\dfrac{|Y_1-Y|^2}{2\eta}+\dfrac{\eta H^2K}{2}
		\end{align}
		Since $Y_1=0$, we have the result. 
	\end{proof}
	
	Now, we prove Lemma~\ref{lem:dual_variable}.
	\begin{proof}
		Note that
		\begin{align}
			Y \sum_{k=1}^K (\rho-V_{k,g,1}^{\pi_k}(x_1)) & = \sum_{k}(Y-Y_k)(\rho-V_{g,1}^{k}(x_1))+Y_k(\rho-V_{g,1}^k)+Y(V_{g,1}^k(x_1)-V_{k,g,1}^{\pi_k}(x_1))\nonumber\\
			& \leq \dfrac{1}{2\eta}Y^2+\dfrac{\eta}{2}H^2K+\sum_{k=1}^K(Y_k\rho-Y_kV_{g,1}^k(x_1))+Y(V_{g,1}^k(x_1)-V_{g,1}^{\pi_k}(x_1))\nonumber\\
			& \leq \dfrac{1}{2\eta}Y^2+\dfrac{\eta}{2}H^2K+\sum_{k=1}^K(Y_kV_{k,g,1}^{\pi^*_k}(x_1)-Y_kV_{g,1}^k(x_1))+\sum_{k=1}^KY(V_{g,1}^k(x_1)-V_{k,g,1}^{\pi_k}(x_1))\nonumber
		\end{align}
		where the first inequality follows from Lemma~\ref{lem:violation}, and the second inequality follows from the fact that $V_{k,g,1}^{\pi^*_k}(x_1)\geq \rho$. Hence, the result simply follows from the above inequality.
	\end{proof}
	
	\subsection{Proof of Lemma~\ref{lem:t_1andt_2}}\label{proof:lem_t1t2}
	We now move on to bound $\cT_1$ and $\cT_2$. First, we state and prove Lemmas~\ref{lem:phi}, \ref{lem:q_diff}, \ref{lem:q_combined_diff}, \ref{lem:close_optimal},\ref{lem:t1}, and \ref{lm:recursion}. 
	
	\begin{lemma}\label{lem:phi}
		There exists a constant $C_2$ such that for any fixed $p\in (0,1)$, if we let $E$ be the event that
		\begin{align}
			\|\sum_{\tau=1}^{k-1}\phi_{j,h}^{\tau}[V_{j,h+1}^{k}(x_{h+1}^{\tau})-\bP_{k,h}V_{j,h+1}^{k}(x_h^{\tau},a_h^{\tau})]\|_{(\Lambda_h^k)^{-1}}\leq C_2dH\sqrt{\chi}
		\end{align}
		for all $j\in \{r,g\}$, $\chi=\log[2(C_1+1)\log(|\cA|) dT/p]$, for some constant $C_2$, then $\Pr(E)=1-2p$.
	\end{lemma}
	This result is similar to the  concentration lemma, which is crucial in controlling the fluctuations in
	least-squares value iteration as done in \cite{JinYanZha_20}. The proof relies on the uniform concentration lemma similar to \cite{JinYanZha_20}. However, there is an additional $\log(|\cA|)$ in $\chi$. This arises due to the fact that the policy (Algorithm~\ref{algo:model_free}) is soft-max unlike the greedy policy in \cite{JinYanZha_20}. \cite{GhoZhoShr_22} shows that greedy policy is unable to prove the uniform concentration lemma. The proof is similar to Lemma 8 in \cite{GhoZhoShr_22}, thus, we remove it. 
	
	Now, we introduce some notations which we use throughout this paper.
	
	For any $k\in \cE$,i.e., any episode $k$ within the frame $\cE$, we define the variation as the following
	\begin{align}
		B_{j,\cE}^k=\sum_{\tau=2}^{k}\sum_{h=1}^H||\theta_{\tau,j,h}-\theta_{\tau-1,j,h}||, B_j^{\cE}=\sum_{\tau=2}^{\cE}\sum_{h=1}^H||\theta_{\tau,j,h}-\theta_{\tau-1,j,h}||\nonumber\\
		B_{p,\cE}^k=\sum_{\tau=2}^{k}\sum_{h=1}^H||\mu_{\tau,h}-\mu_{\tau-1,h}||, B_p^{\cE}=\sum_{\tau=2}^{\cE}\sum_{h=1}^H||\mu_{\tau,h}-\mu_{\tau-1,h}||\nonumber
	\end{align}
	These are local budget variation. Note that $|\cE|=D$.

	Now, we are bound the difference between our estimated $Q_{j,h}^k$ and $Q_{k,j,h}^{\pi}$. Using the Lemma~\ref{lem:phi}, we show the following
	\begin{lemma}\label{lem:q_diff}
		There exists an absolute constant $\beta=C_1dH\sqrt{\iota}$, $\iota=\log(\log(|\cA|)2dT/p)$, and for any fixed policy $\pi$, on the event $E$ defined in Lemma~\ref{lem:phi}, we have 
		\begin{align}
			\langle \phi(x,a),w_{j,h}^k\rangle-Q_{k,j,h}^{\pi}(x,a)=& \bP_{k,h}(V_{j,h+1}^k-V^{\pi}_{k,j,h+1})(x,a)+\Delta_h^k(x,a)+
			+ B_{j}^{\cE}\sqrt{dD}+HB_{p}^{\cE}\sqrt{dD}
		\end{align}
		for some $\Delta_h^k(x,a)$ that satisfies $|\Delta_h^k(x,a)|\leq \beta\sqrt{\phi(x,a)^T(\Lambda_h^k)^{-1}\phi(x,a)}$, for any $k\in \cE$.
	\end{lemma}
	\begin{proof}
		We only prove for $j=r$, the proof for $j=g$ is similar.
		
		Note that $Q_{k,r,h}^{\pi}(x,a)=\langle\phi(x,a),w_{r,h}^{\pi}\rangle=r_{k,h}(x,a)+\bP_{k,h}V_{k,r,h+1}^{\pi}(x,a)$.

		Hence, we have
		\begin{align}\label{eq:diff}
			& w_{r,h}^k-w_{k,r,h}^{\pi}=
			(\Lambda_h^k)^{-1}\sum_{\tau=1}^{k-1}\phi_h^{\tau}[r_h^{\tau}+V_{r,h+1}^k(x_{h+1}^{\tau})]
			-w_{k,r,h}^{\pi}\nonumber\\
			& =-\lambda(\Lambda_h^k)^{-1}(w_{k,r,h}^{\pi})+ (\Lambda_h^k)^{-1}\sum_{\tau=1}^{k-1}\phi_h^{\tau}[r_{\tau,h}(x_h^{\tau},a_h^{\tau})+V_{r,h+1}^k-r_{k,h}(x_h^{\tau},a_h^{\tau})-\bP_{k,h}V_{k,r,h+1}^{\pi}]
		\end{align}
		
		In the above expression, the second term of the right hand-side can be written as
		\begin{align}
			& (\Lambda_h^k)^{-1}\sum_{\tau=1}^{k-1}\phi_h^{\tau}[r_{\tau,h}(x_h^{\tau},a_h^{\tau})+V_{r,h+1}^k-r_{k,h}(x_h^{\tau},a_h^{\tau})-\bP_{k,h}V_{k,r,h+1}^{\pi}]\nonumber\\
			& =(\Lambda_h^k)^{-1}\sum_{\tau=1}^{k-1}\phi_h^{\tau}[r_{\tau,h}(x_h^{\tau},a_h^{\tau})+V_{r,h+1}^k-r_{k,h}(x_h^{\tau},a_h^{\tau})-\bP_{k,h}V_{r,h+1}^k]+(\Lambda_h^k)^{-1}\sum_{\tau=1}^{k-1}\phi_h^{\tau}[\bP_{k,h}V_{r,h+1}^k-\bP_{k,h}V_{k,r,h+1}^{\pi}]\nonumber\\
			& =(\Lambda_h^k)^{-1}\sum_{\tau=1}^{k-1}\phi_h^{\tau}[r_{\tau,h}(x_h^{\tau},a_h^{\tau})-r_{k,h}(x_h^{\tau},a_h^{\tau})]+(\Lambda_h^k)^{-1}\sum_{\tau=1}^{k-1}[V_{r,h+1}^k-\bP_{\tau,h}V_{r,h+1}^k]\nonumber\\& +(\Lambda_h^k)^{-1}\sum_{\tau=1}^{k-1}[\bP_{\tau,h}V_{r,h+1}^k-\bP_{k,h}V_{r,h+1}^k] +(\Lambda_h^k)^{-1}\sum_{\tau=1}^{k-1}\phi_h^{\tau}[\bP_{k,h}V_{r,h+1}^k-\bP_{k,h}V_{k,r,h+1}^{\pi}]
		\end{align}
		By plugging in the above in (\ref{eq:diff}) we obtain
		\begin{align}
			& w_{r,h}^k-w_{k,r,h}^{\pi}\nonumber\\
			&= \underbrace{-\lambda(\Lambda_h^k)^{-1}(w_{k,r,h}^{\pi})}_{q_1}+\underbrace{(\Lambda_h^k)^{-1}\sum_{\tau=1}^{k-1}\phi_h^{\tau}[r_{\tau,h}(x_h^{\tau},a_h^{\tau})-r_{k,h}(x_h^{\tau},a_h^{\tau})]}_{q_2}+\underbrace{(\Lambda_h^k)^{-1}\sum_{\tau=1}^{k-1}[V_{r,h+1}^k-\bP_{\tau,h}V_{r,h+1}^k]}_{q_3}\nonumber\\& +\underbrace{(\Lambda_h^k)^{-1}\sum_{\tau=1}^{k-1}[\bP_{\tau,h}V_{r,h+1}^k-\bP_{k,h}V_{r,h+1}^k]}_{q_4} +\underbrace{(\Lambda_h^k)^{-1}\sum_{\tau=1}^{k-1}\phi_h^{\tau}[\bP_{k,h}V_{r,h+1}^k-\bP_{k,h}V_{k,r,h+1}^{\pi}]}_{q_5}
		\end{align}
		For the first term,
		\begin{align}
			|\langle \phi(x,a),q_1\rangle|\leq \phi(x,a)^T(\Lambda_h^k)^{-1}\lambda w_{k,r,h}^{\pi}\leq ||w_{k,r,h}^{\pi}||||\phi(x,a)||_{(\Lambda_h^k)^{-1}}
		\end{align}
		For the second term  we have 
		\begin{align}
			& \phi(x,a)^T(\Lambda_h^k)^{-1}\sum_{\tau=1}^{k-1}\phi_h^{\tau}[r_{\tau,h}(x_h^{\tau},a_h^{\tau})-r_{k,h}(x_h^{\tau},a_h^{\tau})]\nonumber\\
			& \leq \phi(x,a)^T(\Lambda_h^k)^{-1}\sum_{\tau=1}^{k-1}\phi_h^{\tau}||\phi_h^{\tau}||||\theta_{\tau,r,h}-\theta_{k,r,h}||\nonumber\\
			& \leq  \phi(x,a)^T(\Lambda_h^k)^{-1}\sum_{\tau=1}^{k-1}\phi_h^{\tau}||\phi_h^{\tau}||||\sum_{s=\tau}^{k-1}\theta_{s,r,h}-\theta_{s+1,r,h}||\nonumber\\
			& \leq B^k_{r}\sqrt{dk}||\phi(x,a)||_{(\Lambda_h^k)^{-1}}\nonumber
		\end{align}
		The last inequality follows from Lemma C.4 in \cite{JinYanZha_20}. Since $||\phi(x,a)||_{(\Lambda_h^k)^{-1}}\leq \sqrt{1/\lambda}$ and $D\geq k$. We have
		\begin{align}
			|\langle \phi(x,a),q_2\rangle|\leq B_r^{\cE}\sqrt{dD}
		\end{align}
		
		Similarly, we can bound
		\begin{align}
			\phi(x,a)^T(\Lambda_h^k)^{-1}\sum_{\tau=1}^{k-1}\phi_h^{\tau}[\bP_{\tau,h}V_{r,h+1}^k-\bP_{k,h}V_{r,h+1}^k]
			\leq HB^k_p\sqrt{dk}||\phi(x,a)||_{(\Lambda_h^k)^{-1}}
		\end{align}
		Again since $D\geq k$, and $||\phi(x,a)||_{(\Lambda_h^k)^{-1}}\leq \sqrt{1/\lambda}$, we have
		\begin{align}
			|\langle \phi(x,a),q_3\rangle|\leq HB_p^{\cE}\sqrt{dD}
		\end{align}
		From Lemma, the fourth term can be bounded as
		\begin{align}
			|\langle \phi(x,a),q_4\rangle|\leq CdH\sqrt{\chi}
		\end{align}
		For the fifth term, note that
		\begin{align}\label{eq:diff2}
			& \langle \phi(x,a),q_5\rangle=\langle \phi(x,a),(\Lambda_h^k)^{-1}\sum_{\tau=1}^{k-1}\phi_h^{\tau}[\bP_h(V_{r,h+1}^{k}-V_{k,r,h+1}^{\pi})(x_h^{\tau},a_h^{\tau})]\rangle\nonumber\\
			& =\langle\phi(x,a),(\Lambda_h^k)^{-1}\sum_{\tau=1}^{k-1}\phi_h^{\tau}(\phi_h^{\tau})^T\int(V_{r,h+1}^k-V_{k,r,h+1}^{\pi})(x^{\prime})d\mu_{k,h}(x^{\prime})\rangle\nonumber\\
			& =\langle\phi(x,a),\int(V_{r,h+1}^k-V_{k,r,h+1}^{\pi})(x^{\prime})d\mu_{k,h}(x^{\prime})\rangle -\langle\phi(x,a),\lambda(\Lambda_h^k)^{-1}\int(V_{r,h+1}^k-V_{r,h+1}^{\pi})(x^{\prime})d\mu_{k,h}(x^{\prime})\rangle
		\end{align}
		The last term  in (\ref{eq:diff2}) can be bounded as the following
		\begin{align}\label{eq:diffq31}
			|\langle\phi(x,a),\lambda(\Lambda_h^k)^{-1}\int(V_{r,h+1}^k-V_{k,r,h+1}^{\pi})(x^{\prime})d\mu_{k,h}(x^{\prime})\rangle|\leq 2H\sqrt{d\lambda}\sqrt{\phi(x,a)^T(\Lambda_h^k)^{-1}\phi(x,a)}
		\end{align}
		since $||\int(V_{r,h+1}^k-V_{r,h+1}^{\pi})(x^{\prime})d\mu_{k,h}(x^{\prime})||_2\leq 2H\sqrt{d}$ as $||\mu_{k,h}(\cS)||\leq \sqrt{d}$. The first term in (\ref{eq:diff2}) is equal to
		\begin{align}\label{eq:diffq32}
			\bP_{k,h}(V_{r,h+1}^k-V_{r,h+1}^{\pi})(x,a)
		\end{align}
		Note that $\langle \phi(x,a),w_{r,h}^k\rangle-Q_{k,r,h}^{\pi}(x,a)=\langle \phi(x,a),w_{r,h}^k-w_{k,r,h}^{\pi}\rangle=\langle \phi(x,a),q_1+q_2+q_3+q_4+q_5\rangle$, we have
		\begin{align}
			\langle\phi(x,a),w_{j,h}^k\rangle-Q_{k,j,h}^{\pi}=& \bP_{k,h}(V_{j,h+1}^k-V_{k,j,h+1}^{\pi})(x,a)+\Delta_h^k+ B_{r}^{\cE}\sqrt{dD}+HB_{p}^{\cE}\sqrt{dW}
		\end{align}
		where $|\Delta_h^k|\leq \beta \sqrt{\phi(x,a)^T(\Lambda_h^k)^{-1}\phi(x,a)}$.
	\end{proof}
	Using Lemma~\ref{lem:q_diff}, we also bound the difference between the combined $Q$-function (estimated) and the actual $Q$-function.
	\begin{lemma}\label{lem:q_combined_diff}
		With probability $1-2p$,
		\begin{align}
			Q_{k,r,h}^{\pi}+Y_kQ_{k,g,h}^{\pi}\geq & Q_{r,h}^{k}+Y_kQ_{g,h}^k+\bP_{k,h}(V_{k,r,h+1}^{\pi}+Y_kV_{k,g,h+1}^{\pi}-V_{r,h+1}^k-Y_kV_{g,h+1}^k)\nonumber\\
			& +B_r^{\cE}\sqrt{dD}+Y_kB_g^{\cE}\sqrt{dD}+(1+Y_k)HB_p^{\cE}\sqrt{dD}
		\end{align}
	\end{lemma}
	\begin{proof}
		From Lemma~\ref{lem:q_diff}, we have
		\begin{align}
			Q_{k,r,h}^{\pi}\leq \langle\phi(x,a),w_{r,h}^k\rangle+\bP_{k,h}(V_{k,r,h+1}^{\pi}-V_{r,h}^k)+\beta||\phi(x,a)||_{\Lambda_{k,h}^{-1}}+B_{r}^{\cE}\sqrt{dD}+HB_p^{\cE}\sqrt{dD}
		\end{align}
		From the definition of $Q_{j,h}^k$, we have
		\begin{align}
			Q_{k,r,h}^{\pi}\leq \bP_{k,h}(V_{k,r,h+1}^{\pi}-V_{r,h}^k)+Q_{r,h}^k+B_{r}^{\cE}\sqrt{dD}+HB_p^{\cE}\sqrt{dD}
		\end{align}
		Similarly,
		\begin{align}
			Y_kQ_{k,g,h}^{\pi}\leq Y_k\bP_{k,h}(V_{k,g,h+1}^{\pi}-V_{g,h}^k)+Y_kQ_{g,h}^k+Y_kB_{g}^{\cE}\sqrt{dD}+Y_kHB_p^{\cE}\sqrt{dD}
		\end{align}
	\end{proof}
	We now show that using the soft-max parameter $\alpha$, one can bound the difference between the best estimated value function and the one achieved using the soft-max policy. 
	
	\begin{lemma}\label{lem:close_optimal}
		Then, $\bar{V}_{h}^k(x)-V_h^{k}(x)\leq \dfrac{\log|\cA|}{\alpha}$
	\end{lemma}
	where 
	\begin{definition}\label{defn:barvhk}
		$\bar{V}_{h}^k(\cdot)=\max_{a}[Q_{r,h}^k(\cdot,a)+Y_kQ_{g,h}^k(\cdot,a)]$.
	\end{definition}
	$\bar{V}_{h}^k(\cdot)$ is the value function corresponds to the greedy-policy with respect to the composite $Q$-function.

	\begin{proof}
		Note that
		\begin{align}
			V_h^{k}(x)=\sum_{a}\pi_{h,k}(a|x)[Q_{r,h}^k(x,a)+Y_kQ_{g,h}^k(x,a)]
		\end{align}
		where
		\begin{align}\label{eq:boltz}
			\pi_{h,k}(a|x)=\dfrac{\exp(\alpha[Q_{r,h}^k(x,a)+Y_kQ_{g,h}^k(x,a)])}{\sum_{a}\exp(\alpha[Q_{r,h}^k(x,a)+Y_kQ_{g,h}^k(x,a)])}
		\end{align}
		Denote $a_x=\arg\max_{a}[Q_{r,h}^k(x,a)+Y_kQ_{g,h}^k(x,a)]$
		
		Now, recall from Definition~\ref{defn:barvhk} that $\bar{V}_{h}^k(x)=[Q_{r,h}^k(x,a_x)+Y_kQ_{g,h}^k(x,a_x)]$. Then,
		\begin{align}\label{eq:uppb}
			& \bar{V}_{h}^k(x)-V_{h}^{k}(x)=[Q_{r,h}^k(x,a_x)+Y_kQ_{g,h}^k(x,a_x)]\nonumber\\& - \sum_{a}\pi_{h,k}(a|x)[Q_{r,h}^k(x,a)+Y_kQ_{g,h}^k(x,a)]\nonumber\\
			& \leq \left(\dfrac{\log(\sum_{a}\exp(\alpha(Q_{r,h}^k(x,a)+Y_kQ_{g,h}^k(x,a))))}{\alpha}\right)\nonumber\\& -\sum_{a}\pi_{h,k}(a|x)[Q_{r,h}^k(x,a)+Y_kQ_{g,h}^k(x,a)]\nonumber\\
			& \leq \dfrac{\log(|\cA|)}{\alpha}
		\end{align}
		where the last inequality follows from Proposition 1 in \cite{PanCaiMen_19}.
	
	\end{proof}
	Using the above result, we bound the difference $\cT_1$ (albeit for each episode). 
	
	\begin{lemma}\label{lem:t1}
		With probability $1-2p$,
		\begin{align}
			(V_{k,r,1}^{\pi_k^*}(x_1)+Y_kV_{k,g,1}^{\pi_k^*}(x_1))-(V_{r,1}^k(x_1)+Y_kV_{g,1}^k(x_1))\leq \dfrac{H\log(|\cA|)}{\alpha}+H(B_r^{\cE}\sqrt{D}+Y_kB_g^{\cE}\sqrt{D}+(1+Y_k)HB_p^{\cE}\sqrt{D})\nonumber
		\end{align}
	\end{lemma}
	\begin{proof}
		First, we prove for the step $H$. 
		
		Note that $Q_{j,H+1}^k=0=Q_{j,H+1}^{\pi}$.
		
		Under the event in $E$ as described in Lemma~\ref{lem:phi} and from Lemma~\ref{lem:q_diff}, we have for $j=r,g$,
		\begin{align}
			& |\langle\phi(x,a),w_{j,H}^k(x,a)\rangle-Q_{j,H}^{\pi}(x,a)| \leq  \beta\sqrt{\phi(x,a)^T(\Lambda_H^k)^{-1}\phi(x,a)}
			+B_{j}^{\cE}\sqrt{dD}+HB_{p}^{\cE}\sqrt{dD}\nonumber
		\end{align}
		Hence, for any $(x,a)$,
		\begin{align}
			Q_{j,H}^{\pi}(x,a)& \leq \min\{\langle\phi(x,a),w_{j,H}^k\rangle+\beta\sqrt{\phi(x,a)^T(\Lambda_H^k)^{-1}\phi(x,a)}+B_{j}^{\cE}\sqrt{dD}+HB_{p}^{\cE}\sqrt{dD},H\}\nonumber\\& 
			\leq Q_{j,H}^k(x,a)+B_{j}^{\cE}\sqrt{dD}+HB_{p}^{\cE}\sqrt{dD}
		\end{align}
		Hence, from the definition of $\bar{V}_h^k$,
		\begin{align}
			\bar{V}_{H}^k(x)& =\max_{a}[Q_{r,H}^k(x,a)+Y_kQ_{g,h}^k(x,a)]\nonumber\\
			& \geq \sum_{a}\pi(a|x)[Q_{r,H}^{\pi}(x,a)+Y_kQ_{g,H}^{\pi}(x,a)]\nonumber\\& -(B_{r}^{\cE}\sqrt{dD}+Y_kB_g^{\cE}\sqrt{dD}+(1+Y_k)HB_{p}^{\cE}\sqrt{dD})\nonumber\\
			& \geq V_H^{\pi,Y_k}(x)-(B_{r}^{\cE}\sqrt{dD}+Y_kB_g^{\cE}\sqrt{dD}+H(1+Y_k)B_{p}^{\cE}\sqrt{dD})
		\end{align}
		for any policy $\pi$.  Thus, it also holds for $\pi^{*}_k$, the optimal policy. Hence, from Lemma~\ref{lem:close_optimal}, we have 
		\begin{align}
			V_H^{\pi^*_k,Y_k}(x)-V_{H}^{k}(x)\leq \dfrac{\log(|\cA|)}{\alpha}+(B_{r}^{\cE}\sqrt{dD}+Y_kB_g^{\cE}\sqrt{dD}+(1+Y_k)HB_{p}^{\cE}\sqrt{dD})\nonumber
		\end{align}

		Now, suppose that it is true till the step $h+1$ and consider the step $h$.
		
		Since, it is true till step $h+1$, thus, for any policy $\pi$,
		\begin{align}
			\bP_{k,h}(V_{h+1}^{\pi,Y_k}-V_{h+1}^k)(x,a)\leq & \dfrac{(H-h)\log(|\cA|)}{\alpha}\nonumber\\& +(H-h)(B_{r}^{\cE}\sqrt{dW}+Y_kB_g^{\cE}\sqrt{dW}+(1+Y_k)HB_{p}^{\cE}\sqrt{dW})
		\end{align}
		
		From Lemma~\ref{lem:q_diff}  we have for any $(x,a)$
		\begin{align}
			& Q_{k,r,h}^{\pi}(x,a)+Y_kQ_{k,g,h}^{\pi}(x,a)  \leq  Q_{r,h}^{k}(x,a)+Y_kQ_{g,h}^k(x,a)+\dfrac{(H-h)\log(|\cA|)}{\alpha}\nonumber\\
			& +(H-h+1)(B_{r}^{\cE}\sqrt{dD}+Y_kB_g^{\cE}\sqrt{dD}+(1+Y_k)HB_{p}^{\cE}\sqrt{dD})
		\end{align}
		Hence, 
		\begin{align}
			V^{\pi,Y_k}_h(x)\leq \bar{V}_{h}^k(x)+\dfrac{(H-h)\log(|\cA|)}{\alpha}+(H-h+1)(B_{r}^{\cE}\sqrt{dW}+Y_kB_g^{\cE}\sqrt{dD}+(1+Y_k)HB_{p}^{\cE}\sqrt{dD})\nonumber
		\end{align}
		Now, again from Lemma~\ref{lem:close_optimal}, we have $\bar{V}_{h}^{k}(x)-V_{h}^k(x)\leq \dfrac{\log(|\cA|)}{\alpha}$. Thus,
		\begin{align}
			V^{\pi,Y_k}_h(x)-V_{h}^k(x)\leq \dfrac{(H-h+1)\log(|\cA|)}{\alpha}+(H-h+1)(B_{r}^{\cE}\sqrt{dD}+Y_kB_g^{\cE}\sqrt{dD}+(1+Y_k)HB_{p}^{\cE}\sqrt{dD})
		\end{align}
		Now, since it is true for any policy $\pi$, it will be true for $\pi^{*}_k$. From the definition of $V^{\pi,Y_k}$, we have
		\begin{align}
			& \left(V^{\pi^*}_{r,h}(x)+Y_kV^{\pi^*}_{g,h}(x)\right)-\left(V_{r,h}^k(x)+Y_kV_{g,h}^k(x)\right)\leq  \dfrac{(H-h+1)\log(|\cA|)}{\alpha}\nonumber\\
			& +(H-h+1)(B_{r}^{\cE}\sqrt{dD}+Y_kB_g^{\cE}\sqrt{dD}+(1+Y_k)HB_{p}^{\cE}\sqrt{dD})
		\end{align}
		Hence, the result follows by summing over $K$ and considering $h=1$. 
	\end{proof}
	
	We now focus on bounding $\cT_2$.
	First, we introduce some notations.
	
	Let
	\begin{align}\label{eq:d_martingale}
		& D_{j,h,1}^k=\langle(Q_{j,h}^k(x_{h}^k,\cdot)-Q_{j,h}^{\pi_{k}}(x_{h}^k,\cdot)),\pi_{h,k}(\cdot|x_h^k)\rangle-(Q_{j,h}^{k}(x_h^k,a_h^k)-Q_{j,h}^{\pi_k}(x_h^k,a_h^k))\nonumber\\
		& D_{j,h,2}^k=\bP_{k,h}(V_{j,h+1}^k-V_{j,h+1}^{\pi_k})(x_h^k,a_h^k)-[V_{j,h+1}^k-V_{j,h+1}^{\pi_k}](x_{h+1}^k)
	\end{align}
	
	\begin{lemma}\label{lm:recursion}
		On the event defined in $E$ in Lemma~\ref{lem:phi}, we have
		\begin{align}
			V_{j,1}^{k}(x_1)-V_{k,j,1}^{\pi_k}\leq\sum_{h=1}^{H}(D_{j,h,1}^k+D_{j,h,2}^k)+\sum_{h=1}^{H}2\beta\sqrt{\phi(x_{h}^k,a_{h}^k)^T(\Lambda_h^k)^{-1}\phi(x_{h}^k,a_{h}^k)}\nonumber\\
			+H(B_{j}^{\cE}\sqrt{dD}+HB_{p}^{\cE}\sqrt{dD})
		\end{align}
	\end{lemma}
	\begin{proof}
		By Lemma~\ref{lem:q_diff}, for any $x,h,a,k$
		\begin{align}
			& \langle w_{j,h}^k(x,a),\phi(x,a)\rangle+\beta\sqrt{\phi(x,a)^T(\Lambda_h^k)^{-1}\phi(x,a)}-Q_{j,h}^{\pi_k} \nonumber\\
			& \leq \bP_{k,h}(V_{j,h+1}^k-V_{k,j,h+1}^{\pi_k})(x,a)+2\beta\sqrt{\phi(x,a)^T(\Lambda_h^k)^{-1}\phi(x,a)}
			+H(B_{j}^{\cE}\sqrt{dD}+HB_{p}^{\cE}\sqrt{dD})\nonumber
		\end{align}
		Thus, 
		\begin{align}\label{eq:q_d}
			& Q_{j,h}^{k}(x,a)-Q_{j,h}^{\pi_k}(x,a) \leq \bP_{k,h}(V_{j,h+1}^k-V_{k,j,h+1}^{\pi_k})(x,a)+2\beta\sqrt{\phi(x,a)^T(\Lambda_h^k)^{-1}\phi(x,a)}\nonumber\\
			&  +H(B_{r}^{\cE}\sqrt{dD}+B_g^{\cE}\sqrt{dD}+HB_{p}^{\cE}\sqrt{dD})\nonumber\\
			& \bP_{k,h}(V_{j,h+1}^k-V_{k,j,h+1}^{\pi_k})(x,a)+2\beta\sqrt{\phi(x,a)^T(\Lambda_h^k)^{-1}\phi(x,a)}+\nonumber\\& B_{j}^{\cE}\sqrt{dD}+HB_{p}^{\cE}\sqrt{dD}-(Q_{j,h}^{k}(x,a)-Q_{k,j,h}^{\pi_k}(x,a))\geq 0
		\end{align}
		Since $V_{j,h}^{k}(x)=\sum_{a}\pi_{h,k}(a|x)Q_{j,h}^k(x,a)$ and $V_{k,j,h}^{\pi_k}(x)=\sum_{a}\pi_{h,k}(a|x)Q_{k,j,h}^{\pi_k}(x,a)$ where $\pi_{h,k}(a|\cdot)=\textsc{Soft-Max}_{\alpha}^a(Q_{r,h}^k+Y_kQ_{g,h}^k)$ $\forall a$.
		
		Thus, from (\ref{eq:q_d}),
		\begin{align}\label{eq:recursive}
			& V_{j,h}^{k}(x_h^k)-V_{k,j,h}^{\pi_k}(x_h^k)=\sum_{a}\pi_{h,k}(a|x_{h}^k)[Q_{j,h}^{k}(x_{h}^k,a)-Q_{k,j,h}^{\pi_k}(x_{h}^k,a)]\nonumber\\
			& \leq \sum_{a}\pi_{h,k}(a|x_h^k)[Q_{j,h}^{k}(x_{h}^k,a)-Q_{k,j,h}^{\pi_k}(x_{h}^k,a)]+(B_{j}^{\cE}\sqrt{dD}+HB_{p}^{\cE}\sqrt{dD})\nonumber\\
			& +2\beta\sqrt{\phi(x_{h}^k,a_{h}^k)^T(\Lambda_h^k)^{-1}\phi(x_{h}^k,a_{h}^k)}+\bP_{k,h}(V_{j,h+1}^k-V_{j,h+1}^{\pi_k})(x_{h}^k,a_h^k)-(Q_{j,h}^{k}(x_h^k,a_h^k)-Q_{k,j,h}^{\pi_k}(x_h^k,a_h^k))
		\end{align}
		
		Thus, from (\ref{eq:recursive}), we have
		\begin{align}
			V_{j,h}^k(x_h^k)-V_{j,h}^{\pi_k}(x_h^k)\leq&  D_{j,h,1}^k+D_{j,h,2}^k+[V_{j,h+1}^k-V_{j,h+1}^{\pi_k}](x_{h+1}^k)+2\beta\sqrt{\phi(x_{h}^k,a_{h}^k)^T(\Lambda_h^k)^{-1}\phi(x_{h}^k,a_{h}^k)}\nonumber\\
			& +(B_{j}^{\cE}\sqrt{dD}+HB_{p}^{\cE}\sqrt{dD})
		\end{align}
		Hence, by iterating recursively, we have
		\begin{align}
			V_{j,1}^{k}(x_1)-V_{j,1}^{\pi_k}\leq\sum_{h=1}^{H}(D_{j,h,1}^k+D_{j,h,2}^k)+\sum_{h=1}^{H}2\beta\sqrt{\phi(x_{h}^k,a_{h}^k)^T(\Lambda_h^k)^{-1}\phi(x_{h}^k,a_{h}^k)}+H(B_{j}^{\cE}\sqrt{dD}+HB_{p}^{\cE}\sqrt{dD})
		\end{align}
		The result follows.
	\end{proof}
	Now, we are ready to prove Lemma~\ref{lem:t_1andt_2}.
	
	\textbf{Proof of Lemma~\ref{lem:t_1andt_2}}
	\begin{proof}
		First, from Lemma~\ref{lem:t1},
		\begin{align}
			(V_{k,r,1}^{\pi_k^*}(x_1)+Y_kV_{k,g,1}^{\pi_k^*}(x_1))-(V_{r,1}^k(x_1)+Y_kV_{g,1}^k(x_1))\leq & \dfrac{H\log(|\cA|)}{\alpha}\nonumber\\
			& +H(B_r^{\cE}\sqrt{dD}+Y_kB_g^{\cE}\sqrt{dD}+(1+Y_k)HB_p^{\cE}\sqrt{dD})
		\end{align}
		Note that $Y_k=2H/\delta$. Now, summing over $k$ within frame $\cE$ we obtain
		\begin{align}
			&\sum_{k=1}^{D}(V_{k,r,1}^{\pi_k^*}(x_1)+Y_kV_{k,g,1}^{\pi_k^*}(x_1))-(V_{r,1}^k(x_1)+Y_kV_{g,1}^k(x_1))\leq \nonumber\\& \dfrac{HD\log(|\cA|)}{\alpha}+H\sqrt{d}(B_r^{\cE}D^{3/2}+2H/\delta B_g^{\cE}D^{3/2}+(1+2H/\delta)HB_p^{\cE}D^{3/2})
		\end{align}
		Now, summing over the epochs $\cE$, we obtain
		\begin{align}
			& \sum_{\cE=1}^{K/D}\sum_{k=1}^D(V_{k,r,1}^{\pi_k^*}(x_1)+Y_kV_{k,g,1}^{\pi_k^*}(x_1))-(V_{r,1}^k(x_1)+Y_kV_{g,1}^k(x_1))\leq \dfrac{HK\log(|\cA|)}{\alpha}\nonumber\\& +\sum_{\cE=1}^{K/D}H\sqrt{d}(B_r^{\cE}D^{3/2}+2H/\delta B_g^{\cE}D^{3/2}+(1+2H/\delta)HB_p^{\cE}D^{3/2})\nonumber\\
			& \leq  \dfrac{HK\log(|\cA|)}{\alpha}+H^2(1+2H/\delta)\sqrt{d}BD^{3/2}
		\end{align}
		where we have used the fact that $\sum_{\cE}(B_r^{\cE}+B_g^{\cE}+B_p^{\cE})=B_r+B_g+B_p=B$. This gives the bound for $\cT_1$. Now, we bound $\cT_2$.
		
		From Lemma~\ref{lm:recursion},
		\begin{align}\label{eq:term_t2}
			\sum_{k=1}^D(V_{j,1}^k(x_1)-V_{j,1}^{\pi_k}(x_1))&\leq \sum_{k=1}^D\sum_{h=1}^{H}(D_{j,h,1}^k+D_{j,h,2}^k)+\sum_{k=1}^D\sum_{h=1}^{H}2\beta\sqrt{\phi(x_{h}^k,a_{h}^k)^T(\Lambda_h^k)^{-1}\phi(x_{h}^k,a_{h}^k)}\nonumber\\ &
			+\sum_{\cE=1}^{K/D}\sum_{k=1}^DH(B_{j}^{\cE}\sqrt{dD}+HB_{p}^{\cE}\sqrt{dD})
		\end{align}
		We, now, bound the individual terms of the right-hand side in (\ref{eq:term_t2}). First, we show that the first term corresponds to a Martingale difference.
		
		For any $(k,h)\in [\cE]\times [H]$, we define $\cF_{h,1}^k$ as $\sigma$-algebra generated by the state-action sequences, reward, and constraint values, $\{(x_i^{\tau},a_i^{\tau})\}_{(\tau,i)\in [k-1]\times [H]}\cup \{(x^k_{i},a_i^k)\}_{i\in [h]}$. 
		
		Similarly, we define the $\cF_{h,2}^k$ as the $\sigma$-algebra generated by $\{(x_i^{\tau},a_i^{\tau})\}_{(\tau,i)\in [k-1]\times [H]}\cup \{(x^k_{i},a_i^k)\}_{i\in [h]}\cup\{x_{h+1}^k\}$. $x_{H+1}^k$ is a null state for any $k\in [K]$. 
		
		A filtration is a sequence of $\sigma$-algebras $\{\cF_{h,m}^k\}_{(k,h,m)\in [\cE]\times[H]\times[2]}$ in terms of time index
		\begin{align}
			t(k,h,m)=2(k-1)H+2(h-1)+m
		\end{align}
		which holds that $\cF_{h,m}^k\subset \cF_{h^{\prime},m^{\prime}}^{k^{\prime}}$ for any $t\leq t^{\prime}$. 
		
		Note from the definitions in (\ref{eq:d_martingale}) that $D_{j,h,1}^k\in \mathcal{F}_{h,1}^k$ and $D_{j,h,2}^k\in \mathcal{F}_{h,2}^k$. Thus, for any $(k,h)\in [K]\times [H]$, 
		\begin{align}
			\mathbb{E}[D_{j,h,1}^k|\cF_{h-1,2}^k]=0, \quad \mathbb{E}[D_{j,h,2}^k|\cF_{h,1}^k]=0
		\end{align}
		Notice that $t(k,0,2)=t(k-1,H,2)=2(H-1)k$. Clearly, $\cF_{0,2}^k=\cF_{H,2}^{k-1}$ for any $k\geq 2$. Let $\cF_{0,2}^1$ be empty. We define a Martingale sequence
		\begin{align}
			M_{j,h,m}^k& = \sum_{\tau=1}^{k-1}\sum_{i=1}^{H}(D_{j,i,1}^{\tau}+D_{j,i,2}^{\tau})+\sum_{i=1}^{h-1}(D_{j,i,1}^{k}+D_{j,i,2}^k)+\sum_{l=1}^{m}D_{j,h,l}^k\nonumber\\
			& =\sum_{(\tau,i,l)\in [\cE]\times[H]\times[2], t(\tau,i,l)\leq t(k,h,m)}D^{\tau}_{j,i,l}
		\end{align}
		where $t(k,h,m)=2(k-1)H+2(h-1)+m$ is the time index. Clearly, this martingale is adopted to the filtration $\{\cF_{h,m}^k\}_{(k,h,m)\in [D]\times [H]\times [2]}$, and particularly
		\begin{align}
			\sum_{k=1}^D\sum_{h=1}^H(D_{j,h,1}^k+D_{j,h,2}^k)=M_{j,H,2}^D
		\end{align}
		
		Thus, $M_{j,H,2}^K$ is a Martingale difference satisfying $|M_{j,H,2}^D|\leq 4H$ since $|D_{j,h,1}^k|,|D_{j,h,2}^k|\leq 2H$
		From the Azuma-Hoeffding inequality, we have
		\begin{align}
			\Pr(M_{j,H,2}^D> s)\leq 2\exp(-\dfrac{s^2}{16DH^2})
		\end{align}
		With probability $1-p/2$ at least for any $j=r,g$,
		\begin{align}
			\sum_{k}\sum_{h}M_{j,H,2}^D\leq \sqrt{16DH^2\log(4/p)}
		\end{align}
		
		Now, we bound the second term of the right-hand side of (\ref{eq:term_t2}). Note that the minimum eigen value of $\Lambda_h^k$ is at least $\lambda=1$ for all $(k,h)\in [D]\times [H]$. By Lemma~\ref{lem:le1}, 
		\begin{align}
			\sum_{k=1}^{K}(\phi_h^k)^T(\Lambda_h^k)^{-1}\phi_h^k\leq 2\log\left[\dfrac{\det(\Lambda_h^{k+1})}{\det(\Lambda_h^1)}\right]
		\end{align}
		Moreover, note that $||\Lambda_h^{k+1}||=||\sum_{\tau=1}^{k}\phi_h^k(\phi_h^k)^T+\lambda\bI||\leq \lambda+k$, hence,
		\begin{align}
			\sum_{k=1}^{D} (\phi_h^k)^T(\Lambda_h^k)^{-1}\phi_h^k\leq 2d\log\left[\dfrac{\lambda+k}{\lambda}\right]\leq 2d\iota
		\end{align}
		Now, by Cauchy-Schwartz inequality, we have
		\begin{align}
			\sum_{k=1}^{D}\sum_{h=1}^H \sqrt{(\phi_h^k)^T(\Lambda_h^k)^{-1}\phi_h^k}& \leq \sum_{h=1}^{H}\sqrt{W}[\sum_{k=1}^{K}(\phi_h^k)^T(\Lambda_h^k)^{-1}\phi_h^k]^{1/2}
			 \leq H\sqrt{2dD\iota}
		\end{align}
		Note that $\beta=C_1dH\sqrt{\iota}$. Hence, the second term is bounded by
		\begin{align}
			\cO(\sqrt{H^4d^3D\iota^2})
		\end{align}
		
		The third term of (\ref{eq:term_t2}) is bounded by
		\begin{align}
			\sum_{k=1}^DH(B_{j}^{\cE}\sqrt{dD}+HB_{p}^{\cE}\sqrt{dD})=\sqrt{d}D^{3/2}H(B_j^{\cE}+HB_p^{\cE})
		\end{align}
		Hence, summing (\ref{eq:term_t2}) over the epochs we obtain
		\begin{align}
			\sum_{\cE=1}^{K/D}\sum_{k=1}^D(V_{j,1}^k(x_1)-V_{j,1}^{\pi_k}(x_1))\leq \sum_{\cE=1}^{K/D}\mathcal{O}(\sqrt{H^4d^3D\iota^2})+\sum_{\cE=1}^{K/D}\sqrt{d}D^{3/2}H(B_j^{\cE}+HB_p^{\cE})
		\end{align}
		Replacing $\sum_{\cE}B_j^{\cE}=B_j$, and $\sum_{\cE}B_p^{\cE}=B_p$, we obtain
		\begin{align}
			\sum_{\cE=1}^{K/D}\sum_{k=1}^D(V_{j,1}^k(x_1)-V_{k,j,1}^{\pi_k}(x_1))\leq \mathcal{O}(\sqrt{H^4d^3K^2\iota^2/D})+\sqrt{d}D^{3/2}BH^2
		\end{align}
		Thus,
		\begin{align}
			\sum_{k=1}^K(V_{r,1}^k(x_1)-V_{k,r,1}^{\pi_k}(x_1))+Y(V_{g,1}^k(x_1)-V_{k,g,1}^{\pi_k}(x_1))\leq (1+Y)(\mathcal{O}(\sqrt{H^4d^3K^2\iota^2/D})+\sqrt{d}D^{3/2}BH^2)
		\end{align}
		Hence, the result follows.
	\end{proof}
	
	\subsection{Supporting Results}
	\begin{lemma}
		Under Definition~\ref{defn:linearmdp}, for any fixed policy $\pi$, let $w_{k,j,h}^{\pi}$ be the corresponding weights such that $Q^{\pi}_{k,j,h}=\langle \phi(x,a),w_{k,j,h}^{\pi}\rangle$, for $j\in \{r,g\}$, then we have for all $h\in [H]$ and $k\in [K]$
		\begin{align}
			||w_{k,j,h}^{\pi}||\leq 2H\sqrt{d}
		\end{align}
	\end{lemma}
	\begin{proof}
		From the linearity of the action-value function, we have
		\begin{align}
			Q_{k,j,h}^{\pi}(x,a)&=j_{k,h}(x,a)+\bP_{k,h}V_{k,j,h}^{\pi}(x,a)\nonumber\\
			&  = \langle \phi(x,a),\theta_{j,h}\rangle+\int_{\mathcal{S}}V_{k,j,h+1}^{\pi}(x^{\prime})\langle \phi(x,a),d\mu_{k,h}(x^{\prime})\rangle\nonumber\\
			&  =\langle\phi(x,a),w_{k,j,h}^{\pi}\rangle
		\end{align}
		where $w_{j,h}^{\pi}=\theta_{j,h}+\int_{\mathcal{S}}V_{j,h+1}^{\pi}(x^{\prime})d\mu_h(x^{\prime})$.
		
		Now, $||\theta_{j,h}||\leq \sqrt{d}$, and $||\int_{\mathcal{S}}V_{j,h+1}^{\pi}(x^{\prime})d\mu_h(x^{\prime})||\leq H\sqrt{d}$. Thus, the result follows. 
	\end{proof}
	\begin{lemma}\label{lem:w}
		For any $(k,h)$, the weight $w_{j,h}^k$ satisfies 
		\begin{align}
			||w_{j,h}^k||\leq 2H\sqrt{dk/\lambda}
		\end{align}
	\end{lemma}
	\begin{proof}
		For any vector $v\in \mathcal{R}^d$ we have
		\begin{align}\label{eq:inter}
			|v^Tw_{j,h}^k|=|v^T(\Lambda_h^k)^{-1}\sum_{\tau=1}^{k-1}\phi^{\tau}_h(x_h^{\tau},a_h^{\tau})(j_h(x_h^{\tau},a_h^{\tau})+\sum_{a}\pi_{h+1,k}(a|x_{h+1}^{\tau})Q_{j,h+1}^k(x_{h+1}^{\tau},a))|
		\end{align}
		here $\pi_{h,k}(\cdot|x)$ is the Soft-max policy.
		
		Note that $Q_{j,h+1}^k(x,a)\leq H$ for any $(x,a)$. Hence, from (\ref{eq:inter}) we have
		\begin{align}
			|v^Tw_{j,h}^k|& \leq  \sum_{\tau=1}^{k-1}|v^T(\Lambda_h^k)^{-1}\phi^{\tau}_h|.2H\nonumber\\
			& \leq \sqrt{\sum_{\tau=1}^{k-1}v^T(\Lambda^h_k)^{-1}v}\sqrt{\sum_{\tau=1}^{k-1}\phi_h^{\tau}(\Lambda_h^k)^{-1}\phi_h^{\tau}}.2H\nonumber\\
			& \leq 2H||v||\dfrac{\sqrt{dk}}{\sqrt{\lambda}}
		\end{align}
		Note that $||w_{j,h}^k||=\max_{v:||v||=1}|v^Tw_{j,h}^k|$. Hence, the result follows. 
	\end{proof}
	The following result is shown in \cite{AbbPalSze_11} and in Lemma D.2 in \cite{JinYanZha_20}.
	\begin{lemma}\label{lem:le1}
		Let $\{\phi_t\}_{t\geq 0}$ be a sequence in $\Re^d$ satisfying $\sup_{t\geq 0}||\phi_t||\leq 1$. For any $t\geq 0$, we define $\Lambda_t=\Lambda_0+\sum_{j=0}^{t}\phi_j\phi_j^T\phi_j$. Then if the smallest eigen value of $\Lambda_0$ be at least $1$, we have
		\begin{align}
			\log\left[\dfrac{\det(\Lambda_h^{k+1})}{\det(\Lambda_h^1)}\right]\leq \sum_{k=1}^{K}(\phi_h^k)^T(\Lambda_h^k)^{-1}\phi_h^k\leq 2\log\left[\dfrac{\det(\Lambda_h^{k+1})}{\det(\Lambda_h^1)}\right]
		\end{align}
	\end{lemma}
	
	We use the following result (Lemma J.10 in \cite{DingLav_22}). 
	\begin{lemma}\label{lem:copy}
		Let $\bar{C}^*\geq 2\max_k\mu^{k,*}$, then, if
		\begin{align}
			\sum_{k=1}^K(V_{k,r,1}^{\pi_k^*}(x_1)-V_{k,r,1}^{\pi_k}(x_1))+2\bar{C}^*\sum_{k=1}^K(b_k-V_{k,g,1}^{\pi_k}(x_1))\leq \delta
		\end{align}, then
		\begin{align}
			\sum_{k=1}^K(b_k-V_{k,g,1}^{\pi_k}(x_1))\leq \dfrac{2\delta}{\bar{C}^*}
		\end{align}
	\end{lemma}
	
	\begin{algorithm}[ht]
			\SetAlgoLined
			{Choose} $W=K^{1/2},\cJ$ (defined in Eq. \eqref{eq:j-linear}), $ \gamma_0 = \min\left\{1,\sqrt{\frac{(K/W)\log(K/W)  }{(e-1)KH}}  
			\right\}  , \lambda=1/8$ \;
			Initialize weights of the bandit arms $s_1(j)=1,\forall j=0,1,\dots,J $ \;
			\For{epoch  $i = 1,\dots, \frac{K}{W} $ }{
				Update $p_i(j)\leftarrow (1-\delta)\frac{s_i(j)}{\sum_{j'=0}^Js_i(j')}+\frac{\gamma_0}{J+1} ,\forall j=0,1,\dots,J $ \;
				Draw  an arm $A_i\in [J]$ randomly according to the probabilities $p_i(0),\dots,p_i(J)$ \;
				Set  the estimated budget $B_i \leftarrow  \frac{\sqrt{K}W^{\frac{A_i}{J}} }{\Delta W  }  $ \;
				Run a new instance of Algorithm \ref{algo:model_free} for $W$ episodes with parameter value $B\leftarrow B_i$\;
				Observe the cumulative reward $R_i$ and utility $G_i.$\; 
				\For{arm j=0,1,\dots,J}{
					$\hat{R}_i(j) =  \left\{
					\begin{aligned}
						(&  G_i/K^\lambda )  I_{\{j=A_i \}} / (WH(1+1/K^\lambda)p_i(j))  & \text{ if } G_i < W\rho   \\
						(& R_i + G_i/K^\lambda) I_{\{j=A_i \}} / (WH(1 +1/K^\lambda) p_i(j))  & \text{ if } G_i \geq W\rho 
					\end{aligned}
					\right. $ \tcp*{\small normalization } 
					$s_{i+1}\leftarrow s_i(j) \exp(\gamma_0 \hat{R}_i(j) /(J+1) ) $\;
				}
			}
		\caption{Model Free Primal-Dual Algorithm for Linear Function Approximation for Non-stationary Setting without knowing the variation budget} 	\label{algo:model_free_unknown}
	\end{algorithm}
	\section{DETAILS PROOF OF THEOREM THEOREM \ref{the:linear-non}} \label{ap:proof-linear-non}
	Let $W=K^{\zeta}$ and 
	\begin{align}
		\mathcal{J} =  \left\{ \frac{\sqrt{K}}{\Delta W}, \frac{\sqrt{K} W^{\frac{1}{J}}}{\Delta W },\frac{\sqrt{K}W^{\frac{2}{J}}}{\Delta W},\dots, \frac{\sqrt{K}W}{\Delta W}  \right\},\Delta = \left( \frac{6(1+\xi)}{\xi\delta} \tilde{\cO}((1+\delta)d^{5/4}H^{9/4} )   \right)^{4}   \label{eq:j-linear}
	\end{align}
	where $J=\log W$ as the candidate sets for $B$ in the linear CMDPs. Under assumption $K^{1/8} \geq \dfrac{6(1+\xi)}{\xi \delta}\tilde{O}((1+1/\delta)d^{5/4}B^{1/4}H^{9/4}$ we know the optimal budget $B\in\cJ.$ Let $\hat{B}$ be any candidate value in $\mathcal{J}$ that leads to the lowest regret while achieving zero constraint violation. Let $R_i(B_i)$ be the expected cumulative reward received in epoch $i$ with the estimated epoch length $B$. Then the regret can be decomposed into:
	\begin{align*}
		\text{Regret} (K) = & \mathbb E\left[\sum_{k=1}^K\left(V_{k,1}^{\pi_k^*}(x_{k,1})-V_{k,1}^{\pi_k}(x_{k,1})\right)\right] \\
		= & \mathbb E\left[\sum_{k=1}^K V_{k,1}^{\pi_k^*}(x_{k,1})- \sum_{i=1}^{K/W} R_i(\hat{B}) \right]  +   \mathbb E\left[\sum_{i=1}^{K/W} R_i(\hat{B})-\sum_{i=1}^{K/W} R_i({B_i}) \right].
	\end{align*}
	The first term is the regret of using the candidate $\hat{B}$ from $\mathcal{J};$ the second term is the difference between using $\hat{B}$ and $B_i$ which is selected by Exp3 algorithm.  Applying the analysis of the Exp3 algorithm, we know that by using Lemma \ref{le:non-bandit-c} for any choice of $\hat{B},$ the second term is upper bounded:
	\begin{align*}
		\mathbb E\left[\left(\sum_{i=1}^{K/W} R_i(\hat{B})-\sum_{i=1}^{K/W} R_i({B_i}) \right)\right] \leq \tilde{\mathcal{O}}( H\sqrt{KW}+HK^{1-\lambda} ).
	\end{align*}
	For the first term, according to the regret bound analysis of Algorithm \ref{algo:model_free}, we have for the  $W$ episodes
	\begin{align}
		E\left[\sum_{k=1}^W\left(V_{k,1}^{\pi_k^*}(x_{k,1})-  R_i(\hat{D}) \right)\right]  \leq \tilde{\mathcal{O}}\left(\frac{1+\delta}{\delta} K^{1-\frac{1\zeta}{4}} H^{9/4}d^{5/4}\hat{B}^{1/4}  \right) \label{eq:unknown-regret-linear-1}.
	\end{align}

	We need to consider whether $\hat{B}$ is covered in the range of $\mathcal{J}$ to further obtain the bound of \eqref{eq:unknown-regret-linear-1}. We consider the following two cases
	\begin{itemize}
		\item The first case is that optimal $B$  is covered in the range of $\mathcal{J}.$ Note that two consecutive values in $\mathcal{J}$ only differ from each other by a factor of $W^{\frac{1}{J}},$ then there exists a value $\hat{B}\in\mathcal{J} $ such that $B\leq \hat{B}\leq W^{1/J}B.$ Therefore we can bound the RHS of \eqref{eq:unknown-regret-linear-1} by 
		\begin{align*}
			\tilde{\mathcal{O}}\left(\frac{1+\delta}{\delta} K^{1-\frac{1\zeta}{4}} H^{9/4}d^{5/4}\hat{B}^{1/4}  \right) &\leq \tilde{\mathcal{O}}\left(\frac{1+\delta}{\delta} K^{1-\frac{1\zeta}{4}} H^{9/4}d^{5/4}{W^{1/J}B}^{1/4}  \right) \\
			\leq  &\tilde{\mathcal{O}}\left(\frac{1+\delta}{\delta} K^{1-\frac{1\zeta}{4}} H^{9/4}d^{5/4}{eB}^{1/4}  \right)\\
			= & \tilde{\mathcal{O}}\left(\frac{1+\delta}{\delta} K^{1-\frac{1\zeta}{4}} H^{9/4}d^{5/4}{B}^{1/4}  \right)
		\end{align*}
	
		\item The second case is that $B$ is not covered in the range of $\mathcal{J} ,$i.e., $B\leq \frac{\sqrt{K}}{\Delta W},$ then the optimal candidate value in $\cJ$ is $\frac{\sqrt{K}}{\Delta W}$,we can bound the RHS of \eqref{eq:unknown-regret-linear-1} by 
		\begin{align*}
			\tilde{\mathcal{O}}\left(\frac{1+\delta}{\delta} K^{1-\frac{1\zeta}{4}} H^{9/4}d^{5/4}{\hat{B}}^{1/4}  \right)
			\leq  \tilde{\mathcal{O}}\left(\frac{1+\delta}{\delta} K^{1-\frac{1\zeta}{4}} H^{9/4}d^{5/4}{ (\frac{\sqrt{K}}{\Delta W}})^{1/4}  \right)
		\end{align*}
	\end{itemize}

	For the constraint violation, according to Lemma \ref{le:non-bandit-c} we have 
	\begin{align*}
		& \mathbb{E} \left[\sum^{K}_{k=1}\rho-C_{k,1}^{\pi_k} (x_{k,1},a_{k,1}) \right] =  \mathbb{E} \left[\sum^{K/W}_{i=1} \left(W\rho - G_i(B_i)  \right) \right] \\
		=& \mathbb{E} \left[\sum^{K/W}_{i=1} \left(W\rho  - G_i(\hat{B}) \right) \right] + \mathbb{E}\left[\sum^{K/W}_{i=1} \left( G_i(\hat{B}) - G_i(B_i)  \right) \right] 
	\end{align*}
	For the first term, according to Theorem \ref{thm:linear_cmdp}, by selecting $\epsilon=\dfrac{3(1+\xi)}{\xi}\tilde{O}((1+1/\delta)d^{5/4}\hat{B}^{1/4}H^{9/4}K^{1-\zeta/4})/K$, we have 
	\begin{align}
		\mathbb{E} \left[\sum^{K/W}_{i=1} \left(W\rho  - G_i(\hat{B}) \right) \right] \leq -\dfrac{(1+\xi)}{\xi}\tilde{\mathcal{O}}((1+1/\delta)K^{1-\zeta/4}H^{9/4}d^{5/4}\hat{B}^{1/4}).
	\end{align}
	For the second term, we are able to obtain an upper bound by using Lemma \ref{le:non-bandit-c} 
	\begin{align}
		&  \mathbb{E} \left[\sum_{i=1}^{K/W} (G_i(\hat{B}) - G_i(B_i))   \right]  \leq 12 K^\lambda H \sqrt{K^{1+\zeta}(J+1)\ln(J+1)}  
	\end{align}
	By balancing the terms $\tilde{O}(K^{1-\zeta/4}), \tilde{O}(K^{\lambda + (1+\zeta)/2} ) $ and $K^{1-\lambda},$ the best selection are $\zeta=1/2$ and $\lambda=1/8.$ Therefore we further obtain
	\begin{align}
		\hbox{Violation}(K) = 0.
	\end{align}
	We finish the proof of Theorem \ref{the:linear-non}.

	\section{ANOTHER APPROACH FOR UNKNOWN BUDGET}\label{ap:method2-non-linear}
	We consider a primal-dual adaptation in the outer loop as well. In particular, after collecting $R_i(B_i)$ and $G_i(B_i)$ under the selected epoch length $B_i$, the bandit reward is $R_i(B_i)+Y_iG_i(B_i),$ where $Y_i=\min\{\max\{Y_{i-1}+\eta(\rho-G_i(B_i)/W),0\},\xi\}$. Then line $10$ in Algorithm \ref{algo:model_free_unknown} is replaced with 
	$$ \hat{R}_i(j) = (R_i(B_i)+Y_iG_i(B_i)) / (WH + \xi WH)   $$ 

	Let $W=d^{1/2}H^{-1/2}K^{1/2}$ be the epoch length, and 
	$$\mathcal{J} =  \left\{ 1,W^{\frac{1}{J}},\ldots, W  \right\},$$
	where $J=\log W$ as the candidate sets for $D$ in the linear CMDPs. We still use Exp-3 to choose an arm.  From the Exp-3 analysis we know for any $D^{\dagger}$
	\begin{align}\label{eq:exp3}
		& \sum_{m}(R_m(D^{\dagger})+Y_m   G_m(D^{\dagger}))-(R_m(D_m)+Y_m G_m(D_m)) \nonumber \\
		\leq & 2\sqrt{e-1}WH(1+ \xi)\sqrt{(K/W)(J+1)\ln(J+1)} = \tilde{\mathcal{O}}(H\xi \sqrt{KW} ) ,
	\end{align}\label{eq:conven-opt}
	Now, from the dual domain analysis, we obtain a similar to (Lemma~\ref{lem:dual_variable})
	\begin{align}\label{eq:dual_unknownbudget}
		\sum_{m}(Y-Y_m)(W\rho - G_m(D_m))\leq \dfrac{Y^2W}{2\eta}+ \frac{\eta H^2K}{2}
	\end{align}
	We note that $\eta=\sqrt{\xi^2W/(KH^2)}$, then the upper bound is $\xi\sqrt{WKH^2}.$
	From the results analysis of the constraint violation from Theorem \ref{thm:linear_cmdp}, we have for the optimal choice of $D^{\dagger}$ from $\cJ$
	\begin{align}
		\sum_{m}(W\rho -G_m(D^{\dagger}))\leq & \tilde{\cO}(K\sqrt{d^3H^4/D^{\dagger}}+D^{\dagger}\sqrt{dD^{\dagger}}H^2B). \label{eq:constr_unknownbudget}\\
		\sum_k^KV_{k,1}^{\pi_k^*}(x_{k,1}) - \sum_{m}R_i(D^\dagger) \leq& \tilde{\cO}(K\sqrt{d^3H^4/D^{\dagger}}+D^{\dagger}\sqrt{dD^{\dagger}}H^2B)\label{eq:reg_unknownbudget}.
	\end{align}
	Hence, we have
	\begin{align}
		&\sum_{m}-Y_m(G_m(D^{\dagger})-G_m(D_m))= \sum_m-Y_m(G_m(D^\dagger)  -W\rho)+\sum_m-Y_m(W\rho-G_m(D_m))\nonumber\\  \leq& \tilde{\cO}\left(K\sqrt{d^3H^4/D^{\dagger}}\xi+D^{\dagger}\sqrt{dD^{\dagger}}H^2B\xi+\xi\sqrt{WKH^2}\right)
	\end{align}
	where we use (\ref{eq:dual_unknownbudget}) (with $Y=0$) for the first inequality, and (\ref{eq:constr_unknownbudget}) (where we use $|Y_m|\leq \xi$) for the second term. 
	
	Hence, from (\ref{eq:exp3})
	\begin{align}\label{eq:reg_arms}
		&\sum_{m}(R_m(D^{\dagger}) - R_m(D_m)) \nonumber\\
		\leq & \tilde{\mathcal{O}}(H\xi 2\sqrt{e-1}WH(1+ \xi)\sqrt{(K/W)(J+1)\ln(J+1)} + \sum_{m}-Y_m(G_m(D^{\dagger})-G_m(D_m)) \nonumber \\
		\leq &\tilde{\cO}\left(K\sqrt{d^3H^4/D^{\dagger}}\xi+D^{\dagger}\sqrt{dD^{\dagger}}H^2B\xi+\xi\sqrt{WKH^2}  + H\xi\sqrt{KW} \right)
	\end{align}
	
	Now, suppose that optimal $D$ exists in the range, thus, $D^{\dagger}\leq D\leq D^{\dagger}W^{1/J}=eD^{\dagger}$. Hence, from $D=B^{-1/2}W$, and (\ref{eq:reg_unknownbudget}) we have the regret bound of $\tilde{\mathcal{O}}((1+1/\delta)H^{9/4}d^{5/4}B^{1/4}K^{3/4})$.
	
	If $D$ is not covered -- if $D<1$, then $B^{-1/2}d^{1/2}H^{-1/2}K^{1/2}\leq 1$, thus, $B\geq \cO(K)$ which will make the regret and violation bound vacuous. Thus, we consider $D> W$. Hence, $B^{-1/2}d^{1/2}H^{-1/2}K^{1/2}> d^{1/2}H^{-1/2}K^{1/2}$, thus, we have $B<1$. Hence, the optimal $D^{\dagger}=d^{1/2}H^{-1/2}K^{1/2}$ by balancing the terms in \eqref{eq:reg_arms}. Thus, the regret bound again follows, i.e., the regret bound is $\tilde{\mathcal{O}}((1+1/\delta)H^{9/4}d^{5/4}B^{1/4}K^{3/4})$.

	Now, we bound the constraint violation.  Note that
	\begin{align}
		&  \sum_k^KV_{k,1}^{\pi_k^*}(x_{k,1}) - \sum_m R_m(D_m) + Y\sum_m(W\rho - G_m(D_m))                            \nonumber \\
		= & \sum_k^KV_{k,1}^{\pi_k^*}(x_{k,1})-\sum_m R_m(D^\dagger) + \sum_m Y_m(W\rho - G_m(D^\dagger)) \nonumber\\
		&+\sum_m Y_m (G_m(D^\dagger) - G_m(D_m)) + \sum_m (R_m(D^\dagger) -R_m(D_m)) + \sum_m (Y-Y_m)(W\rho - G_m(D_m)) \nonumber\\
		\leq &\tilde{\cO}\left(K\sqrt{d^3H^4/D^{\dagger}}\xi+D^{\dagger}\sqrt{dD^{\dagger}}H^2B\xi+\xi\sqrt{WKH^2}  + H\xi\sqrt{KW} \right)
	\end{align}
	where we use (\ref{eq:reg_unknownbudget}), (\ref{eq:constr_unknownbudget}), (\ref{eq:exp3}), and (\ref{eq:dual_unknownbudget}) to bound each term in the right-hand side respectively.
	
	By using lemma \ref{lem:copy}, we can have
	\begin{align}
		\sum_m W\rho - G_m(D_m) \leq & \tilde{\cO}\left(K\sqrt{d^3H^4/D^{\dagger}}+D^{\dagger}\sqrt{dD^{\dagger}}H^2B+\sqrt{WKH^2}  + H\sqrt{KW} \right.  \nonumber \\
		&\left. + \frac{1}{\xi}(K\sqrt{d^3H^4/D^{\dagger}})+D^{\dagger}\sqrt{dD^{\dagger}}H^2B  ) \right)
	\end{align}
	From a similar argument (for regret) where optimal $D$ is covered within the range or not, we bound $D^{\dagger}$ and obtain the result for constraint violation. We prove the results by substituting $\xi=\frac{2H}{\gamma}.$ 
		\begin{figure}[!h]
		\centering
		\begin{subfigure}[b]{0.30\textwidth}
			\centering
			\includegraphics[width=5cm,height=4cm]{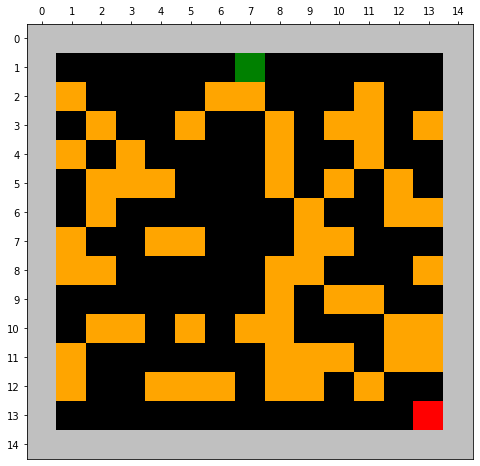}
			\caption{Grid World}
			\label{fig:non-env}
		\end{subfigure}
		\hfill
		\begin{subfigure}[b]{0.55\textwidth}
			\centering
			\includegraphics[width=9cm,height=5cm]{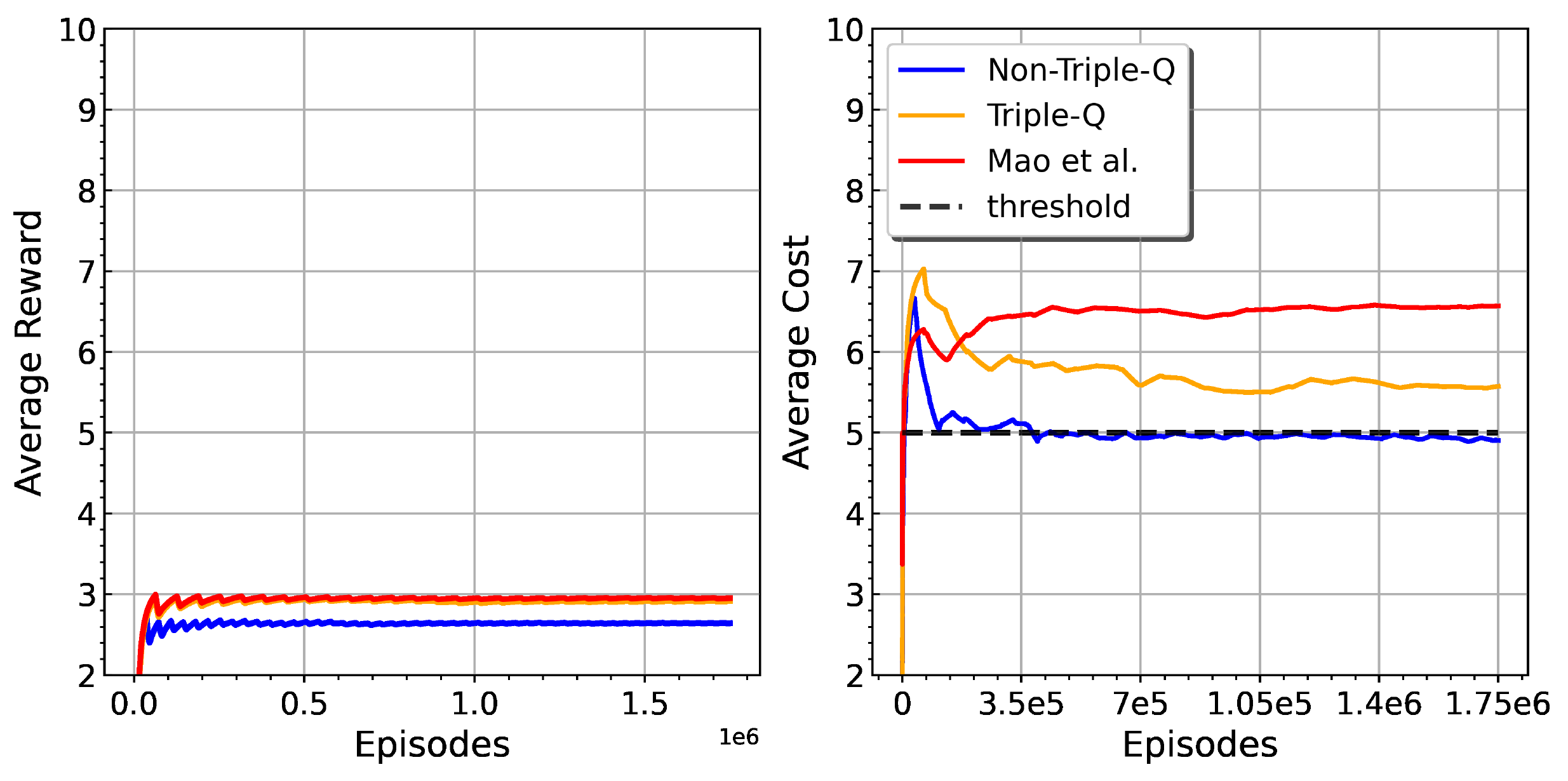}
			\caption{Average Reward and Cost during training}
			\label{fig:per}
		\end{subfigure}
		\caption{ Performance of the three algorithms under a non-stationary environment }
	\end{figure}
	\section{Simulation}\label{sec:sim}
	We compare Algorithm \ref{alg:triple-Q} with two baseline algorithms: an algorithm \citep{MaoZhaZhu_20} for non-stationary MDPs, and an algorithm \citep{WeiLiuYin_22} for stationary constrained MDPs using a grid-world environment, which is shown in Figure. \ref{fig:non-env}. The objective of the agent is to travel to the destination as quickly as possible while avoiding obstacles for safety. Hitting an obstacle incurs a cost of 1. The reward for the destination is $1.$ Denote the Euclidean distance from the current location $x$ to the destination as $d_0(x),$ the longest Euclidean distance is denoted by $d_{\max},$ then the reward function for a locations $x$ is defined as $\frac{0.1*(d_{\max}-d_0)}{d_{\max}}.$ The cost constraint is set to be $5$ (we used cost instead of utility in this simulation), which means the agent is only allowed to hit the obstacles at most five times. To account for the statistical significance, all results were averaged over $10$ trials. To test the algorithms in a non-stationary environment, we gradually vary the transition probability, reward, and cost functions. In particular, the reward is added an additional variation of $\pm\frac{0.1}{K},$ where the sign is uniformly sampled, the cost varies $\frac{0.1}{K}$ at all the locations. We vary the transitions in a way that the intended transition "succeeds" with probability $0.95$ at the beginning; that is, even if the agent takes the correct action at a certain step, there is still a $0.05$ probability that it will take an action randomly. The probability is increased with $\frac{0.1}{K}$ at each iteration.
	
	As shown in Figure. \ref{fig:per}, we can observe that our Algorithm \ref{alg:triple-Q} can quickly learn a well-performed policy while satisfying the safety constraint (below the threshold), while other methods all fail to satisfy the constraint.

\end{document}